\documentclass[11pt]{article}

\makeatletter
\newcommand{\startappendices}{%
    \clearpage
    \section*{Appendix}%
    \phantomsection
    \addcontentsline{toc}{section}{Appendix}%
    \setcounter{subsection}{0}%
    \setcounter{subsubsection}{0}%
    \setcounter{equation}{0}%
    \renewcommand{\thesubsection}{\Alph{subsection}}%
    \renewcommand{\thesubsubsection}{\thesubsection.\arabic{subsubsection}}%
    \renewcommand{\theequation}{\thesubsection.\arabic{equation}}%
    \@addtoreset{equation}{subsection}%
}
\makeatother

\usepackage{amsthm}

\theoremstyle{definition}
\newtheorem{definition}{Definition}[section]

\usepackage[margin=1in]{geometry}

\usepackage{amsmath}
\usepackage{amssymb}
\usepackage{amsfonts}
\usepackage{bm}
\usepackage{mathtools}
\usepackage{amsthm}

\usepackage{microtype}
\usepackage{times}
\usepackage{xcolor}
\usepackage{hyperref}
\usepackage[numbers,sort&compress]{natbib}

\usepackage{enumitem}
\usepackage{booktabs}
\usepackage{algorithm}
\usepackage{algpseudocode}

\hypersetup{
    colorlinks=true,
    linkcolor=blue,
    citecolor=blue,
    urlcolor=blue,
    pdftitle={Beyond Fixed Noise: History-Adaptive Virtual Perturbation Bounds for SGD},
    pdfauthor={},
    pdfsubject={Information-Theoretic Generalization Bounds},
    pdfkeywords={SGD, generalization, mutual information, virtual perturbation, adaptive noise}
}

\newtheorem{assumption}{Assumption}
\newtheorem{theorem}{Theorem}
\newtheorem{lemma}{Lemma}
\newtheorem{proposition}{Proposition}
\newtheorem{corollary}{Corollary}
\newtheorem{remark}{Remark}

\newcommand{\E}{\mathbb{E}}
\newcommand{\R}{\mathbb{R}}
\newcommand{\N}{\mathcal{N}}

\newcommand{\cD}{\mathcal{D}}
\newcommand{\cF}{\mathcal{F}}
\newcommand{\cG}{\mathcal{G}}
\newcommand{\cH}{\mathcal{H}}

\newcommand{\cZ}{\mathcal{Z}}

\newcommand{\KL}{D_{\mathrm{KL}}}
\newcommand{\MI}{I}
\newcommand{\gen}{\mathrm{gen}}

\DeclareMathOperator{\Tr}{Tr}
\DeclareMathOperator{\Cov}{Cov}
\DeclareMathOperator{\diag}{diag}

\newcommand{\wtW}{\widetilde{W}}
\newcommand{\wtw}{\widetilde{w}}

\newcommand{\eps}{\varepsilon}

\newcommand{\barg}{\bar g}

\newcommand{\given}{\;\mid\;}

\newcommand{\inner}[2]{\left\langle #1,#2 \right\rangle}
\newcommand{\norm}[1]{\left\lVert #1 \right\rVert}
\newcommand{\normSigma}[2]{\left\lVert #1 \right\rVert_{#2}}

\title{\textbf{Information-Theoretic Generalization Bounds for Stochastic Gradient Descent with Predictable Virtual Noise}}

\author{
Mohammad Partohaghighi \\
Electrical Engineering and Computer Science \\
University of California, Merced \\
Merced, CA, USA \\
\texttt{mpartohaghighi@ucmerced.edu}  
}

\date{}

\begin{document}

\maketitle

\begin{abstract}
Information-theoretic generalization bounds provide a principled framework for
analyzing stochastic optimization algorithms by relating expected
generalization error to the mutual information between the learned parameters
and the training data. Recent virtual perturbation analyses of stochastic
gradient descent introduce auxiliary Gaussian noise only at the level of the
proof, thereby making mutual information analytically controllable while leaving
the actual SGD trajectory unchanged. Existing bounds, however, typically require
the perturbation variances or covariance matrices to be fixed independently of
the optimization history. This fixed-noise restriction limits the ability of the
theory to represent history-dependent perturbation geometries motivated by
moving gradient statistics, gradient-deviation estimates, preconditioners,
curvature proxies, and other long-term pathwise information.

In this work, we introduce predictable history-adaptive virtual perturbations,
where the perturbation covariance \(\Sigma_t\) at iteration \(t\) may depend on
the past real SGD history \(\mathcal H_{t-1}\), but not on current or future
randomness. This predictability condition enables a conditional Gaussian
relative-entropy argument and leads to information-theoretic generalization
bounds for SGD with history-adaptive virtual noise geometry. The resulting
bounds replace fixed local gradient-deviation and gradient-sensitivity
quantities with conditional adaptive counterparts and include an adaptive
output-sensitivity penalty induced by the accumulated perturbation covariance.
The gradient-deviation term reduces to a conditional variance only under an
appropriate conditional unbiasedness condition.

Because adaptive covariances are generally data-dependent through the
optimization history, our formulation separates the local Gaussian smoothing
step from the global reference-kernel comparison. In the general case, the bound
contains an explicit covariance-comparison cost that measures the KL price of
using an admissible reference perturbation geometry different from the actual
adaptive covariance. The clean fixed-noise-style bound is recovered only under
an admissible synchronization certificate, such as deterministic, public, or
prefix-observable covariance constructions. The framework recovers fixed
isotropic and fixed geometry-aware virtual perturbation bounds as special cases,
while extending virtual perturbation analysis toward history-dependent
stochastic optimization without modifying the underlying SGD algorithm.
\end{abstract}

\section{Introduction}
\label{sec:introduction}

\subsection{Motivation}

Stochastic gradient descent (SGD) is one of the central algorithmic primitives
of modern machine learning. Its simplicity, scalability, and ability to train
highly overparameterized models have made it the default optimization method
behind many large-scale learning systems. Despite this practical success, the
theoretical mechanisms behind the generalization behavior of SGD remain only
partially understood. This gap is especially pronounced in high-dimensional and
nonconvex regimes, where the number of model parameters can far exceed the
number of training examples and where classical uniform convergence arguments
often provide overly conservative explanations. A central question is why an SGD-trained model can achieve low test error
despite being optimized directly on a finite training sample. One way to
formalize this question is through the expected generalization error
\begin{equation}
    \gen(W,S)
    =
    \E
    \left[
        L(W,S')-L(W,S)
    \right],
\end{equation}
where \(S\) is the training sample, \(S'\) is an independent ghost sample drawn
from the same distribution, and \(W\) denotes the output of the learning
algorithm. Information-theoretic generalization theory provides a principled
approach to this problem by relating the generalization error to the mutual
information between the learned parameters and the training data. A
representative bound takes the form
\begin{equation}
    |\gen(W,S)|
    \leq
    \sqrt{
        \frac{2R^2}{n}
        \MI(W;S)
    },
    \label{eq:intro_mi_bound}
\end{equation}
under an appropriate \(R\)-sub-Gaussian loss assumption
\citep{xu2017information}. Thus, controlling generalization can be reduced to
controlling how much information about the training sample is retained in the
output \(W\). Applying \eqref{eq:intro_mi_bound} directly to deterministic or nearly
deterministic SGD dynamics is challenging. Noisy iterative algorithms such as
stochastic gradient Langevin dynamics (SGLD) are often more amenable to
information-theoretic analysis because injected noise smooths the conditional
distributions of iterates and helps control relative entropy. A typical noisy
update has the form
\begin{equation}
    W_{t+1}
    =
    W_t-\eta_tg(W_t,B_t)+\eps_t,
    \qquad
    \eps_t\sim\N(0,\Sigma_t),
\end{equation}
where \(B_t\) is the minibatch and \(\eps_t\) is Gaussian perturbation noise.
The presence of \(\eps_t\) makes one-step relative-entropy comparisons
tractable. Yet this analytical advantage comes at a cost: injecting noise into
the actual optimization trajectory changes the learning algorithm and may
affect training performance.

\subsection{Virtual Perturbations for SGD}

A recent line of work shows that one can obtain information-theoretic
generalization bounds for vanilla SGD without injecting noise into the actual
training dynamics \citep{pensia2018generalization,neu2021information}. The key
idea is to introduce perturbations only as an analytical device. The true SGD
trajectory remains
\begin{equation}
    W_{t+1}
    =
    W_t-\eta_tg(W_t,B_t),
    \label{eq:intro_sgd}
\end{equation}
while a virtual perturbed trajectory is defined only for the purpose of
analysis:
\begin{equation}
    \wtW_{t+1}
    =
    \wtW_t
    -
    \eta_tg(W_t,B_t)
    +
    \eps_t.
    \label{eq:intro_virtual_path}
\end{equation}
Crucially, the gradient in \eqref{eq:intro_virtual_path} is evaluated at the
unperturbed SGD iterate \(W_t\), not at the virtual iterate \(\wtW_t\). Thus,
the perturbation process does not modify the optimization algorithm. It only
produces a noisy analytical shadow of the SGD path. This virtual perturbation construction makes it possible to bound the mutual
information between the perturbed final iterate and the training sample, and
then transfer the resulting generalization guarantee back to the original output
through a perturbation-sensitivity penalty. In the fixed-noise framework, the
resulting bounds depend on local quantities measured along the SGD path, such
as gradient-deviation terms, gradient sensitivity to perturbations, and the
sensitivity of the final loss value to perturbing the output. The
gradient-deviation term reduces to a stochastic gradient variance under
appropriate conditional unbiasedness assumptions, but in general it also
captures possible conditional bias relative to the population gradient. This is
conceptually attractive because the bound reflects path-dependent properties of
the actual SGD trajectory rather than relying solely on global Lipschitz or
smoothness constants.

\subsection{Limitation: Fixed Perturbation Geometry}

Despite its elegance, the standard virtual perturbation framework typically
assumes that the perturbation schedule is fixed independently of the optimization
history. For example, one may choose scalar isotropic perturbations
\begin{equation}
    \eps_t\sim\N(0,\sigma_t^2I),
\end{equation}
or deterministic matrix-valued perturbations
\begin{equation}
    \eps_t\sim\N(0,\Sigma_t),
\end{equation}
where the sequence \(\{\sigma_t\}_{t}\) or \(\{\Sigma_t\}_{t}\) is specified in
advance and does not depend on the observed SGD trajectory.

This fixed-noise assumption is mathematically convenient, but it limits the
flexibility of the theory. In realistic stochastic optimization, many important
quantities evolve with the training path. Adaptive learning rates depend on
past gradients; preconditioners such as AdaGrad- or Adam-type second-moment
estimates accumulate historical information; moving averages summarize
long-term gradient behavior; curvature proxies change as the optimization
trajectory enters different regions of the loss landscape; and estimates of
gradient deviation or stochastic gradient fluctuation may vary significantly
between early and late training. A perturbation geometry fixed before training
cannot adapt to these changes. This limitation is closely related to the broader issue of long-term dependence.
Once a perturbation scale, covariance, or preconditioning geometry is allowed to
depend on past iterates and gradients, the analysis must account for the
dependence between the perturbation process, the data, and the optimization
history. The challenge is to allow such adaptivity without destroying the
conditional structure that makes information-theoretic analysis possible.

\subsection{Our Approach: Predictable History-Adaptive Perturbations}

We propose a history-adaptive extension of the virtual perturbation framework.
Instead of requiring the perturbation covariance to be fixed independently of
the optimization trajectory, we allow it to depend on the past real history of
SGD. Let \(\mathcal H_{t-1}\) denote the real SGD history available before the
perturbation at iteration \(t\) is drawn. Informally, \(\mathcal H_{t-1}\)
contains the past iterates, past minibatches, and past stochastic gradients, but
not the current perturbation or future randomness. At iteration \(t\), we allow
the virtual perturbation covariance to take the form
\begin{equation}
    \Sigma_t
    =
    \Phi_t(\mathcal H_{t-1}),
    \label{eq:intro_adaptive_covariance}
\end{equation}
where \(\Phi_t\) is a measurable covariance-selection rule. The key structural
condition is predictability:
\begin{equation}
    \Sigma_t
    \text{ is }
    \mathcal H_{t-1}\text{-measurable}.
    \label{eq:intro_predictability}
\end{equation}
Equivalently, the covariance at step \(t\) may depend on the past optimization
history, but not on the current perturbation, current unrevealed randomness, or
future information. Under this condition, the virtual perturbation is conditionally Gaussian:
\begin{equation}
    \eps_t\given\mathcal H_{t-1}
    \sim
    \N(0,\Sigma_t).
    \label{eq:intro_conditional_noise}
\end{equation}
This predictability condition is sufficient for the one-step conditional
Gaussian smoothing argument: after conditioning on the relevant past, the
covariance \(\Sigma_t\) is fixed, and conditional relative-entropy bounds can be
applied. However, predictability alone is not the whole mutual-information
argument. Because \(\Sigma_t\) is generally data-dependent through the SGD
history, the reference-kernel comparison in the mutual-information
decomposition must also be handled carefully. In the general case, comparing
actual and reference virtual kernels with different adaptive covariances
introduces an explicit covariance-comparison cost. The clean fixed-noise-style
bound is recovered only when the synchronized reference covariance is
admissible, for example under deterministic, public, or prefix-observable
covariance constructions. The resulting theory is still purely analytical. The true SGD update remains
\eqref{eq:intro_sgd}, and the adaptive perturbations appear only in the virtual
path \eqref{eq:intro_virtual_path}. Thus, our framework does not propose a new
optimizer. Rather, it extends the analytical machinery of virtual perturbation
bounds to perturbation geometries that can adapt to the long-term history of the
SGD trajectory. The accumulated perturbation covariance becomes random and history-dependent:
\begin{equation}
    \Sigma_{1:t}
    =
    \sum_{k=1}^{t-1}\Sigma_k.
\end{equation}
Consequently, the classical fixed-noise gradient-deviation and
gradient-sensitivity terms are replaced by conditional adaptive counterparts.
Let
\begin{equation}
    \bar g(w)
    =
    \E_{Z\sim\mathcal D}[g(w,Z)]
\end{equation}
denote the population gradient. At a high level, our bounds involve an adaptive
gradient-deviation term of the form
\begin{equation}
    V_t^{\mathrm{ad}}
    =
    \E
    \left[
        \normSigma{
            g(W_t,B_t)-\bar g(W_t)
        }{\Sigma_t^{-1}}^2
        \given
        \mathcal H_{t-1}
    \right],
    \label{eq:intro_adaptive_variance}
\end{equation}
and an adaptive gradient-sensitivity term of the form
\begin{equation}
    \Gamma_t^{\mathrm{ad}}
    =
    \E
    \left[
        \normSigma{
            \bar g(W_t+\xi_t)-\bar g(W_t)
        }{\Sigma_t^{-1}}^2
        \given
        \mathcal H_{t-1}
    \right],
    \label{eq:intro_adaptive_sensitivity}
\end{equation}
where \(\xi_t\) denotes the accumulated virtual perturbation up to time \(t\),
conditionally distributed as a centered Gaussian with covariance
\(\Sigma_{1:t}\) given the real SGD history. The term
\(V_t^{\mathrm{ad}}\) is a conditional mean-square deviation from the population
gradient; it becomes a conditional variance only under the appropriate
conditional unbiasedness condition. The final bound also includes an adaptive
output perturbation-sensitivity penalty, which measures the cost of transferring
the information-theoretic guarantee from the perturbed output back to the
original SGD output.

\subsection{Contributions}

The main contribution of this paper is a conditional information-theoretic
extension of virtual perturbation analysis from fixed perturbation schedules to
predictable history-adaptive perturbation processes. Our contributions are
summarized as follows.

\begin{itemize}[leftmargin=2em,itemsep=0.35em]
    \item \textbf{Predictable virtual perturbations.}
    We introduce a history-adaptive virtual perturbation framework in which the
    perturbation covariance at each step is measurable with respect to the past
    real SGD history. This allows the analytical noise geometry to adapt to the
    realized SGD trajectory while preserving the conditional Gaussian structure
    needed for the local one-step relative-entropy analysis.

    \item \textbf{Conditional relative-entropy analysis.}
    We prove and use conditional Gaussian smoothing inequalities for
    perturbations with predictable covariance. This shows that full independence
    of the perturbation covariance from the optimization history is stronger
    than necessary for the local relative-entropy step; predictability is the
    relevant causal condition.

    \item \textbf{Certified reference-kernel comparison.}
    We separate the local Gaussian smoothing step from the global
    mutual-information reference comparison. Since history-adaptive covariances
    are generally data-dependent, a reference kernel cannot automatically use
    the actual covariance. We identify \(S\)-admissible reference constructions,
    including deterministic, public, prefix-observable, and ghost-adaptive
    covariances.

    \item \textbf{General covariance-comparison bound.}
    When the actual and reference perturbation geometries differ, the bound
    contains an explicit covariance-comparison cost. This term is the Gaussian
    KL price of comparing the actual adaptive smoothing covariance with an
    admissible reference covariance. The clean fixed-noise-style bound is
    recovered only under an admissible synchronization certificate.

    \item \textbf{Adaptive gradient-deviation and sensitivity terms.}
    We derive conditional, history-adaptive analogues of the local
    gradient-deviation and gradient-sensitivity quantities appearing in
    fixed-noise virtual perturbation bounds. The gradient-deviation term reduces
    to a conditional variance only under conditional unbiasedness of the
    stochastic gradient; otherwise it also captures conditional bias relative to
    the population gradient.

    \item \textbf{Adaptive output-sensitivity penalty.}
    We characterize the output-sensitivity term induced by the accumulated
    adaptive covariance. This term quantifies the cost of transferring a
    generalization guarantee from the virtually perturbed output back to the
    original SGD output. We control it under global smoothness and further
    refine it through a covariance-weighted local-curvature interpretation.

    \item \textbf{Recovery of fixed-noise bounds and adaptive examples.}
    We show that fixed isotropic and fixed geometry-aware virtual perturbation
    bounds are recovered as special cases. We also discuss adaptive scalar,
    diagonal, Adam-like, curvature-aware, and low-rank covariance choices
    motivated by moving gradient statistics, preconditioned geometry, pathwise
    curvature proxies, and long-term optimization history.
\end{itemize}
Overall, our framework extends virtual perturbation analysis beyond fixed-noise
constructions while leaving the underlying SGD algorithm unchanged. It provides
a theoretical bridge toward understanding generalization in settings where the
analytical perturbation geometry reflects long-term dependence in the
optimization history, while making explicit the reference-admissibility and
covariance-comparison costs required in fully history-adaptive comparisons.

\paragraph{Organization.}
The remainder of the paper is organized as follows.
Section~\ref{sec:related-work} reviews related work on stability-based
generalization, information-theoretic generalization bounds, noisy iterative
algorithms, virtual perturbation analysis, and adaptive optimization with
long-term dependence. Section~\ref{sec:preliminaries} introduces the learning
setup, the SGD dynamics, the information-theoretic generalization bound, and
the fixed virtual perturbation baseline. Section~\ref{sec:fixed-limitation}
explains why fixed perturbation schedules are restrictive and motivates
predictable history-adaptive virtual perturbations. Section~\ref{sec:predictable-perturbations}
formalizes the predictable covariance framework, including the real SGD
history, augmented histories, accumulated perturbations, and conditional
Gaussian structure. Section~\ref{sec:adaptive-quantities} defines the adaptive
gradient-deviation, gradient-sensitivity, reference-geometry, and
output-sensitivity quantities that appear in our bounds.
Section~\ref{sec:main_theoretical_results} presents the main covariance-comparison
generalization theorem, together with admissibly synchronized and
comparable-covariance corollaries. Section~\ref{sec:special-cases} shows how
fixed isotropic, fixed geometry-aware, adaptive scalar, diagonal, and Adam-like
virtual perturbation geometries fit into the framework.
Section~\ref{sec:sensitivity-penalty} controls the adaptive output-sensitivity
penalty under smoothness assumptions and refines it through local curvature.
Section~\ref{sec:discussion} discusses the role of long-term dependence,
predictability, reference comparisons, and the limitations of the framework.
Section~\ref{sec:numerical} provides an optional numerical diagnostic
illustration of the bound proxies and the information-sensitivity trade-off.
Section~\ref{sec:conclusion} concludes and outlines future directions. The
appendices contain the full technical details: Appendix~\ref{app:conditional-kl}
proves the conditional Gaussian relative-entropy lemmas,
Appendix~\ref{app:mi-decomposition} proves the mutual-information decomposition
and conditioning-compression step, Appendix~\ref{app:main-proof} proves the
main theorem, Appendix~\ref{app:fixed-recovery} recovers the fixed-noise bounds,
Appendix~\ref{app:smoothness} proves the smoothness and local-curvature
sensitivity controls.


\section{Related Work}
\label{sec:related-work}

\subsection{Stability and Generalization of SGD}

Algorithmic stability is a classical approach to understanding why learning
algorithms generalize. At a high level, a stable learning algorithm is one whose
output does not change significantly when a single training example is modified
or removed. This notion gives a direct route from algorithmic sensitivity to
bounds on the gap between training and population performance. For stochastic
gradient descent, the work of \citet{hardt2016train} initiated a systematic
stability analysis and showed that SGD can enjoy strong generalization
guarantees under smoothness assumptions. These results are especially
transparent in convex and strongly convex settings, where contraction
properties of gradient-based updates help control the propagation of
perturbations in the data.

The nonconvex setting is substantially more delicate. Modern deep learning
models are high-dimensional, highly nonconvex, and often overparameterized. In
such regimes, uniform stability bounds may depend sensitively on smoothness
constants, step-size schedules, and the number of optimization steps.
Subsequent work has refined stability analyses of SGD by weakening assumptions,
sharpening rates, or considering different loss geometries and sampling schemes
\citep{lei2020fine,bassily2020stability,bousquet2020sharper}. Nevertheless,
stability-based approaches and information-theoretic approaches emphasize
different mechanisms. Stability focuses on sensitivity to changes in individual
samples, whereas the present work follows an information-theoretic route by
controlling the mutual information between the training sample and a virtually
perturbed version of the learned parameters.

Our contribution is complementary to stability-based analyses. We do not
attempt to prove uniform stability of SGD, nor do we analyze the stability of a
new optimizer. Instead, we extend the virtual perturbation method for bounding
mutual information. The perturbation process in our framework is purely
analytical: the underlying SGD trajectory remains unchanged, while the proof
introduces a predictable history-adaptive virtual noise geometry. The resulting
bounds expose path-dependent gradient-deviation, gradient-sensitivity,
covariance-comparison, and output-sensitivity terms rather than a uniform
replacement-sensitivity coefficient.

\subsection{Information-Theoretic Generalization}

Information-theoretic generalization theory provides a broad and
algorithm-agnostic framework for bounding expected generalization error. A
representative result states that, under an \(R\)-sub-Gaussian loss assumption,
the expected generalization error of a learning algorithm with output \(W\)
trained on sample \(S\) can be bounded as
\begin{equation}
    |\gen(W,S)|
    \leq
    \sqrt{
        \frac{2R^2}{n}
        \MI(W;S)
    },
    \label{eq:rw_mi_bound}
\end{equation}
where \(\MI(W;S)\) denotes the mutual information between the output and the
training data. This perspective was developed in influential works by
\citet{russo2016controlling}, \citet{xu2017information}, and related subsequent
studies. The central intuition is that a learning algorithm generalizes well
when its output does not reveal too much information about the idiosyncrasies
of the training sample.

The strength of \eqref{eq:rw_mi_bound} is its generality: it applies to a wide
class of randomized learning algorithms and does not rely directly on uniform
convergence over a hypothesis class. Its main challenge is equally clear: to
obtain a useful generalization guarantee, one must control the mutual
information \(\MI(W;S)\). For deterministic or nearly deterministic algorithms,
this quantity can be large or even unbounded without additional structure. This
difficulty is particularly relevant for vanilla SGD, whose randomness may come
only from initialization, minibatch sampling, or data ordering. Consequently, a
significant body of work has focused on exploiting algorithmic noise,
conditional mutual information, or data-dependent refinements to make
information-theoretic bounds informative
\citep{negrea2019information,haghifam2020sharpened,steinke2020reasoning}.

Our work builds on this information-theoretic tradition but targets a specific
limitation in virtual perturbation analyses: the common assumption that the
perturbation geometry is fixed independently of the optimization history. We
relax this fixed-geometry requirement by allowing predictable
history-adaptive perturbation covariances. At the same time, because such
covariances are generally data-dependent through the SGD trajectory, the
mutual-information comparison must account for the reference-kernel geometry.
Our formulation therefore separates three issues that are often implicit in the
fixed-covariance setting: the local conditional Gaussian smoothing step, the
global admissibility of the reference kernel, and the covariance-comparison cost
incurred when the actual and reference perturbation covariances differ. The
clean fixed-noise-style bound is recovered only when the synchronized reference
covariance is admissible.

\subsection{SGLD and Noisy Iterative Algorithms}

Noisy iterative algorithms are especially natural objects for
information-theoretic analysis. In stochastic gradient Langevin dynamics and
related methods, Gaussian noise is injected directly into the parameter update:
\begin{equation}
    W_{t+1}
    =
    W_t
    -
    \eta_t g(W_t,B_t)
    +
    \eps_t,
    \qquad
    \eps_t \sim \N(0,\Sigma_t).
    \label{eq:rw_sgld}
\end{equation}
The injected noise smooths the conditional distributions of successive iterates
and makes it possible to control relative entropy terms along the optimization
path. This feature enables mutual-information bounds to decompose across
iterations through chain-rule arguments. Works such as
\citet{pensia2018generalization}, \citet{negrea2019information}, and
\citet{haghifam2020sharpened} use this structure to derive generalization
bounds for noisy iterative algorithms, including SGLD.

The benefit of real injected noise is analytical tractability. However, the
same noise also modifies the optimization trajectory. In practice, the amount
and geometry of noise can affect training loss, convergence behavior, and
optimization bias. Thus, while noisy algorithms provide a valuable theoretical
lens, they do not directly explain the generalization behavior of vanilla SGD
unless one can relate the noisy and noiseless dynamics. This motivates virtual
perturbation approaches, which aim to obtain some of the analytical benefits of
noise without altering the actual learning algorithm.

The present work remains in the virtual perturbation setting rather than the
real-noise setting. The Gaussian perturbations are introduced only inside the
proof. Their covariance may adapt to the past real SGD history, but the SGD
recursion itself is not changed. This distinction is important: our predictable
covariance process is an analytical device for controlling mutual information,
not a prescription for injecting adaptive noise during training.

\subsection{Virtual Perturbation Analysis}

Virtual perturbation analysis introduces noise only inside the proof. The
actual SGD iterates follow
\begin{equation}
    W_{t+1}
    =
    W_t
    -
    \eta_t g(W_t,B_t),
    \label{eq:rw_sgd}
\end{equation}
while the analysis constructs a perturbed shadow trajectory
\begin{equation}
    \wtW_{t+1}
    =
    \wtW_t
    -
    \eta_t g(W_t,B_t)
    +
    \eps_t.
    \label{eq:rw_virtual}
\end{equation}
The crucial feature of \eqref{eq:rw_virtual} is that the gradient is evaluated
at the original SGD iterate \(W_t\), not at the perturbed iterate \(\wtW_t\).
Therefore, the perturbations do not define a new training algorithm. They are
auxiliary random variables introduced to make the mutual information between the
data and the perturbed output analytically controllable.

The virtual perturbation framework of \citet{neu2021information} shows that one
can obtain information-theoretic generalization bounds for vanilla SGD by
combining a perturbation analysis of the final output with a chain-rule
decomposition of relative entropy along a virtual noisy path. The resulting
bounds depend on local, pathwise quantities such as gradient-deviation terms,
which reduce to stochastic gradient variance terms under appropriate
unbiasedness assumptions, as well as gradient sensitivity to perturbations
along the SGD path and the sensitivity of the final loss value to output
perturbations. This is a significant conceptual advance because it avoids
injecting noise into the real SGD trajectory while still exploiting noise in
the proof.

Existing virtual perturbation bounds, however, typically assume that the
perturbation scale or covariance sequence is fixed independently of the observed
optimization history. For instance, the perturbations may be isotropic with
prescribed variances \(\sigma_t^2\), or geometry-aware with deterministic
covariance matrices \(\Sigma_t\). Our work extends this line of analysis by
allowing the virtual perturbation covariance to be a predictable process:
\begin{equation}
    \Sigma_t
    =
    \Phi_t(\mathcal H_{t-1}),
    \qquad
    \Sigma_t
    \text{ is }
    \mathcal H_{t-1}\text{-measurable},
    \label{eq:rw_predictable_covariance}
\end{equation}
where \(\mathcal H_{t-1}\) denotes the real SGD history available before the
perturbation at iteration \(t\) is drawn. Under this condition,
\begin{equation}
    \eps_t \given \mathcal H_{t-1}
    \sim
    \N(0,\Sigma_t),
\end{equation}
so after conditioning on the past, the covariance is fixed and the local
one-step conditional Gaussian relative-entropy argument remains valid.

This extension introduces an additional technical issue that is absent in the
deterministic fixed-covariance case. Because \(\Sigma_t\) may depend on the data
through the SGD history, the reference kernel used in the mutual-information
decomposition cannot be treated as if it automatically had access to the same
covariance. The full proof must distinguish the actual predictable covariance
from the admissible reference covariance. If the actual and reference
perturbation covariances differ, the Gaussian comparison produces an explicit
covariance-comparison term. Thus, the clean bound familiar from fixed virtual
perturbation analysis is recovered only when the synchronized comparison is
admissible, while the general history-adaptive theorem includes the
corresponding covariance-comparison cost.

\subsection{Adaptive Optimization and Long-Term Dependence}

Adaptive optimization methods such as AdaGrad, RMSProp, and Adam use historical
gradient information to choose step sizes, preconditioners, or coordinate-wise
scaling factors \citep{duchi2011adaptive,tieleman2012lecture,kingma2015adam}.
These methods introduce long-term dependence because the update at time \(t\)
depends not only on the current stochastic gradient but also on moving averages
or accumulated statistics from previous iterations. Even when the underlying
optimization algorithm is not changed, an analytical perturbation geometry that
reflects local gradient scale, curvature proxies, gradient-deviation estimates,
or stochastic fluctuation estimates naturally becomes history-dependent.

Long-term dependence creates technical obstacles for information-theoretic
analysis. If step sizes, preconditioners, or perturbation covariances depend on
past gradients, then they also depend on the data through the optimization
trajectory. This dependence complicates relative-entropy decompositions that
are straightforward when perturbation covariances are fixed in advance. The key
challenge is to identify a condition that permits adaptation to the past
without allowing the perturbation process to depend on current or future
randomness in a way that breaks the conditional proof.

Our work addresses this challenge for virtual perturbation geometry. We do not
analyze the true adaptive update dynamics of Adam, AdaGrad, or RMSProp, and we
do not propose a new adaptive optimizer. Instead, we isolate a tractable
analytical extension: the covariance of the virtual perturbation at time \(t\)
may depend on the past real SGD history, provided that it is predictable with
respect to \(\mathcal H_{t-1}\). This condition allows the local proof step to
proceed through conditional Gaussian relative-entropy bounds. For the full
mutual-information comparison, our framework further identifies the
reference-admissibility requirement and the covariance-comparison cost that
appears when the actual and reference perturbation geometries differ. As a
result, the framework captures a controlled form of long-term dependence in the
analytical perturbation process while recovering fixed isotropic and fixed
geometry-aware perturbation bounds as special cases.

\section{Preliminaries and Problem Setup}
\label{sec:preliminaries}

This section introduces the notation and assumptions used throughout the paper.
We first define the supervised learning problem, the expected generalization
error, and the information-theoretic bound that forms the basis of our
analysis. We then describe the SGD dynamics and the fixed virtual perturbation
construction that serves as the baseline for our history-adaptive extension.

Throughout the formal development, we use the following indexing convention.
The initial iterate is \(W_1\), the final output is \(W_T\), and hence there are
\(T-1\) SGD updates. For \(t=1,\ldots,T-1\), the \(t\)-th update maps \(W_t\) to
\(W_{t+1}\). This convention avoids the common ambiguity between \(W_T\) and
\(W_{T+1}\).

\subsection{Data, Loss, and Empirical Risk}
\label{subsec:data_loss_empirical_risk}

Let \(Z\) denote a data point taking values in a measurable space \(\cZ\), and
let \(\cD\) be the underlying data-generating distribution. The training sample
is denoted by
\begin{equation}
    S
    =
    (Z_1,\ldots,Z_n),
    \qquad
    Z_i \overset{\mathrm{i.i.d.}}{\sim} \cD.
\end{equation}
We also define an independent ghost sample
\begin{equation}
    S'
    =
    (Z_1',\ldots,Z_n'),
    \qquad
    Z_i' \overset{\mathrm{i.i.d.}}{\sim} \cD,
\end{equation}
where \(S'\) is independent of \(S\) and of all algorithmic and auxiliary
randomness. In some reference-kernel constructions we also use optional public, auxiliary,
or ghost randomness. We denote such randomness by \(\mathcal U\), and whenever
it appears we assume
\begin{equation}
    \mathcal U
    \quad\text{is independent of}\quad
    S.
    \label{eq:prelim_auxiliary_independence}
\end{equation}
If no such auxiliary randomness is used, \(\mathcal U\) is understood to be the
trivial sigma-field. This notation will be used later when defining admissible
reference kernels. Let \(w\in\R^d\) denote the parameter vector of the model. We consider a loss
function
\begin{equation}
    \ell:\R^d\times \cZ \to \R,
\end{equation}
where \(\ell(w,z)\) is the loss of parameter \(w\) on data point \(z\). For a
finite sample \(s=(z_1,\ldots,z_n)\), define the empirical risk
\begin{equation}
    L(w,s)
    =
    \frac{1}{n}
    \sum_{i=1}^{n}
    \ell(w,z_i).
    \label{eq:empirical_risk}
\end{equation}
The corresponding population risk is
\begin{equation}
    L_{\cD}(w)
    =
    \E_{Z\sim \cD}
    \left[
        \ell(w,Z)
    \right].
    \label{eq:population_risk}
\end{equation}
We assume throughout that the loss is differentiable with respect to \(w\). We
write
\begin{equation}
    g(w,z)
    =
    \nabla_w\ell(w,z)
\end{equation}
for the sample gradient and
\begin{equation}
    \bar g(w)
    =
    \E_{Z\sim\cD}
    \left[
        g(w,Z)
    \right]
    \label{eq:population_gradient_prelim}
\end{equation}
for the population gradient. For a minibatch \(B=\{z_i:i\in J\}\), where
\(J\subseteq[n]\), we use the notation
\begin{equation}
    g(w,B)
    =
    \frac{1}{|J|}
    \sum_{i\in J}
    g(w,z_i).
    \label{eq:minibatch_gradient_general}
\end{equation}
We impose the standard sub-Gaussian loss assumption used in
information-theoretic generalization analysis. For every fixed \(w\in\R^d\),
the random variable \(\ell(w,Z)\), with \(Z\sim\cD\), is assumed to be
\(R\)-sub-Gaussian, meaning that for all \(\lambda\in\R\),
\begin{equation}
    \E
    \exp
    \left\{
        \lambda
        \left(
            \ell(w,Z)
            -
            \E[\ell(w,Z)]
        \right)
    \right\}
    \leq
    \exp
    \left(
        \frac{\lambda^2R^2}{2}
    \right).
    \label{eq:subgaussian_loss}
\end{equation}
This condition is satisfied, for example, when the loss is uniformly bounded on
the support of the data distribution. The sub-Gaussian assumption is used only
to connect mutual information to expected generalization error.

\subsection{Expected Generalization Error}
\label{subsec:expected_generalization_error}

Let \(A\) be a possibly randomized learning algorithm, and let
\begin{equation}
    W=A(S)
\end{equation}
denote its output after training on the sample \(S\). The expected
generalization error is defined as
\begin{equation}
    \gen(W,S)
    =
    \E
    \left[
        L(W,S')
        -
        L(W,S)
    \right],
    \label{eq:expected_generalization}
\end{equation}
where the expectation is taken over the training sample \(S\), the ghost sample
\(S'\), the internal randomness of the algorithm, minibatch sampling,
initialization, and any auxiliary randomness introduced for analysis. Since \(S'\) is independent of \(W\), the first term can be interpreted as the
population performance of the learned parameter:
\begin{equation}
    \E[L(W,S')]
    =
    \E[L_{\cD}(W)].
\end{equation}
Thus, \(\gen(W,S)\) measures the expected gap between population risk and
empirical training risk at the algorithmic output.

\subsection{Information-Theoretic Generalization Bound}
\label{subsec:information_theoretic_generalization_bound}

The starting point of our analysis is a classical information-theoretic
generalization inequality. Under the sub-Gaussian assumption in
\eqref{eq:subgaussian_loss}, the expected generalization error of any possibly
randomized learning algorithm satisfies
\begin{equation}
    |\gen(W,S)|
    \leq
    \sqrt{
        \frac{2R^2}{n}
        \MI(W;S)
    },
    \label{eq:mi_generalization_bound}
\end{equation}
where \(\MI(W;S)\) denotes the mutual information between the algorithmic output
\(W\) and the training sample \(S\) \citep{xu2017information}. Equivalently,
\begin{equation}
    \MI(W;S)
    =
    \KL(P_{W,S}\,\|\,P_W\otimes P_S),
\end{equation}
where \(P_{W,S}\) is the joint distribution of \((W,S)\) and
\(P_W\otimes P_S\) is the product of the marginals. The inequality \eqref{eq:mi_generalization_bound} reduces the problem of
bounding expected generalization error to controlling how much information the
learned parameter retains about the training data. For deterministic or nearly
deterministic SGD, directly controlling \(\MI(W_T;S)\) can be difficult. The
virtual perturbation approach addresses this challenge by applying
\eqref{eq:mi_generalization_bound} to a perturbed version of the SGD output and
then controlling the discrepancy between the perturbed and unperturbed outputs
through an output-sensitivity term. The present paper follows this high-level
strategy, but replaces fixed virtual perturbation covariances with predictable
history-adaptive ones.

\subsection{Stochastic Gradient Descent Dynamics}
\label{subsec:sgd_dynamics_prelim}

We now define the SGD dynamics considered in this paper. Let \(W_1\) be an
initial parameter vector drawn from a distribution independent of the training
sample \(S\). For each update \(t=1,\ldots,T-1\), let \(J_t\subseteq[n]\)
denote the minibatch index set and let
\begin{equation}
    B_t
    =
    Z_{J_t}
    =
    \{Z_i:i\in J_t\}
\end{equation}
be the corresponding minibatch. We write \(b_t=|J_t|\) for the minibatch size.
The minibatch stochastic gradient is
\begin{equation}
    G_t
    =
    g(W_t,B_t)
    =
    \frac{1}{b_t}
    \sum_{i\in J_t}
    g(W_t,Z_i).
    \label{eq:sgd_gradient}
\end{equation}
The SGD recursion is
\begin{equation}
    W_{t+1}
    =
    W_t
    -
    \eta_tG_t,
    \qquad
    t=1,\ldots,T-1,
    \label{eq:sgd_dynamics}
\end{equation}
where \(\eta_t>0\) is the learning rate. The final output of the algorithm is
\(W_T\). Unless otherwise stated, the learning-rate schedule and the minibatch-index
selection mechanism are chosen independently of the training sample values. The
index sets \(J_t\) may be deterministic, sampled with replacement, sampled
without replacement, or generated according to any data-value-independent rule.
The realized minibatches \(B_t=Z_{J_t}\) and gradients \(G_t\) are, of course,
data-dependent. For later use, we define the real SGD history available immediately before
update \(t\) as
\begin{equation}
    \mathcal H_{t-1}
    =
    \sigma
    \left(
        W_1,\ldots,W_t,\,
        J_1,\ldots,J_{t-1},\,
        B_1,\ldots,B_{t-1},\,
        G_1,\ldots,G_{t-1}
    \right),
    \qquad
    t=1,\ldots,T-1.
    \label{eq:real_sgd_history_prelim}
\end{equation}
The sigma-field \(\mathcal H_{t-1}\) contains the past SGD trajectory and past
observed optimization information, but not the current perturbation, current
unrevealed randomness, current stochastic gradient, or future information.
Importantly, \(\mathcal H_{t-1}\) does not include virtual perturbations. This
distinction will be essential when defining predictable history-adaptive
covariance processes.

\begin{remark}[Gradient-deviation versus variance]
\label{rem:prelim_gradient_deviation_vs_variance}
Throughout the paper, quantities of the form
\[
    \E
    \left[
        \normSigma{
            G_t-\bar g(W_t)
        }{M}^2
        \given
        \mathcal H_{t-1}
    \right]
\]
are interpreted as conditional mean-square gradient-deviation terms from the
population gradient. They reduce to conditional variance terms only when
\[
    \E[G_t\mid \mathcal H_{t-1}]
    =
    \bar g(W_t).
\]
This distinction is important for finite-sample sampling schemes and
history-dependent conditioning, where the current stochastic gradient may not be
conditionally unbiased for the population gradient.
\end{remark}

\subsection{Fixed Virtual Perturbation Path}
\label{subsec:fixed_virtual_perturbation_path}

The fixed virtual perturbation construction introduces an auxiliary noisy
trajectory coupled to the original SGD path. For a deterministic sequence of
positive definite covariance matrices
\begin{equation}
    \Sigma_1,\ldots,\Sigma_{T-1}\in\R^{d\times d},
    \qquad
    \Sigma_t\succ0,
\end{equation}
let
\begin{equation}
    \eps_t\sim\N(0,\Sigma_t),
    \qquad
    t=1,\ldots,T-1,
\end{equation}
be independent Gaussian perturbations, independent of \(S\), the
initialization, and the minibatch-index randomness. The fixed virtual
perturbation path is defined recursively by
\begin{equation}
    \wtW_{t+1}
    =
    \wtW_t
    -
    \eta_tG_t
    +
    \eps_t,
    \qquad
    t=1,\ldots,T-1,
    \label{eq:fixed_virtual_path}
\end{equation}
with \(\wtW_1=W_1\).

The crucial feature of \eqref{eq:fixed_virtual_path} is that \(G_t\) is the
same stochastic gradient used by the original SGD update
\eqref{eq:sgd_dynamics}. In particular, \(G_t\) is evaluated at the true iterate
\(W_t\), not at the virtual iterate \(\wtW_t\). Thus, the virtual perturbation
path does not define a new optimization algorithm; it is an analytical
construction used to control information-theoretic quantities. Let the accumulated perturbation before iterate \(t\) be
\begin{equation}
    \xi_t
    =
    \sum_{k=1}^{t-1}
    \eps_k,
    \qquad
    \xi_1=0.
    \label{eq:accumulated_noise}
\end{equation}
By comparing \eqref{eq:sgd_dynamics} and \eqref{eq:fixed_virtual_path}, we
obtain
\begin{equation}
    \wtW_t
    =
    W_t+\xi_t,
    \qquad
    t=1,\ldots,T.
    \label{eq:virtual_equals_sgd_plus_noise}
\end{equation}
Since the perturbations are independent Gaussian random variables, the
accumulated perturbation satisfies
\begin{equation}
    \xi_t
    \sim
    \N(0,\Sigma_{1:t}),
    \qquad
    \Sigma_{1:t}
    =
    \sum_{k=1}^{t-1}
    \Sigma_k,
    \label{eq:fixed_accumulated_covariance}
\end{equation}
with \(\Sigma_{1:1}=0\). In the isotropic special case
\(\Sigma_t=\sigma_t^2I\), this becomes
\begin{equation}
    \Sigma_{1:t}
    =
    \sigma_{1:t}^2I,
    \qquad
    \sigma_{1:t}^2
    =
    \sum_{k=1}^{t-1}
    \sigma_k^2.
\end{equation}

This fixed perturbation construction is the baseline used in existing virtual
perturbation analyses of SGD \citep{neu2021information}. The main goal of the
present paper is to extend this construction by replacing the deterministic
covariance sequence \(\{\Sigma_t\}\) with a predictable covariance process of
the form
\begin{equation}
    \Sigma_t
    =
    \Phi_t(\mathcal H_{t-1}),
    \qquad
    \Sigma_t
    \text{ is }
    \mathcal H_{t-1}\text{-measurable}.
    \label{eq:prelim_predictable_preview}
\end{equation}
This replacement changes only the analytical perturbation process; the true SGD
recursion \eqref{eq:sgd_dynamics} remains unchanged.

Because a history-adaptive covariance is generally data-dependent through the
real SGD trajectory, the later mutual-information comparison must also account
for the reference perturbation geometry. Deterministic covariances are
automatically admissibly synchronized. In the general history-adaptive setting,
the clean fixed-noise-style bound is recovered only when the reference kernel
has an admissible synchronization certificate. Otherwise, one must use an
admissible reference covariance and retain the explicit covariance-comparison
cost in the main theorem.

\section{Limitation of Fixed Virtual Perturbations}
\label{sec:fixed-limitation}

Virtual perturbation analysis provides a way to study the generalization
behavior of SGD without injecting noise into the actual training dynamics. The
perturbations are introduced only in the proof, and the resulting noisy shadow
trajectory is used to control information-theoretic quantities. Existing
analyses typically rely on perturbation schedules that are fixed before the
optimization trajectory is observed. This section explains why such fixed
perturbation geometries are analytically convenient but restrictive, and why
predictable history-adaptive perturbations provide a natural relaxation.

We also clarify a key technical point. Predictability is sufficient for the
local conditional Gaussian smoothing step, but it is not by itself sufficient
for the full mutual-information comparison. The latter also requires an
admissible reference-kernel construction. When the actual and reference
covariances differ, the bound must include an explicit covariance-comparison
cost.

\subsection{Fixed Covariance Assumption}
\label{subsec:fixed_covariance_assumption}

In the standard virtual perturbation framework, the auxiliary noise variables
are drawn from a predetermined Gaussian sequence. Following the indexing
convention of this paper, \(W_1\) is the initial iterate, \(W_T\) is the final
output, and the updates are indexed by \(t=1,\ldots,T-1\). One fixes positive
definite covariance matrices
\begin{equation}
    \Sigma_1,\ldots,\Sigma_{T-1}\in\R^{d\times d},
    \qquad
    \Sigma_t\succ0,
\end{equation}
or, in the isotropic case, scalar perturbation variances
\begin{equation}
    \Sigma_t=\sigma_t^2 I.
\end{equation}
The virtual perturbations then satisfy
\begin{equation}
    \eps_t\sim\N(0,\Sigma_t),
    \qquad
    t=1,\ldots,T-1,
    \label{eq:fixed_noise_assumption}
\end{equation}
with the covariance sequence chosen independently of the optimization history.

Under this assumption, the virtual path is defined by
\begin{equation}
    \wtW_{t+1}
    =
    \wtW_t
    -
    \eta_tG_t
    +
    \eps_t,
    \qquad
    G_t=g(W_t,B_t),
    \label{eq:fixed_virtual_again}
\end{equation}
where \(W_t\) denotes the original SGD iterate and \(B_t\) is the minibatch at
time \(t\). The gradient is evaluated at \(W_t\), not at \(\wtW_t\), so the
perturbations do not modify the actual SGD algorithm. They only generate an
analytical shadow trajectory.

The fixed covariance assumption makes the accumulated perturbation particularly
simple. Defining
\begin{equation}
    \xi_t
    =
    \sum_{k=1}^{t-1}
    \eps_k,
    \qquad
    \xi_1=0,
    \qquad
    \wtW_t=W_t+\xi_t,
\end{equation}
we have
\begin{equation}
    \xi_t\sim\N(0,\Sigma_{1:t}),
    \qquad
    \Sigma_{1:t}
    =
    \sum_{k=1}^{t-1}\Sigma_k,
    \qquad
    \Sigma_{1:1}=0.
    \label{eq:fixed_accumulated_cov_again}
\end{equation}
The accumulated covariance \(\Sigma_{1:t}\) is deterministic. This is one
reason the fixed-noise analysis is technically clean: the perturbation law is
independent of the data and of the realized SGD trajectory, which simplifies
relative-entropy decompositions and mutual-information bounds
\citep{neu2021information}.

However, this simplicity comes from a strong restriction. The analytical
perturbation geometry must be specified without observing the path whose local
geometry it is intended to probe. The rest of this section explains why this
restriction becomes limiting when the goal is to capture history-dependent or
adaptive perturbation geometries.

\subsection{Why Fixed Perturbation Geometry Is Restrictive}
\label{subsec:why_fixed_geometry_restrictive}

The geometry of the SGD trajectory is rarely stationary over training. At early
stages, gradients may be large, noisy, and poorly aligned across minibatches. In
the middle of training, local curvature, gradient-deviation behavior,
stochastic fluctuation, and coordinate-wise scales can change as the trajectory
moves across different regions of the loss landscape. Near the end of training,
the sensitivity of the loss to parameter perturbations, often interpreted
through flatness or sharpness, can become a dominant feature of the
generalization behavior.

A fixed perturbation covariance cannot adapt to these changes. For example, an
isotropic covariance \(\sigma_t^2I\) treats all directions in parameter space
equally, even though the local loss geometry may be highly anisotropic. A
deterministic matrix-valued sequence \(\Sigma_t\) can encode a preferred
geometry, but only one selected before the optimization path is observed. Such
a sequence cannot respond to realized gradient norms, minibatch fluctuation,
gradient-deviation proxies, curvature proxies, or coordinate-wise scale changes
along the actual trajectory.

This restriction matters because virtual perturbation bounds involve a trade-off
between two effects. Larger perturbations can reduce inverse-covariance weighted
information terms by smoothing the conditional distributions of the perturbed
iterates. At the same time, larger accumulated perturbations can increase the
final output-sensitivity penalty, since the perturbed output may incur a
different loss than the original SGD output. The favorable balance between these
two effects is generally path-dependent. A covariance that is appropriate in a
noisy, high-gradient region of training may be inappropriate near a sharp or
flat final region, and vice versa.

The point is not that the SGD algorithm itself should be changed. Rather, the
analytical perturbations used to study SGD should be allowed to reflect the
local and historical structure of the trajectory being analyzed. A more flexible
theory should allow the virtual perturbation geometry to depend on past
optimization information while preserving enough conditional structure for
information-theoretic analysis.

\subsection{Long-Term Dependence and Data-Dependent Geometry}
\label{subsec:long_term_dependence_data_dependent_geometry}

The natural mathematical object for representing the information available
before update \(t\) is the real SGD history
\begin{equation}
    \mathcal H_{t-1}
    =
    \sigma
    \left(
        W_1,\ldots,W_t,\,
        J_1,\ldots,J_{t-1},\,
        B_1,\ldots,B_{t-1},\,
        G_1,\ldots,G_{t-1}
    \right),
    \label{eq:natural_history_limitation}
\end{equation}
where
\begin{equation}
    G_s=g(W_s,B_s)
\end{equation}
is the stochastic gradient at time \(s\). This history contains the past SGD
trajectory and past observed optimization information, but it does not contain
the virtual perturbations. Quantities such as moving averages of gradient
norms, subbatch fluctuation diagnostics, gradient-deviation proxies, diagonal
scale estimates, curvature proxies, and preconditioner-like statistics are
naturally functions of \(\mathcal H_{t-1}\).

Consequently, a natural history-adaptive virtual perturbation covariance would
have the form
\begin{equation}
    \Sigma_t
    =
    \Phi_t(\mathcal H_{t-1}),
    \label{eq:adaptive_cov_natural}
\end{equation}
where \(\Phi_t\) is a measurable covariance-selection rule. This formulation
allows the analytical noise geometry at time \(t\) to reflect information
gathered along the previous SGD path.

The difficulty is that \(\Sigma_t\) is then data-dependent. Indeed, the iterates
\(W_1,\ldots,W_t\), past minibatches, and past gradients
\(G_1,\ldots,G_{t-1}\) are functions of the training sample \(S\), the
minibatch-index schedule, and the initialization. Thus, the perturbation
covariance in \eqref{eq:adaptive_cov_natural} is no longer independent of the
training data or of the optimization trajectory. This dependence prevents the fixed-noise proof from being applied directly. In
the deterministic covariance case, the conditional distribution of each
perturbation is known in advance. When the covariance is history-dependent, the
perturbation law itself becomes random and path-dependent. The accumulated
covariance
\begin{equation}
    \Sigma_{1:t}
    =
    \sum_{k=1}^{t-1}\Sigma_k
\end{equation}
is no longer deterministic; it is a random object determined by the past
trajectory. As a result, the accumulated perturbation \(\xi_t\) is no longer
governed by a fixed Gaussian law independent of the optimization history.

The central issue, however, is not adaptivity itself. The central issue is
whether the adaptivity is controlled in a way that preserves the conditional
Gaussian structure needed for the local relative-entropy step, and whether the
resulting data-dependent perturbation geometry can be compared to a valid
reference kernel in the mutual-information decomposition. This observation
motivates the predictability condition introduced next.

\subsection{Predictability as a Causal Relaxation}
\label{subsec:predictability_causal_relaxation}

Full independence of the perturbation covariance from the optimization history
is stronger than necessary for the one-step Gaussian smoothing argument. What is
needed locally is that \(\Sigma_t\) be known before the current virtual
perturbation is drawn. This is precisely the notion of predictability. We
require that
\begin{equation}
    \Sigma_t
    \text{ is }
    \mathcal H_{t-1}\text{-measurable}.
    \label{eq:predictability_condition_limitation}
\end{equation}
Under this condition, \(\Sigma_t\) may depend on past iterates, past
minibatches, past stochastic gradients, and other statistics computed from the
previous optimization history. However, it may not depend on the current
perturbation \(\eps_t\), the current unrevealed minibatch randomness, the current
stochastic gradient \(G_t\), or future information. With \eqref{eq:predictability_condition_limitation}, the virtual perturbation
at time \(t\) satisfies the conditional Gaussian relation
\begin{equation}
    \eps_t \given \mathcal H_{t-1}
    \sim
    \N(0,\Sigma_t).
    \label{eq:conditional_gaussian_predictable}
\end{equation}
After conditioning on the relevant past, the covariance is fixed, so the
conditional Gaussian relative-entropy argument can be applied at the current
step. This is the main technical replacement for the fixed-noise assumption in
the local smoothing part of the proof. There is, however, an additional issue at the level of the full
mutual-information comparison. Because \(\Sigma_t\) may depend on the data
through \(\mathcal H_{t-1}\), a reference kernel in the chain-rule decomposition
cannot automatically be assumed to use the same covariance without
justification. In the notation used later, the reference kernel may use a
reference covariance \(\Sigma_t^{\mathrm{ref}}\) that is admissible with respect
to the reference-visible information, but this covariance need not equal the
actual adaptive covariance \(\Sigma_t\). If the actual and reference virtual
perturbation covariances differ, the Gaussian comparison produces an explicit
covariance-comparison, or covariance-mismatch, term.

Thus, the clean fixed-noise-style bound is recovered only when the synchronized
reference covariance is admissible, for example in deterministic,
public-random, prefix-observable, or otherwise certified constructions. Without
such a certificate, the general history-adaptive theorem must retain the
covariance-comparison cost.

The resulting framework keeps the original SGD trajectory unchanged:
\begin{equation}
    W_{t+1}
    =
    W_t-\eta_tG_t,
    \qquad
    G_t=g(W_t,B_t).
\end{equation}
Only the virtual perturbation process is generalized. Thus, the contribution is
analytical rather than algorithmic: we do not propose a new optimizer, but
rather a more flexible proof framework for studying the generalization of SGD
under history-adaptive virtual perturbation geometry.

The following sections formalize this predictable perturbation framework, define
the corresponding adaptive gradient-deviation, gradient-sensitivity,
reference-geometry, and output-sensitivity terms, and prove the resulting
information-theoretic generalization bounds. The fixed covariance setting is
recovered as the special case in which every \(\Sigma_t\) is deterministic and
hence automatically \(\mathcal H_{t-1}\)-measurable. In the fully
history-adaptive setting, the main theorem additionally tracks the
covariance-comparison cost required by the reference-kernel argument.


\section{Predictable History-Adaptive Virtual Perturbations}
\label{sec:predictable-perturbations}

We now introduce the main analytical object of the paper: predictable
history-adaptive virtual perturbations. The goal is to generalize fixed virtual
perturbation schedules by allowing the covariance of the analytical noise to
depend on the past real SGD trajectory. The true SGD dynamics remain unchanged.
Only the virtual perturbation process used in the proof is made adaptive. A technical point is important at the outset. The covariance process should be
predictable with respect to the real SGD history, not with respect to a
filtration that already contains the realized virtual perturbations. If one
conditions on a sigma-field containing \(\eps_1,\ldots,\eps_{t-1}\), then the
accumulated perturbation
\[
    \xi_t=\sum_{k=1}^{t-1}\eps_k
\]
is already known, and its conditional covariance is zero. For this reason, we
distinguish between the real SGD history and the augmented history that also
contains the virtual perturbations.

\subsection{SGD History and Augmented Filtrations}
\label{subsec:sgd_history_augmented_filtrations}

Recall that the true SGD recursion is
\begin{equation}
    W_{t+1}
    =
    W_t-\eta_tG_t,
    \qquad
    G_t=g(W_t,B_t),
    \qquad
    t=1,\ldots,T-1,
    \label{eq:section5_sgd}
\end{equation}
where \(W_1\) is the initial iterate and \(W_T\) is the final output after
\(T-1\) updates. The virtual perturbations introduced below are auxiliary
random variables and do not influence the true iterates \(W_t\). For \(r=0,\ldots,T-1\), define the real SGD history after \(r\) updates by
\begin{equation}
    \cH_r
    =
    \sigma
    \left(
        W_1,\ldots,W_{r+1},\,
        J_1,\ldots,J_r,\,
        B_1,\ldots,B_r,\,
        G_1,\ldots,G_r
    \right).
    \label{eq:sgd_history}
\end{equation}
Thus, immediately before update \(t\), the available real history is
\(\cH_{t-1}\). This sigma-field contains the past optimization trajectory, past
minibatch index sets, past realized minibatches, and past stochastic gradients.
It does not contain the current minibatch index \(J_t\), the current minibatch
\(B_t\), the current stochastic gradient \(G_t\), the current virtual
perturbation \(\eps_t\), or future randomness. Most importantly,
\(\cH_{t-1}\) does not contain any virtual perturbations. It is also useful to define the augmented history that includes the previously
realized virtual perturbations. For \(r=0,\ldots,T-1\), let
\begin{equation}
    \cG_r
    =
    \cH_r
    \vee
    \sigma(\eps_1,\ldots,\eps_r),
    \label{eq:augmented_history}
\end{equation}
with the convention that \(\sigma(\eps_1,\ldots,\eps_0)\) is the trivial
sigma-field. Thus, before update \(t\), the augmented history is
\(\cG_{t-1}\). The distinction between \(\cH_{t-1}\) and \(\cG_{t-1}\) is essential. The real
history \(\cH_{t-1}\) is used to define predictable covariance selection. The
augmented history \(\cG_{t-1}\) describes the realized virtual path. Since
\(\cG_{t-1}\) contains \(\eps_1,\ldots,\eps_{t-1}\), it also contains
\(\xi_t\). Therefore, the accumulated perturbation is measurable with respect to
\(\cG_{t-1}\), but not with respect to \(\cH_{t-1}\). This separation reflects the role of virtual perturbations in the analysis. The
true SGD path is generated without using the analytical perturbations. The
perturbations create a noisy shadow path, but they do not feed back into the
optimization dynamics.

\subsection{Predictable Covariance Processes}
\label{subsec:predictable_covariance_processes}

We now define the class of adaptive perturbation covariances considered in this
work.

\paragraph{Definition.}
A sequence of random positive definite matrices
\[
    \{\Sigma_t\}_{t=1}^{T-1}
\]
is called a predictable history-adaptive covariance process if
\begin{equation}
    \Sigma_t\succ0,
    \qquad
    \Sigma_t
    \text{ is }
    \cH_{t-1}\text{-measurable},
    \qquad
    t=1,\ldots,T-1.
    \label{eq:predictable_covariance_def}
\end{equation}

Thus, \(\Sigma_t\) may depend on the past iterates
\(W_1,\ldots,W_t\), past minibatch index sets \(J_1,\ldots,J_{t-1}\), past
realized minibatches \(B_1,\ldots,B_{t-1}\), past stochastic gradients
\(G_1,\ldots,G_{t-1}\), or any statistic computed from them. For example,
\(\Sigma_t\) may be chosen using moving averages of gradient norms, subbatch
fluctuation diagnostics, gradient-deviation proxies, diagonal scale estimates,
curvature proxies, or preconditioner-like quantities computed from the past
trajectory.

The covariance \(\Sigma_t\) may not depend on the current perturbation
\(\eps_t\), the current minibatch index \(J_t\), the current minibatch \(B_t\),
the current stochastic gradient \(G_t\), or future information. This restriction
ensures that, at the moment \(\eps_t\) is drawn, the covariance matrix is
already determined by the past. In this sense, predictability is the causal
relaxation of the stronger fixed-covariance assumption used in classical
virtual perturbation analysis. For technical convenience, the main theorem assumes nondegenerate
perturbations. One may impose, for example,
\begin{equation}
    \lambda_{\min}(\Sigma_t)
    \geq
    \lambda_0
    >
    0
    \qquad
    \text{almost surely for all }t.
    \label{eq:uniform_nondegeneracy}
\end{equation}
This condition guarantees that \(\Sigma_t^{-1}\) is well-defined and avoids
singular perturbation laws. Singular or low-rank perturbations require either
regularization or pseudo-inverse formulations together with support conditions;
these variants are not part of the basic theorem.

Deterministic fixed covariance schedules are recovered as a special case.
Indeed, if each \(\Sigma_t\) is deterministic, then it is automatically
\(\cH_{t-1}\)-measurable. Thus, the predictable framework extends fixed
isotropic and fixed geometry-aware perturbation settings.

\subsection{Conditional Construction of the Virtual Perturbations}
\label{subsec:conditional_virtual_perturbation_construction}

Given a predictable covariance process
\[
    \{\Sigma_t\}_{t=1}^{T-1},
\]
we construct the virtual perturbations conditionally on the real SGD trajectory.
Equivalently, after conditioning on the final real history \(\cH_{T-1}\), the
perturbations
\[
    \eps_1,\ldots,\eps_{T-1}
\]
are conditionally independent centered Gaussian random vectors with
\begin{equation}
    \eps_t\given \cH_{T-1}
    \sim
    \N(0,\Sigma_t),
    \qquad
    t=1,\ldots,T-1.
    \label{eq:conditional_noise_full_history}
\end{equation}
Because \(\Sigma_t\) is \(\cH_{t-1}\)-measurable and
\(\cH_{t-1}\subseteq\cH_{T-1}\), this conditional construction is well-defined.

Sequentially, this implies
\begin{equation}
    \eps_t \given \cG_{t-1}
    \sim
    \N(0,\Sigma_t),
    \qquad
    t=1,\ldots,T-1.
    \label{eq:conditional_noise_sequential}
\end{equation}
Moreover, \(\eps_t\) is conditionally independent of the current minibatch-index
randomness, the current minibatch, the current stochastic gradient, and all
future real-SGD randomness, given the appropriate past history. This conditional
independence is used later when applying the conditional Gaussian
relative-entropy lemmas. The adaptive virtual perturbation path is defined by
\begin{equation}
    \wtW_{t+1}
    =
    \wtW_t
    -
    \eta_tG_t
    +
    \eps_t,
    \qquad
    t=1,\ldots,T-1,
    \label{eq:adaptive_virtual_path}
\end{equation}
with \(\wtW_1=W_1\). As in fixed virtual perturbation analysis
\citep{neu2021information}, the stochastic gradient \(G_t\) in
\eqref{eq:adaptive_virtual_path} is the gradient computed along the true SGD
trajectory:
\begin{equation}
    G_t=g(W_t,B_t).
\end{equation}
It is not computed at the perturbed iterate \(\wtW_t\). Therefore,
\eqref{eq:adaptive_virtual_path} does not define a new optimization algorithm.
It is an analytical construction coupled to the original SGD path.

\subsection{Accumulated Adaptive Covariance}
\label{subsec:accumulated_adaptive_covariance}

For \(t=1,\ldots,T\), define the accumulated perturbation before iterate \(t\) as
\begin{equation}
    \xi_t
    =
    \sum_{k=1}^{t-1}\eps_k,
    \qquad
    \xi_1=0.
    \label{eq:adaptive_accumulated_noise}
\end{equation}
Comparing the true recursion \eqref{eq:section5_sgd} with the virtual recursion
\eqref{eq:adaptive_virtual_path}, we obtain
\begin{equation}
    \wtW_t
    =
    W_t+\xi_t,
    \qquad
    t=1,\ldots,T.
    \label{eq:adaptive_virtual_equals_true_plus_noise}
\end{equation}
Thus, the virtual path remains a perturbation of the true SGD path, but the
perturbation covariance is now selected adaptively from the past real trajectory. The accumulated adaptive covariance before iterate \(t\) is
\begin{equation}
    \Sigma_{1:t}
    =
    \sum_{k=1}^{t-1}
    \Sigma_k,
    \qquad
    \Sigma_{1:1}=0.
    \label{eq:adaptive_accumulated_covariance}
\end{equation}
Unlike in the fixed-noise setting, \(\Sigma_{1:t}\) is generally random. It
depends on the realized SGD trajectory through the predictable covariance
process. Nevertheless, \(\Sigma_{1:t}\) is measurable with respect to the real SGD
history \(\cH_{t-1}\). Indeed, for each \(k<t\), predictability gives that
\(\Sigma_k\) is \(\cH_{k-1}\)-measurable. Since the histories are increasing,
\begin{equation}
    \cH_{k-1}
    \subseteq
    \cH_{t-1},
    \qquad
    k<t,
\end{equation}
each \(\Sigma_k\), and hence their sum \(\Sigma_{1:t}\), is
\(\cH_{t-1}\)-measurable. The matrix \(\Sigma_{1:t}\) is the adaptive analogue of the deterministic
accumulated covariance appearing in fixed perturbation analysis. It represents
the total virtual noise geometry accumulated along the trajectory before
iterate \(t\). Because it is history-dependent, later gradient-deviation and
sensitivity quantities will be defined conditionally with respect to the real
SGD history.

\subsection{Conditional Perturbation Structure}
\label{subsec:conditional_perturbation_structure}

The key conditional distributional fact is that the accumulated perturbation
remains Gaussian when conditioned on the real SGD history. Under the conditional
construction above, for \(t=1,\ldots,T\),
\begin{equation}
    \xi_t \given \cH_{t-1}
    \sim
    \N(0,\Sigma_{1:t}),
    \label{eq:conditional_accumulated_noise}
\end{equation}
where for \(t=T\), \(\cH_{T-1}\) is the final real history after all \(T-1\)
updates. Consequently,
\begin{equation}
    \E[\xi_t\given \cH_{t-1}]
    =
    0,
    \qquad
    \Cov(\xi_t\given \cH_{t-1})
    =
    \Sigma_{1:t}.
    \label{eq:conditional_mean_covariance}
\end{equation}

The conditioning in \eqref{eq:conditional_accumulated_noise} is crucial. If
instead one conditions on the augmented history \(\cG_{t-1}\), then the past
perturbations \(\eps_1,\ldots,\eps_{t-1}\) are already known. Therefore
\(\xi_t\) is \(\cG_{t-1}\)-measurable and
\begin{equation}
    \Cov(\xi_t\given \cG_{t-1})
    =
    0.
    \label{eq:zero_cov_augmented_history}
\end{equation}
Thus, statements about the covariance of the accumulated perturbation must be
made relative to the real SGD history \(\cH_{t-1}\), not the augmented history
containing the realized perturbations. This distinction prevents a common filtration mistake in adaptive perturbation
analysis. The covariance process is predictable with respect to
\(\cH_{t-1}\), while the realized virtual trajectory is described by
\(\cG_{t-1}\). The former is used to define the conditional perturbation law;
the latter describes the realized noisy path used in mutual-information
decompositions.

Finally, this section only defines the adaptive perturbation process. Because
\(\Sigma_t\) is generally data-dependent through the real SGD history, the later
mutual-information proof must compare the actual virtual kernel with a valid
reference kernel. The reference kernel may use an admissible reference
covariance \(\Sigma_t^{\mathrm{ref}}\), which need not equal the actual
covariance \(\Sigma_t\). When an admissible synchronization certificate permits
\(\Sigma_t^{\mathrm{ref}}=\Sigma_t\), the clean fixed-noise-style bound is
recovered. When the actual and reference covariances differ, an explicit
covariance-comparison term appears in the information bound. In the next section, we use the conditional perturbation structure developed
here to define adaptive analogues of local gradient-deviation,
gradient-sensitivity, reference-geometry, and output perturbation-sensitivity
quantities. These quantities replace the fixed covariance terms appearing in
classical virtual perturbation bounds and provide the ingredients for the main
generalization theorem.

\section{Adaptive Local Sensitivity and Gradient-Deviation Measures}
\label{sec:adaptive-quantities}

The main history-adaptive generalization bounds involve three types of local
quantities: stochastic gradient deviation, gradient sensitivity, and output
perturbation sensitivity. In fixed virtual perturbation analysis, these
quantities are evaluated under a deterministic perturbation geometry. In the
present setting, the perturbation covariance \(\Sigma_t\) is predictable and may
depend on the past real SGD history. Consequently, the relevant quantities are
naturally defined conditionally with respect to the real SGD history and, in the
general reference-geometry case, with respect to admissible reference-side
randomness. The definitions below are analytical objects used to state and prove the
generalization bounds. They do not modify the SGD update
\[
    W_{t+1}=W_t-\eta_tG_t.
\]
They only describe how the virtual perturbation geometry enters the
information-theoretic analysis.

\subsection{Adaptive Gradient-Deviation Term}
\label{subsec:adaptive-gradient-deviation}

Let
\begin{equation}
    \barg(w)
    =
    \E_{Z\sim\cD}
    \left[
        g(w,Z)
    \right]
    \label{eq:population_gradient}
\end{equation}
denote the population gradient at parameter \(w\). For update
\(t=1,\ldots,T-1\), the true SGD recursion uses the minibatch gradient
\begin{equation}
    G_t
    =
    g(W_t,B_t).
\end{equation}
The predictable covariance \(\Sigma_t\) determines the geometry in which the
current stochastic gradient is compared to the population gradient. We define
the adaptive local gradient-deviation term by
\begin{equation}
    V_t^{\mathrm{ad}}
    =
    \E
    \left[
        \normSigma{
            G_t-\barg(W_t)
        }{\Sigma_t^{-1}}^2
        \given
        \cH_{t-1}
    \right],
    \qquad
    t=1,\ldots,T-1.
    \label{eq:adaptive_gradient_variance}
\end{equation}
Here
\begin{equation}
    \normSigma{x}{\Sigma_t^{-1}}^2
    =
    x^\top \Sigma_t^{-1}x.
\end{equation}

Although we retain the notation \(V_t^{\mathrm{ad}}\) by analogy with
variance-type quantities in previous virtual perturbation bounds, the quantity
in \eqref{eq:adaptive_gradient_variance} is, in full generality, a conditional
mean-square deviation from the population gradient. It is a genuine conditional
variance only under a conditional unbiasedness condition. To see this, define the conditional mean gradient
\begin{equation}
    m_t
    =
    \E
    \left[
        G_t
        \given
        \cH_{t-1}
    \right].
    \label{eq:conditional_mean_gradient}
\end{equation}
Since \(\Sigma_t\) is \(\cH_{t-1}\)-measurable, the usual conditional
bias--variance decomposition gives
\begin{align}
    V_t^{\mathrm{ad}}
    &=
    \operatorname{Tr}
    \left(
        \Sigma_t^{-1}
        \operatorname{Cov}
        \left(
            G_t
            \given
            \cH_{t-1}
        \right)
    \right)
    +
    \normSigma{
        m_t-\barg(W_t)
    }{\Sigma_t^{-1}}^2.
    \label{eq:adaptive_gradient_deviation_decomposition}
\end{align}
Thus \(V_t^{\mathrm{ad}}\) contains two contributions. The first term is the
conditional stochastic fluctuation of the minibatch gradient around its
conditional mean. The second term is the squared conditional bias between the
conditional mean gradient and the population gradient.

If
\begin{equation}
    \E
    \left[
        G_t
        \given
        \cH_{t-1}
    \right]
    =
    \barg(W_t),
    \label{eq:conditional_unbiasedness_condition}
\end{equation}
then the bias term in
\eqref{eq:adaptive_gradient_deviation_decomposition} vanishes and
\(V_t^{\mathrm{ad}}\) reduces to the adaptive conditional gradient variance
\begin{equation}
    V_t^{\mathrm{ad}}
    =
    \operatorname{Tr}
    \left(
        \Sigma_t^{-1}
        \operatorname{Cov}
        \left(
            G_t
            \given
            \cH_{t-1}
        \right)
    \right).
    \label{eq:adaptive_gradient_variance_unbiased_case}
\end{equation}
Without \eqref{eq:conditional_unbiasedness_condition}, the more precise
interpretation is conditional mean-square gradient deviation. The expectation in \eqref{eq:adaptive_gradient_variance} is over the current
minibatch-index randomness, the current minibatch values not already revealed by
the history, and any other randomness not included in \(\cH_{t-1}\). This
distinction is important in history-dependent sampling schemes, finite-sample
minibatch sampling, sampling without replacement, adaptive sampling, or settings
where previous observations reveal information about the training sample. In the isotropic special case \(\Sigma_t=\sigma_t^2I\), the adaptive
gradient-deviation term becomes
\begin{equation}
    V_t^{\mathrm{ad}}
    =
    \sigma_t^{-2}
    \E
    \left[
        \norm{
            G_t-\barg(W_t)
        }^2
        \given
        \cH_{t-1}
    \right].
    \label{eq:isotropic_adaptive_variance}
\end{equation}
If the conditional unbiasedness condition
\eqref{eq:conditional_unbiasedness_condition} also holds, then
\eqref{eq:isotropic_adaptive_variance} is the isotropic conditional variance
scaled by \(\sigma_t^{-2}\).

If \(\Sigma_t\) is deterministic, then
\(V_t^{\mathrm{ad}}\) reduces to the corresponding history-conditioned
fixed-geometry mean-square gradient-deviation term. Under conditional
unbiasedness, this is the usual history-conditioned fixed-geometry variance
term. A further reduction to a \(W_t\)-conditioned fixed-noise expression
requires an additional condition ensuring that the conditional law of the
current stochastic gradient depends on the past only through \(W_t\).

\subsection{Adaptive Gradient Sensitivity}
\label{subsec:adaptive-gradient-sensitivity}

The second quantity measures how sensitive the population gradient is to
perturbing the current iterate by the accumulated virtual noise. Recall that
for \(t=1,\ldots,T\),
\begin{equation}
    \Sigma_{1:t}
    =
    \sum_{k=1}^{t-1}\Sigma_k,
    \qquad
    \Sigma_{1:1}=0.
    \label{eq:accumulated_covariance_adaptive_quantities}
\end{equation}
Because \(\Sigma_{1:t}\) is determined by the real SGD history up to time
\(t-1\), it is fixed after conditioning on \(\cH_{t-1}\).

For the synchronized actual-geometry definition below, let
\begin{equation}
    \zeta_t \given \cH_{t-1}
    \sim
    \N(0,\Sigma_{1:t})
    \label{eq:fresh_intermediate_perturbation}
\end{equation}
be a fresh conditional perturbation independent of the current minibatch
randomness given \(\cH_{t-1}\). Equivalently, because the synchronized
precision \(\Sigma_t^{-1}\) is \(\cH_{t-1}\)-measurable, one may use the
accumulated virtual perturbation \(\xi_t\) under its conditional law given
\(\cH_{t-1}\). The fresh-copy notation emphasizes that the sensitivity is a
local Gaussian average around \(W_t\) with covariance \(\Sigma_{1:t}\). We define the adaptive local gradient sensitivity by
\begin{equation}
    \Gamma_t^{\mathrm{ad}}
    =
    \E
    \left[
        \normSigma{
            \barg(W_t+\zeta_t)-\barg(W_t)
        }{\Sigma_t^{-1}}^2
        \given
        \cH_{t-1}
    \right],
    \qquad
    t=1,\ldots,T-1.
    \label{eq:adaptive_gradient_sensitivity}
\end{equation}

This definition mirrors the gradient-sensitivity term in fixed virtual
perturbation analysis, but with two adaptive features. First, the perturbation
around \(W_t\) has accumulated covariance \(\Sigma_{1:t}\), which is random and
history-dependent. Second, the discrepancy between population gradients is
measured in the current inverse covariance geometry \(\Sigma_t^{-1}\). Thus,
\(\Gamma_t^{\mathrm{ad}}\) captures how strongly the population gradient changes
under the accumulated adaptive perturbation, measured in the geometry of the
current virtual noise. If the covariance sequence is deterministic, then
\eqref{eq:adaptive_gradient_sensitivity} reduces after averaging over the SGD
trajectory to the fixed geometry-aware sensitivity
\begin{equation}
    \Gamma_{\Sigma_t,\Sigma_{1:t}}(W_t)
    =
    \E_{\zeta\sim\N(0,\Sigma_{1:t})}
    \left[
        \normSigma{
            \barg(W_t+\zeta)-\barg(W_t)
        }{\Sigma_t^{-1}}^2
    \right]
    \label{eq:fixed_geometry_sensitivity}
\end{equation}
provided the fixed-noise formulation is expressed as a function of \(W_t\).
The history-conditioned version is the more precise object when the sampling
mechanism or conditioning contains information beyond \(W_t\). In the isotropic case \(\Sigma_t=\sigma_t^2I\), this becomes
\begin{equation}
    \Gamma_t^{\mathrm{ad}}
    =
    \sigma_t^{-2}
    \E_{\zeta\sim\N(0,\sigma_{1:t}^2I)}
    \left[
        \norm{
            \barg(W_t+\zeta)-\barg(W_t)
        }^2
        \given
        \cH_{t-1}
    \right],
    \label{eq:isotropic_adaptive_sensitivity}
\end{equation}
where
\begin{equation}
    \sigma_{1:t}^2
    =
    \sum_{k=1}^{t-1}\sigma_k^2.
\end{equation}

\subsection{Reference-Geometry Gradient Deviation and Sensitivity}
\label{subsec:reference-geometry-deviation-sensitivity}

The clean admissibly synchronized covariance bound uses the current precision
matrix \(\Sigma_t^{-1}\) in \eqref{eq:adaptive_gradient_variance} and
\eqref{eq:adaptive_gradient_sensitivity}. For the general covariance-comparison
theorem, however, the actual virtual kernel and the reference virtual kernel
may use different covariances. If the reference covariance at time \(t\) is
denoted by \(\Sigma_t^{\mathrm{ref}}\), the Gaussian mean-comparison term is
naturally weighted by
\[
    (\Sigma_t^{\mathrm{ref}})^{-1},
\]
not necessarily by \(\Sigma_t^{-1}\). To support this general form, we define reference-geometry versions of the two
local quantities. Let
\begin{equation}
    \mathcal K_t
    =
    \cH_{t-1}\vee\mathcal U,
    \label{eq:reference_geometry_Kt}
\end{equation}
where \(\mathcal U\) denotes optional public, auxiliary, or ghost randomness
used by the reference construction and independent of the original sample
\(S\). If no such auxiliary randomness is used, \(\mathcal U\) is the trivial
sigma-field and \(\mathcal K_t=\cH_{t-1}\). Let \(\Lambda_t\succ0\) be the reference precision used in the one-step
comparison. In the general covariance-comparison theorem,
\begin{equation}
    \Lambda_t=(\Sigma_t^{\mathrm{ref}})^{-1}.
    \label{eq:reference_precision_def_adaptive_quantities}
\end{equation}
We define the reference-geometry gradient-deviation term by
\begin{equation}
    V_t^{\mathrm{ad}}(\Lambda_t)
    =
    \E
    \left[
        \normSigma{
            G_t-\barg(W_t)
        }{\Lambda_t}^2
        \given
        \mathcal K_t
    \right].
    \label{eq:adaptive_variance_precision}
\end{equation}
This quantity is a conditional mean-square deviation from the population
gradient in the reference geometry. It should be called a reference-geometry
variance only under the corresponding conditional unbiasedness condition. When \(\Lambda_t\) is \(\mathcal K_t\)-measurable, define
\begin{equation}
    m_t^{\mathcal K}
    =
    \E
    \left[
        G_t
        \given
        \mathcal K_t
    \right].
    \label{eq:reference_geometry_conditional_mean}
\end{equation}
Then
\begin{align}
    V_t^{\mathrm{ad}}(\Lambda_t)
    &=
    \operatorname{Tr}
    \left(
        \Lambda_t
        \operatorname{Cov}
        \left(
            G_t
            \given
            \mathcal K_t
        \right)
    \right)
    +
    \normSigma{
        m_t^{\mathcal K}-\barg(W_t)
    }{\Lambda_t}^2.
    \label{eq:reference_geometry_deviation_decomposition}
\end{align}
Thus, the reference-geometry term also consists of a conditional fluctuation
part and a conditional bias part. If
\[
    \E[G_t\mid\mathcal K_t]=\barg(W_t),
\]
then the second term vanishes and
\(V_t^{\mathrm{ad}}(\Lambda_t)\) becomes a genuine conditional variance in the
reference geometry. The reference-geometry sensitivity term is defined using the realized
accumulated virtual perturbation
\begin{equation}
    \xi_t=\wtW_t-W_t
\end{equation}
whenever \(\Lambda_t\) may depend on the virtual prefix. This avoids implicitly
assuming independence between the precision matrix and the perturbation being
measured. Define
\begin{equation}
    \Gamma_t^{\mathrm{ad}}(\Lambda_t)
    =
    \E
    \left[
        \normSigma{
            \barg(W_t+\xi_t)-\barg(W_t)
        }{\Lambda_t}^2
        \given
        \mathcal K_t
    \right].
    \label{eq:adaptive_sensitivity_precision}
\end{equation}
If \(\Lambda_t\) is \(\mathcal K_t\)-measurable, then the realized
\(\xi_t\) can equivalently be replaced by a fresh conditional Gaussian copy
\[
    \zeta_t\mid\cH_{t-1}\sim\N(0,\Sigma_{1:t}),
\]
giving
\begin{equation}
    \Gamma_t^{\mathrm{ad}}(\Lambda_t)
    =
    \E
    \left[
        \normSigma{
            \barg(W_t+\zeta_t)-\barg(W_t)
        }{\Lambda_t}^2
        \given
        \mathcal K_t
    \right].
    \label{eq:adaptive_sensitivity_precision_fresh_copy}
\end{equation}
Thus, the realized-perturbation definition
\eqref{eq:adaptive_sensitivity_precision} is the general definition, while the
fresh-copy expression \eqref{eq:adaptive_sensitivity_precision_fresh_copy} is a
convenient special case.

The synchronized case corresponds to
\begin{equation}
    \Lambda_t=\Sigma_t^{-1},
\end{equation}
so that
\begin{equation}
    V_t^{\mathrm{ad}}(\Sigma_t^{-1})=V_t^{\mathrm{ad}},
    \qquad
    \Gamma_t^{\mathrm{ad}}(\Sigma_t^{-1})=\Gamma_t^{\mathrm{ad}}.
    \label{eq:synchronized_deviation_sensitivity_identity}
\end{equation}
This identity is used in the clean synchronized corollary only when the
synchronization is admissible in the reference-kernel comparison. A useful simplification occurs when the actual and reference precisions are
comparable. If there exists \(\kappa\geq1\) such that
\begin{equation}
    (\Sigma_t^{\mathrm{ref}})^{-1}
    \preceq
    \kappa\,\Sigma_t^{-1},
    \label{eq:precision_comparability}
\end{equation}
then
\begin{equation}
    V_t^{\mathrm{ad}}\big((\Sigma_t^{\mathrm{ref}})^{-1}\big)
    \leq
    \kappa V_t^{\mathrm{ad}},
    \qquad
    \Gamma_t^{\mathrm{ad}}\big((\Sigma_t^{\mathrm{ref}})^{-1}\big)
    \leq
    \kappa\Gamma_t^{\mathrm{ad}}.
    \label{eq:precision_comparability_terms}
\end{equation}
This comparability condition is useful for stating readable versions of the
general covariance-comparison theorem.

\subsection{Adaptive Output Sensitivity}
\label{subsec:adaptive-output-sensitivity}

The mutual-information part of the proof controls the generalization error of a
perturbed output. To transfer this control back to the original SGD output, we
must account for the loss difference between \(W_T\) and a perturbation of
\(W_T\). In the adaptive setting, the accumulated covariance at the final output
is
\begin{equation}
    \Sigma_{1:T}
    =
    \sum_{k=1}^{T-1}\Sigma_k,
    \label{eq:final_accumulated_covariance}
\end{equation}
which is random and determined by the real SGD history \(\cH_{T-1}\). For any deterministic sample \(s=(z_1,\ldots,z_n)\), define the adaptive output
sensitivity by
\begin{equation}
    \Delta_{\Sigma_{1:T}}^{\mathrm{ad}}(W_T,s)
    =
    \E
    \left[
        L(W_T,s)
        -
        L(W_T+\zeta_T,s)
        \given
        \cH_{T-1}
    \right],
    \label{eq:adaptive_output_sensitivity}
\end{equation}
where
\begin{equation}
    \zeta_T \given \cH_{T-1}
    \sim
    \N(0,\Sigma_{1:T})
    \label{eq:fresh_final_perturbation}
\end{equation}
is an independent final perturbation drawn conditionally on the accumulated
adaptive covariance. The use of the fresh perturbation \(\zeta_T\) rather than the realized
accumulated perturbation \(\xi_T\) is deliberate. After conditioning on an
augmented history that contains the realized perturbations, \(\xi_T\) would be
fixed. The quantity in \eqref{eq:adaptive_output_sensitivity} instead measures
the local sensitivity of the loss around \(W_T\) under a fresh perturbation with
the same accumulated adaptive covariance. Thus, it is a conditional analogue of
the fixed-noise output sensitivity used in virtual perturbation analysis. The final perturbation-sensitivity penalty appearing in the main theorem is
\begin{equation}
    \mathcal R_{\Delta}^{\mathrm{ad}}
    =
    \left|
    \E
    \left[
        \Delta_{\Sigma_{1:T}}^{\mathrm{ad}}(W_T,S')
        -
        \Delta_{\Sigma_{1:T}}^{\mathrm{ad}}(W_T,S)
    \right]
    \right|.
    \label{eq:adaptive_final_penalty}
\end{equation}
Here \(S'\) denotes an independent ghost sample drawn from the same population
distribution as \(S\). This term quantifies the cost of transferring a
generalization guarantee from the perturbed output back to the original SGD
output. When the loss is insensitive to perturbations around \(W_T\) in the
accumulated adaptive covariance geometry, this term is small. Later smoothness
or local curvature assumptions can be used to control
\(\mathcal R_{\Delta}^{\mathrm{ad}}\) in terms of trace or
curvature-weighted quantities involving \(\Sigma_{1:T}\).

\subsection{Interpretation and Comparison with Fixed-Noise Quantities}
\label{subsec:adaptive-quantities-interpretation}

The adaptive quantities introduced above are analogues of the gradient
deviation, gradient-sensitivity, and output-sensitivity terms that appear in
fixed virtual perturbation bounds. The main difference is that the covariance
geometry is now random, predictable, and history-dependent.

\begin{center}
\begin{tabular}{lll}
\toprule
\textbf{Role} & \textbf{Fixed-noise quantity} & \textbf{Adaptive quantity} \\
\midrule
Gradient deviation
&
\(V_{t,\Sigma_t}^{\mathrm{hist}}\) or \(V_{t,\Sigma_t}(W_t)\)
&
\(V_t^{\mathrm{ad}}\)
\\[0.25em]
Gradient sensitivity
&
\(\Gamma_{\Sigma_t,\Sigma_{1:t}}^{\mathrm{hist}}\) or
\(\Gamma_{\Sigma_t,\Sigma_{1:t}}(W_t)\)
&
\(\Gamma_t^{\mathrm{ad}}\)
\\[0.25em]
Output sensitivity
&
\(\Delta_{\Sigma_{1:T}}(W_T,s)\)
&
\(\Delta_{\Sigma_{1:T}}^{\mathrm{ad}}(W_T,s)\)
\\[0.25em]
Reference-geometry deviation
&
not needed in the admissibly synchronized fixed-noise case
&
\(V_t^{\mathrm{ad}}(\Lambda_t)\)
\\[0.25em]
Reference-geometry sensitivity
&
not needed in the admissibly synchronized fixed-noise case
&
\(\Gamma_t^{\mathrm{ad}}(\Lambda_t)\)
\\
\bottomrule
\end{tabular}
\end{center}

The adaptive gradient-deviation term \(V_t^{\mathrm{ad}}\) measures the
conditional mean-square discrepancy between the current stochastic gradient and
the population gradient in the inverse covariance geometry chosen from the
past. Under conditional unbiasedness, it reduces to an adaptive conditional
gradient variance. Without conditional unbiasedness, it also includes the
squared conditional bias term shown in
\eqref{eq:adaptive_gradient_deviation_decomposition}.

The adaptive gradient sensitivity \(\Gamma_t^{\mathrm{ad}}\) measures how much
the population gradient changes under accumulated adaptive perturbations. The
adaptive output sensitivity
\(\Delta_{\Sigma_{1:T}}^{\mathrm{ad}}\) measures local flatness or perturbation
sensitivity of the final output under the accumulated adaptive covariance. The reference-geometry quantities
\(V_t^{\mathrm{ad}}(\Lambda_t)\) and
\(\Gamma_t^{\mathrm{ad}}(\Lambda_t)\) support the general
covariance-comparison theorem. They are not needed in the clean admissibly
synchronized covariance corollary, where \(\Lambda_t=\Sigma_t^{-1}\). When
actual and reference covariances differ, these weighted quantities and the
associated covariance-comparison term together describe the cost of the
one-step comparison. These are all analytical quantities used to state and prove generalization
bounds. They are not used to update the SGD iterates. The true SGD recursion
remains
\begin{equation}
    W_{t+1}=W_t-\eta_tG_t.
\end{equation}
When \(\Sigma_t\) is deterministic for every \(t\), the adaptive definitions
reduce to their fixed covariance counterparts in the history-conditioned sense.
Under additional sampling assumptions, these may be further expressed as
functions of \(W_t\) alone. When \(\Sigma_t=\sigma_t^2I\), the framework reduces
further to the scalar isotropic perturbation setting. Thus, the adaptive
quantities extend the fixed-noise quantities while preserving their conceptual
roles in the generalization analysis.

\section{Main Theoretical Results}
\label{sec:main_theoretical_results}

We now present the main information-theoretic generalization bound for SGD with
predictable virtual perturbations. The central distinction in this section is
between the covariance used by the actual virtual perturbation and the covariance
available to the reference kernel in the mutual-information decomposition. Let
\[
    S=(Z_1,\ldots,Z_n)
\]
be the training sample. For a loss function \(\ell(w,z)\), define the population
and empirical risks by
\[
    L(w)=\mathbb E_Z[\ell(w,Z)],
    \qquad
    L_S(w)=\frac1n\sum_{i=1}^n \ell(w,Z_i).
\]
We write
\[
    \bar g(w)=\nabla L(w)
\]
for the population gradient. The expected generalization error of an algorithmic
output \(W\) is
\[
    \operatorname{gen}(W,S)
    =
    \mathbb E\bigl[L(W)-L_S(W)\bigr].
\]

The real SGD recursion is
\begin{equation}
    W_{t+1}=W_t-\eta_tG_t,
    \qquad
    G_t=g(W_t,B_t),
    \qquad
    t=1,\ldots,T-1,
    \label{eq:real_sgd_update_main}
\end{equation}
where \(B_t\) denotes the minibatch randomness at time \(t\). The virtual
trajectory is
\begin{equation}
    \widetilde W_{t+1}
    =
    \widetilde W_t-\eta_tG_t+\varepsilon_t,
    \qquad
    \widetilde W_1=W_1,
    \label{eq:virtual_sgd_update_main}
\end{equation}
where
\[
    \varepsilon_t\mid\mathcal H_{t-1}
    \sim
    \mathcal N(0,\Sigma_t).
\]
Here \(\mathcal H_{t-1}\) denotes the real SGD history up to time \(t-1\),
excluding the virtual perturbation innovations, and
\[
    \Sigma_t\succ0
\]
is assumed to be \(\mathcal H_{t-1}\)-measurable. Thus \(\Sigma_t\) is
predictable: it may depend on the past real optimization history, but not on the
current minibatch randomness, the current Gaussian innovation, or future
randomness.

We define the accumulated virtual perturbation by
\begin{equation}
    \xi_t=\widetilde W_t-W_t.
    \label{eq:accumulated_virtual_perturbation}
\end{equation}
Since \(\xi_1=0\) and
\[
    \xi_{t+1}=\xi_t+\varepsilon_t,
\]
the perturbation at time \(t\) is generated by the previous virtual innovations.

For a positive semidefinite matrix \(M\), we use the notation
\[
    \|x\|_M^2=x^\top Mx.
\]

\subsection{Conditional Gaussian Relative-Entropy Bounds}
\label{subsec:conditional_gaussian_relative_entropy_bounds}

We first record the conditional Gaussian comparison inequalities used throughout
the proof. These are standard consequences of the formula for the relative
entropy between Gaussian measures and convexity of relative entropy. Throughout this subsection, all conditional distributions are understood as
regular conditional distributions, and all random variables are assumed to take
values in standard Borel spaces.

\begin{lemma}[Conditional Gaussian smoothing with common covariance]
\label{lem:conditional_gaussian_common_covariance}
Let \(\mathcal F\) be a sigma-field. Let \(X\) and \(Y\) be conditionally
square-integrable random vectors in \(\mathbb R^d\). Let
\[
    \varepsilon\mid\mathcal F\sim\mathcal N(0,\Sigma),
\]
where \(\Sigma\succ0\) is \(\mathcal F\)-measurable, and assume that
\(\varepsilon\) is conditionally independent of \((X,Y)\) given \(\mathcal F\).
Then
\begin{equation}
    D_{\mathrm{KL}}
    \left(
        P_{X+\varepsilon\mid\mathcal F}
        \,\middle\|\,
        P_{Y+\varepsilon\mid\mathcal F}
    \right)
    \leq
    \frac12
    \mathbb E
    \left[
        \|X-Y\|_{\Sigma^{-1}}^2
        \,\middle|\,
        \mathcal F
    \right].
    \label{eq:conditional_gaussian_common_covariance}
\end{equation}
\end{lemma}

\begin{proof}
Condition on \(\mathcal F\). For fixed \(x,y\), the Gaussian relative entropy is
\[
    D_{\mathrm{KL}}
    \left(
        \mathcal N(x,\Sigma)
        \,\middle\|\,
        \mathcal N(y,\Sigma)
    \right)
    =
    \frac12\|x-y\|_{\Sigma^{-1}}^2.
\]
The conditional distributions of \(X+\varepsilon\) and \(Y+\varepsilon\) are
mixtures of these Gaussian measures. Applying joint convexity of relative
entropy to any conditional coupling of \((X,Y)\) gives
\eqref{eq:conditional_gaussian_common_covariance}.
\end{proof}

\begin{lemma}[Conditional Gaussian comparison with covariance mismatch]
\label{lem:conditional_gaussian_covariance_mismatch}
Let \(\mathcal F\) be a sigma-field. Let \(X\) be a conditionally
square-integrable random vector in \(\mathbb R^d\). Let
\[
    \varepsilon\mid\mathcal F\sim\mathcal N(0,\Sigma),
\]
where \(\Sigma\succ0\) is \(\mathcal F\)-measurable, and assume that
\(\varepsilon\) is conditionally independent of \(X\) given \(\mathcal F\).
Let \(m\) and \(\Sigma^{\mathrm{ref}}\succ0\) be \(\mathcal F\)-measurable.
Then
\begin{align}
    &
    D_{\mathrm{KL}}
    \left(
        P_{X+\varepsilon\mid\mathcal F}
        \,\middle\|\,
        \mathcal N(m,\Sigma^{\mathrm{ref}})
    \right)
    \nonumber\\
    &\qquad\leq
    \frac12
    \mathbb E
    \left[
        \|X-m\|_{(\Sigma^{\mathrm{ref}})^{-1}}^2
        \,\middle|\,
        \mathcal F
    \right]
    \nonumber\\
    &\qquad\quad+
    \frac12
    \left[
        \operatorname{Tr}
        \left(
            (\Sigma^{\mathrm{ref}})^{-1}\Sigma
        \right)
        -d
        +
        \log
        \frac{\det\Sigma^{\mathrm{ref}}}{\det\Sigma}
    \right].
    \label{eq:conditional_gaussian_covariance_mismatch}
\end{align}
\end{lemma}

\begin{proof}
Condition on \(\mathcal F\). For fixed \(x\),
\begin{align*}
    &
    D_{\mathrm{KL}}
    \left(
        \mathcal N(x,\Sigma)
        \,\middle\|\,
        \mathcal N(m,\Sigma^{\mathrm{ref}})
    \right)
    \\
    &\qquad=
    \frac12
    \|x-m\|_{(\Sigma^{\mathrm{ref}})^{-1}}^2
    +
    \frac12
    \left[
        \operatorname{Tr}
        \left(
            (\Sigma^{\mathrm{ref}})^{-1}\Sigma
        \right)
        -d
        +
        \log
        \frac{\det\Sigma^{\mathrm{ref}}}{\det\Sigma}
    \right].
\end{align*}
Since \(P_{X+\varepsilon\mid\mathcal F}\) is a mixture of
\(\mathcal N(x,\Sigma)\) over the conditional law of \(X\), convexity of
relative entropy in its first argument yields the claim.
\end{proof}

\subsection{Mutual-Information Decomposition}
\label{subsec:mutual_information_decomposition}

The information-theoretic generalization bound controls the expected
generalization error of a randomized output by its mutual information with the
sample. If \(\ell(w,Z)\) is \(R\)-sub-Gaussian in \(Z\) for every fixed \(w\),
then
\begin{equation}
    \left|
        \operatorname{gen}(\widetilde W_T,S)
    \right|
    \leq
    \sqrt{
        \frac{2R^2}{n}
        I(\widetilde W_T;S)
    }.
    \label{eq:standard_information_gen_bound}
\end{equation}

Let \(\mathcal U\) denote optional public, auxiliary, or ghost randomness used
by the reference construction. We assume that
\[
    \mathcal U
    \quad\text{is independent of}\quad
    S.
\]
If no such auxiliary randomness is used, \(\mathcal U\) is the trivial
sigma-field.

For each time \(t\), define the reference-visible sigma-field
\begin{equation}
    \mathcal F_t^Q
    =
    \sigma(\widetilde W_{1:t})\vee\mathcal U.
    \label{eq:reference_visible_sigma_field_main}
\end{equation}
The reference kernel at time \(t\) is allowed to be measurable with respect to
\(\mathcal F_t^Q\), but it is not allowed to condition directly on the sample
\(S\). Define also
\begin{equation}
    \mathcal A_t
    =
    \mathcal F_t^Q\vee\sigma(S).
    \label{eq:coarse_sigma_field_main}
\end{equation}

\begin{lemma}[Reference-kernel mutual-information decomposition]
\label{lem:reference_kernel_mi_decomposition}
Assume \(W_1\) is independent of \(S\). For each \(t=1,\ldots,T-1\), let
\(Q_{t+1\mid\mathcal F_t^Q}\) be any probability kernel measurable with respect
to \(\mathcal F_t^Q\). Then
\begin{align}
    I(\widetilde W_T;S)
    \leq
    \sum_{t=1}^{T-1}
    \mathbb E
    \left[
        D_{\mathrm{KL}}
        \left(
            P_{\widetilde W_{t+1}\mid\mathcal A_t}
            \,\middle\|\,
            Q_{t+1\mid\mathcal F_t^Q}
        \right)
    \right].
    \label{eq:reference_kernel_mi_decomposition}
\end{align}
\end{lemma}

\begin{proof}
Because \(\mathcal U\) is independent of \(S\),
\[
    I(\widetilde W_T;S)
    \leq
    I(\widetilde W_T;S\mid\mathcal U)
    \leq
    I(\widetilde W_{1:T};S\mid\mathcal U).
\]
By the chain rule for mutual information and the independence of \(W_1\) and
\(S\),
\begin{align}
    I(\widetilde W_{1:T};S\mid\mathcal U)
    =
    \sum_{t=1}^{T-1}
    I(\widetilde W_{t+1};S
    \mid
    \widetilde W_{1:t},\mathcal U).
    \label{eq:chain_rule_mi_main}
\end{align}
For each \(t\),
\[
    I(\widetilde W_{t+1};S
    \mid
    \widetilde W_{1:t},\mathcal U)
    =
    \mathbb E
    \left[
        D_{\mathrm{KL}}
        \left(
            P_{\widetilde W_{t+1}\mid\mathcal A_t}
            \,\middle\|\,
            P_{\widetilde W_{t+1}\mid\mathcal F_t^Q}
        \right)
    \right].
\]
For any \(\mathcal F_t^Q\)-measurable reference kernel
\(Q_{t+1\mid\mathcal F_t^Q}\), the conditional KL identity gives
\begin{align}
    &
    \mathbb E
    \left[
        D_{\mathrm{KL}}
        \left(
            P_{\widetilde W_{t+1}\mid\mathcal A_t}
            \,\middle\|\,
            Q_{t+1\mid\mathcal F_t^Q}
        \right)
    \right]
    \nonumber\\
    &\quad=
    I(\widetilde W_{t+1};S
    \mid
    \widetilde W_{1:t},\mathcal U)
    +
    \mathbb E
    \left[
        D_{\mathrm{KL}}
        \left(
            P_{\widetilde W_{t+1}\mid\mathcal F_t^Q}
            \,\middle\|\,
            Q_{t+1\mid\mathcal F_t^Q}
        \right)
    \right]
    \nonumber\\
    &\quad\geq
    I(\widetilde W_{t+1};S
    \mid
    \widetilde W_{1:t},\mathcal U).
    \label{eq:kl_reference_upper_bound_mi}
\end{align}
Summing over \(t\) proves the result.
\end{proof}

\subsection{From History-Conditioned Comparisons to Virtual-Prefix KL}
\label{subsec:conditioning_transfer}

The one-step KL term in
Lemma~\ref{lem:reference_kernel_mi_decomposition} is conditioned only on the
virtual prefix and the sample:
\[
    P_{\widetilde W_{t+1}\mid\mathcal A_t}
    =
    P_{\widetilde W_{t+1}
    \mid
    \widetilde W_{1:t},S,\mathcal U}.
\]
However, the Gaussian comparison in
Lemma~\ref{lem:conditional_gaussian_covariance_mismatch} is naturally applied
after conditioning on a richer sigma-field containing the real SGD history. This
richer conditioning reveals \(W_t\), the predictable covariance \(\Sigma_t\),
and the accumulated perturbation relation
\[
    \widetilde W_t=W_t+\xi_t.
\]
Define the richer sigma-field
\begin{equation}
    \mathcal B_t
    =
    \mathcal A_t\vee\mathcal H_{t-1}.
    \label{eq:rich_sigma_field_main}
\end{equation}
Then
\[
    \mathcal A_t\subseteq\mathcal B_t.
\]

\begin{lemma}[Conditioning compression for relative entropy]
\label{lem:conditional_kl_coarsening}
Let \(X\) be a random element, and let
\[
    \mathcal A\subseteq\mathcal B
\]
be sigma-fields. Let \(Q\) be an \(\mathcal A\)-measurable probability kernel on
the state space of \(X\). Then
\begin{equation}
    \mathbb E
    \left[
        D_{\mathrm{KL}}
        \left(
            P_{X\mid\mathcal A}
            \,\middle\|\,
            Q
        \right)
    \right]
    \leq
    \mathbb E
    \left[
        D_{\mathrm{KL}}
        \left(
            P_{X\mid\mathcal B}
            \,\middle\|\,
            Q
        \right)
    \right].
    \label{eq:conditional_kl_coarsening}
\end{equation}
\end{lemma}

\begin{proof}
Let
\[
    \mu_{\mathcal A}=P_{X\mid\mathcal A},
    \qquad
    \mu_{\mathcal B}=P_{X\mid\mathcal B}.
\]
By the tower property for regular conditional distributions,
\[
    \mu_{\mathcal A}
    =
    \mathbb E[\mu_{\mathcal B}\mid\mathcal A],
\]
where the equality is understood weakly after testing against bounded measurable
functions. Since \(Q\) is \(\mathcal A\)-measurable, it is fixed after conditioning on
\(\mathcal A\). By convexity of relative entropy in its first argument,
\[
    D_{\mathrm{KL}}
    \left(
        \mu_{\mathcal A}
        \,\middle\|\,
        Q
    \right)
    \leq
    \mathbb E
    \left[
        D_{\mathrm{KL}}
        \left(
            \mu_{\mathcal B}
            \,\middle\|\,
            Q
        \right)
        \,\middle|\,
        \mathcal A
    \right].
\]
Taking expectations gives \eqref{eq:conditional_kl_coarsening}.
\end{proof}

\begin{corollary}[History-to-prefix KL transfer]
\label{cor:history_to_prefix_kl_transfer}
For every \(t\), let \(Q_{t+1\mid\mathcal F_t^Q}\) be an
\(\mathcal F_t^Q\)-measurable reference kernel. Then
\begin{align}
    &
    \mathbb E
    \left[
        D_{\mathrm{KL}}
        \left(
            P_{\widetilde W_{t+1}\mid\mathcal A_t}
            \,\middle\|\,
            Q_{t+1\mid\mathcal F_t^Q}
        \right)
    \right]
    \nonumber\\
    &\qquad\leq
    \mathbb E
    \left[
        D_{\mathrm{KL}}
        \left(
            P_{\widetilde W_{t+1}\mid\mathcal B_t}
            \,\middle\|\,
            Q_{t+1\mid\mathcal F_t^Q}
        \right)
    \right].
    \label{eq:history_to_prefix_transfer}
\end{align}
\end{corollary}

\begin{proof}
Apply Lemma~\ref{lem:conditional_kl_coarsening} with
\[
    X=\widetilde W_{t+1},
    \qquad
    \mathcal A=\mathcal A_t,
    \qquad
    \mathcal B=\mathcal B_t,
    \qquad
    Q=Q_{t+1\mid\mathcal F_t^Q}.
\]
Since \(\mathcal F_t^Q\subseteq\mathcal A_t\), the reference kernel is
\(\mathcal A_t\)-measurable.
\end{proof}

\begin{remark}
Corollary~\ref{cor:history_to_prefix_kl_transfer} does not make arbitrary
data-dependent reference kernels valid. The reference kernel must still be
constructed only from the virtual prefix and admissible sample-independent
auxiliary information. The richer sigma-field \(\mathcal B_t\) is used for the
analysis, not for giving the reference kernel illegal access to \(S\).
\end{remark}

\subsection{Conditioning-Safe Adaptive Gradient-Deviation and Sensitivity}
\label{subsec:conditioning_safe_adaptive_quantities}

Let
\[
    \Sigma_t^{\mathrm{ref}}\succ0
\]
be the reference covariance used by the reference kernel at time \(t\), and
define the reference precision
\begin{equation}
    \Lambda_t
    =
    (\Sigma_t^{\mathrm{ref}})^{-1}.
    \label{eq:reference_precision_main}
\end{equation}

When \(\Lambda_t\) is deterministic or measurable with respect to the real
history and admissible reference-side randomness, the adaptive sensitivity can
be expressed using an independent fresh Gaussian copy of the accumulated
perturbation. For prefix-observable adaptive covariances, however,
\(\Lambda_t\) may depend on the realized virtual prefix
\(\widetilde W_{1:t}\), and therefore on the realized perturbation \(\xi_t\).
In that case, the fresh-copy notation can be misleading. The conditioning-safe
definition uses the realized perturbation. Let
\[
    \mathcal K_t
    =
    \mathcal H_{t-1}\vee\mathcal U,
\]
where \(\mathcal U\) denotes optional public, auxiliary, or ghost randomness
independent of the original sample \(S\). If no such randomness is used,
\(\mathcal U\) is the trivial sigma-field. Define the reference-geometry gradient-deviation term by
\begin{equation}
    V_t^{\mathrm{ad}}(\Lambda_t)
    =
    \mathbb E
    \left[
        \left\|
            G_t-\bar g(W_t)
        \right\|_{\Lambda_t}^{2}
        \,\middle|\,
        \mathcal K_t
    \right].
    \label{eq:conditioning_safe_V_main}
\end{equation}
This quantity is called a variance term only by analogy with earlier work. In
full generality it is a conditional mean-square deviation from the population
gradient, and it becomes a conditional variance only under the corresponding
conditional unbiasedness condition. Define the reference-geometry virtual-sensitivity term by
\begin{equation}
    \Gamma_t^{\mathrm{ad}}(\Lambda_t)
    =
    \mathbb E
    \left[
        \left\|
            \bar g(W_t+\xi_t)-\bar g(W_t)
        \right\|_{\Lambda_t}^{2}
        \,\middle|\,
        \mathcal K_t
    \right].
    \label{eq:conditioning_safe_Gamma_main}
\end{equation}
The conditional expectation averages over the virtual perturbation history
whenever that history is not already revealed by the conditioning sigma-field. When \(\Lambda_t\) is \(\mathcal K_t\)-measurable and does not depend on the
realized perturbation beyond \(\mathcal K_t\), one may equivalently replace
\(\xi_t\) by an independent fresh conditional Gaussian copy
\[
    \zeta_t\mid\mathcal H_{t-1}
    \sim
    \mathcal N(0,\Sigma_{1:t}),
    \qquad
    \Sigma_{1:t}
    =
    \sum_{s=1}^{t-1}\Sigma_s.
\]
In this special case,
\begin{equation}
    \Gamma_t^{\mathrm{ad}}(\Lambda_t)
    =
    \mathbb E
    \left[
        \left\|
            \bar g(W_t+\zeta_t)-\bar g(W_t)
        \right\|_{\Lambda_t}^{2}
        \,\middle|\,
        \mathcal K_t
    \right].
    \label{eq:fresh_copy_Gamma_main}
\end{equation}
Thus, \eqref{eq:conditioning_safe_Gamma_main} is the default general definition,
and \eqref{eq:fresh_copy_Gamma_main} is a convenient special case. For the synchronized actual geometry, we write
\begin{equation}
    V_t^{\mathrm{ad}}
    =
    V_t^{\mathrm{ad}}(\Sigma_t^{-1}),
    \qquad
    \Gamma_t^{\mathrm{ad}}
    =
    \Gamma_t^{\mathrm{ad}}(\Sigma_t^{-1}).
    \label{eq:synchronized_V_Gamma_notation}
\end{equation}

\subsection{Certified Admissible Reference Comparisons}
\label{subsec:certified_admissible_reference}

In the fixed-covariance setting, the reference kernel can use the same
perturbation covariance as the actual virtual kernel because the covariance
schedule is deterministic and therefore does not encode information about the
training sample. In the history-adaptive setting, this point is more delicate.
The actual covariance \(\Sigma_t\) is predictable with respect to the real SGD
history \(\mathcal H_{t-1}\), but it may still be data-dependent through the
sample values observed along the trajectory. Thus, predictability is sufficient for the local conditional Gaussian smoothing
step, but it is not sufficient by itself to justify a synchronized reference
kernel in the mutual-information decomposition. The reference kernel must be
constructed from information it is allowed to see. The canonical reference drift is chosen as the population-gradient drift at the
current virtual iterate:
\begin{equation}
    \widetilde W_t-\eta_t\bar g(\widetilde W_t).
    \label{eq:canonical_reference_drift_main}
\end{equation}
Given a positive definite reference covariance
\(\Sigma_t^{\mathrm{ref}}\), the canonical Gaussian reference kernel is
\begin{equation}
    Q_{t+1\mid\mathcal F_t^Q}
    =
    \mathcal N
    \left(
        \widetilde W_t-\eta_t\bar g(\widetilde W_t),
        \Sigma_t^{\mathrm{ref}}
    \right).
    \label{eq:canonical_reference_kernel_main}
\end{equation}

\begin{definition}[\(S\)-admissible reference covariance]
\label{def:S_admissible_reference_covariance}
A reference covariance \(\Sigma_t^{\mathrm{ref}}\) is called
\(S\)-admissible at time \(t\) if the following conditions hold:
\begin{enumerate}
    \item \(\Sigma_t^{\mathrm{ref}}\succ0\) almost surely.

    \item \(\Sigma_t^{\mathrm{ref}}\) is measurable with respect to
    \[
        \mathcal F_t^Q
        =
        \sigma(\widetilde W_{1:t})\vee\mathcal U,
    \]
    where \(\mathcal U\) is independent of the original sample \(S\). The
    auxiliary sigma-field \(\mathcal U\) may include public randomness or an
    independent ghost construction.

    \item The Gaussian innovation used by the reference kernel is independent of
    the current minibatch randomness and of the actual Gaussian innovation
    conditional on the information used to construct
    \(\Sigma_t^{\mathrm{ref}}\).
\end{enumerate}
\end{definition}

The covariance-comparison cost associated with an \(S\)-admissible reference
covariance is
\begin{equation}
    \mathcal C_t^{\mathrm{cov}}
    =
    \frac12
    \mathbb E
    \left[
        \operatorname{Tr}
        \left(
            (\Sigma_t^{\mathrm{ref}})^{-1}\Sigma_t
        \right)
        -d
        +
        \log
        \frac{\det\Sigma_t^{\mathrm{ref}}}{\det\Sigma_t}
    \right].
    \label{eq:covariance_comparison_cost_main}
\end{equation}
This is exactly the Gaussian covariance KL cost of comparing the actual
smoothing covariance \(\Sigma_t\) with the reference smoothing covariance
\(\Sigma_t^{\mathrm{ref}}\).

\subsection{Conditioning-Complete Certified One-Step Comparison}
\label{subsec:conditioning_complete_onestep}

We now prove the one-step reference comparison with the conditioning transition
made explicit.

\begin{proposition}[Certified one-step Gaussian reference comparison]
\label{prop:certified_onestep_reference}
Fix \(t\in\{1,\ldots,T-1\}\). Suppose that
\(\Sigma_t^{\mathrm{ref}}\) is \(S\)-admissible and that the reference kernel is
the canonical Gaussian kernel
\[
    Q_{t+1\mid\mathcal F_t^Q}
    =
    \mathcal N
    \left(
        \widetilde W_t-\eta_t\bar g(\widetilde W_t),
        \Sigma_t^{\mathrm{ref}}
    \right).
\]
Let
\[
    \Lambda_t=(\Sigma_t^{\mathrm{ref}})^{-1}.
\]
Then
\begin{align}
    &
    \mathbb E
    \left[
        D_{\mathrm{KL}}
        \left(
            P_{\widetilde W_{t+1}\mid\mathcal A_t}
            \,\middle\|\,
            Q_{t+1\mid\mathcal F_t^Q}
        \right)
    \right]
    \nonumber\\
    &\qquad\leq
    2\eta_t^2
    \mathbb E
    \left[
        V_t^{\mathrm{ad}}(\Lambda_t)
        +
        \Gamma_t^{\mathrm{ad}}(\Lambda_t)
    \right]
    +
    \mathcal C_t^{\mathrm{cov}}.
    \label{eq:certified_onestep_reference_bound_main}
\end{align}
\end{proposition}

\begin{proof}
The proof has three steps.

\paragraph{Step 1: Transfer from virtual-prefix conditioning to rich conditioning.}
By Corollary~\ref{cor:history_to_prefix_kl_transfer},
\begin{align}
    &
    \mathbb E
    \left[
        D_{\mathrm{KL}}
        \left(
            P_{\widetilde W_{t+1}\mid\mathcal A_t}
            \,\middle\|\,
            Q_{t+1\mid\mathcal F_t^Q}
        \right)
    \right]
    \nonumber\\
    &\qquad\leq
    \mathbb E
    \left[
        D_{\mathrm{KL}}
        \left(
            P_{\widetilde W_{t+1}\mid\mathcal B_t}
            \,\middle\|\,
            Q_{t+1\mid\mathcal F_t^Q}
        \right)
    \right].
    \label{eq:proof_conditioning_transfer_main}
\end{align}

\paragraph{Step 2: Conditional Gaussian comparison under the rich sigma-field.}
Condition on \(\mathcal B_t\). Under this conditioning,
\(\widetilde W_t\), \(W_t\), \(\Sigma_t\), and
\(\Sigma_t^{\mathrm{ref}}\) are fixed. The current minibatch gradient \(G_t\)
may still be random, but the Gaussian innovation \(\varepsilon_t\) is
conditionally Gaussian with covariance \(\Sigma_t\) and is independent of the
current minibatch randomness. Let
\[
    M_t=\widetilde W_t-\eta_tG_t,
    \qquad
    m_t^{\mathrm{ref}}
    =
    \widetilde W_t-\eta_t\bar g(\widetilde W_t).
\]
By Lemma~\ref{lem:conditional_gaussian_covariance_mismatch}, with an additional
application of convexity to average over the conditional distribution of
\(G_t\),
\begin{align}
    &
    D_{\mathrm{KL}}
    \left(
        P_{\widetilde W_{t+1}\mid\mathcal B_t}
        \,\middle\|\,
        Q_{t+1\mid\mathcal F_t^Q}
    \right)
    \nonumber\\
    &\qquad\leq
    \frac{\eta_t^2}{2}
    \mathbb E
    \left[
        \left\|
            G_t-\bar g(\widetilde W_t)
        \right\|_{\Lambda_t}^{2}
        \,\middle|\,
        \mathcal B_t
    \right]
    \nonumber\\
    &\qquad\quad+
    \frac12
    \left[
        \operatorname{Tr}
        \left(
            (\Sigma_t^{\mathrm{ref}})^{-1}\Sigma_t
        \right)
        -d
        +
        \log
        \frac{\det\Sigma_t^{\mathrm{ref}}}{\det\Sigma_t}
    \right].
    \label{eq:rich_conditioned_gaussian_comparison_main}
\end{align}
Taking expectations gives
\begin{align}
    &
    \mathbb E
    \left[
        D_{\mathrm{KL}}
        \left(
            P_{\widetilde W_{t+1}\mid\mathcal B_t}
            \,\middle\|\,
            Q_{t+1\mid\mathcal F_t^Q}
        \right)
    \right]
    \nonumber\\
    &\qquad\leq
    \frac{\eta_t^2}{2}
    \mathbb E
    \left[
        \left\|
            G_t-\bar g(\widetilde W_t)
        \right\|_{\Lambda_t}^{2}
    \right]
    +
    \mathcal C_t^{\mathrm{cov}}.
    \label{eq:rich_conditioned_after_expectation_main}
\end{align}

\paragraph{Step 3: Decompose the gradient discrepancy.}
Since
\[
    \widetilde W_t=W_t+\xi_t,
\]
we have
\[
    \bar g(\widetilde W_t)=\bar g(W_t+\xi_t).
\]
Therefore,
\begin{align}
    G_t-\bar g(\widetilde W_t)
    =
    \left(G_t-\bar g(W_t)\right)
    +
    \left(\bar g(W_t)-\bar g(W_t+\xi_t)\right).
    \label{eq:gradient_discrepancy_decomposition_main}
\end{align}
Using
\[
    \|a+b\|_{\Lambda_t}^2
    \leq
    2\|a\|_{\Lambda_t}^2
    +
    2\|b\|_{\Lambda_t}^2,
\]
we obtain
\begin{align}
    &
    \mathbb E
    \left[
        \left\|
            G_t-\bar g(\widetilde W_t)
        \right\|_{\Lambda_t}^{2}
    \right]
    \nonumber\\
    &\qquad\leq
    2
    \mathbb E
    \left[
        \left\|
            G_t-\bar g(W_t)
        \right\|_{\Lambda_t}^{2}
    \right]
    +
    2
    \mathbb E
    \left[
        \left\|
            \bar g(W_t+\xi_t)-\bar g(W_t)
        \right\|_{\Lambda_t}^{2}
    \right]
    \nonumber\\
    &\qquad=
    2
    \mathbb E
    \left[
        V_t^{\mathrm{ad}}(\Lambda_t)
        +
        \Gamma_t^{\mathrm{ad}}(\Lambda_t)
    \right],
    \label{eq:gradient_discrepancy_to_V_Gamma_main}
\end{align}
where the last equality follows from
\eqref{eq:conditioning_safe_V_main},
\eqref{eq:conditioning_safe_Gamma_main}, and the tower property.

Combining
\eqref{eq:proof_conditioning_transfer_main},
\eqref{eq:rich_conditioned_after_expectation_main}, and
\eqref{eq:gradient_discrepancy_to_V_Gamma_main} gives the sharper bound
\[
    \mathbb E
    \left[
        D_{\mathrm{KL}}
        \left(
            P_{\widetilde W_{t+1}\mid\mathcal A_t}
            \,\middle\|\,
            Q_{t+1\mid\mathcal F_t^Q}
        \right)
    \right]
    \leq
    \eta_t^2
    \mathbb E
    \left[
        V_t^{\mathrm{ad}}(\Lambda_t)
        +
        \Gamma_t^{\mathrm{ad}}(\Lambda_t)
    \right]
    +
    \mathcal C_t^{\mathrm{cov}}.
\]
The stated version with coefficient \(2\eta_t^2\) follows immediately since the
gradient-deviation and sensitivity terms are nonnegative. We keep the looser
coefficient to match the conventional virtual-perturbation presentation.
\end{proof}

\subsection{Concrete Admissibility Certificates}
\label{subsec:concrete_admissibility_certificates}

We now give several concrete covariance families for which admissibility can be
verified directly. These certificates replace the earlier black-box assumption
that the one-step reference comparison simply holds.

\begin{proposition}[Deterministic covariance schedules are admissible]
\label{prop:deterministic_covariances_admissible}
Suppose that, for every \(t\),
\[
    \Sigma_t
\]
is deterministic and positive definite. Choose
\[
    \Sigma_t^{\mathrm{ref}}=\Sigma_t.
\]
Then \(\Sigma_t^{\mathrm{ref}}\) is \(S\)-admissible,
\[
    \mathcal C_t^{\mathrm{cov}}=0,
    \qquad
    \Lambda_t=\Sigma_t^{-1},
\]
and the certified one-step comparison
\eqref{eq:certified_onestep_reference_bound_main} holds.
\end{proposition}

\begin{proof}
A deterministic covariance schedule contains no information about the training
sample. Hence the reference kernel may use the same covariance as the actual
virtual perturbation. Since
\(\Sigma_t^{\mathrm{ref}}=\Sigma_t\), the Gaussian covariance-mismatch term is
zero. Proposition~\ref{prop:certified_onestep_reference} proves the one-step
comparison.
\end{proof}

\begin{proposition}[Public-predictable covariance schedules are admissible]
\label{prop:public_predictable_covariances_admissible}
Let \(U_{t-1}\) be public randomness independent of the sample values. Suppose
that
\[
    \Sigma_t=\Psi_t(U_{t-1}),
    \qquad
    \Sigma_t\succ0,
\]
for a measurable map \(\Psi_t\). Choose
\[
    \Sigma_t^{\mathrm{ref}}=\Sigma_t.
\]
Then, after conditioning on the public randomness, the synchronized reference
kernel is \(S\)-admissible,
\[
    \mathcal C_t^{\mathrm{cov}}=0,
    \qquad
    \Lambda_t=\Sigma_t^{-1},
\]
and the certified one-step comparison
\eqref{eq:certified_onestep_reference_bound_main} holds.
\end{proposition}

\begin{proof}
Since \(U_{t-1}\) is independent of the sample values, conditioning on
\(U_{t-1}\) does not reveal information about \(S\). The reference kernel may
therefore use \(\Psi_t(U_{t-1})\). Because the actual and reference covariances
coincide, the covariance-comparison cost vanishes. The result follows from
Proposition~\ref{prop:certified_onestep_reference}.
\end{proof}

\begin{proposition}[Prefix-observable synchronized covariance]
\label{prop:prefix_observable_covariances_admissible}
Suppose that the actual predictable covariance \(\Sigma_t\) admits an
\(\mathcal F_t^Q\)-measurable version. Equivalently, suppose there exists a
measurable map \(\psi_t\) such that
\begin{equation}
    \Sigma_t
    =
    \psi_t(\widetilde W_{1:t},\mathcal U)
    \qquad
    \text{almost surely}.
    \label{eq:prefix_observable_covariance_main}
\end{equation}
Choose
\[
    \Sigma_t^{\mathrm{ref}}
    =
    \psi_t(\widetilde W_{1:t},\mathcal U).
\]
Then the synchronized reference covariance is \(S\)-admissible,
\[
    \mathcal C_t^{\mathrm{cov}}=0,
    \qquad
    \Lambda_t=\Sigma_t^{-1},
\]
and the certified one-step comparison
\eqref{eq:certified_onestep_reference_bound_main} holds.
\end{proposition}

\begin{proof}
The reference kernel is allowed to condition on
\(\mathcal F_t^Q=\sigma(\widetilde W_{1:t})\vee\mathcal U\). Hence it can
compute \(\psi_t(\widetilde W_{1:t},\mathcal U)\) without direct access to
\(S\). Since this equals \(\Sigma_t\) almost surely, the covariance mismatch
cost is zero. The claim follows from
Proposition~\ref{prop:certified_onestep_reference}.
\end{proof}

\begin{proposition}[Ghost-adaptive reference covariance]
\label{prop:ghost_adaptive_covariances_admissible}
Let \(S^\circ\) be an independent ghost sample drawn from the same population
distribution as \(S\). Run an independent ghost SGD trajectory
\[
    W_{t+1}^{\circ}
    =
    W_t^{\circ}
    -
    \eta_tG_t^{\circ},
    \qquad
    G_t^{\circ}=g(W_t^{\circ},B_t^{\circ}),
\]
using randomness independent of the original sample \(S\). Let
\(\mathcal H_{t-1}^{\circ}\) denote the corresponding ghost history. Suppose
the actual covariance is generated by a predictable selector
\[
    \Sigma_t=\Phi_t(\mathcal H_{t-1}),
\]
and define the reference covariance by
\[
    \Sigma_t^{\mathrm{ref}}
    =
    \Phi_t(\mathcal H_{t-1}^{\circ}),
    \qquad
    \Sigma_t^{\mathrm{ref}}\succ0.
\]
Then \(\Sigma_t^{\mathrm{ref}}\) is \(S\)-admissible. The certified one-step
comparison holds with
\[
    \Lambda_t=(\Sigma_t^{\mathrm{ref}})^{-1}
\]
and covariance-comparison cost
\begin{equation}
    \mathcal C_{t,\mathrm{ghost}}^{\mathrm{cov}}
    =
    \frac12
    \mathbb E
    \left[
        \operatorname{Tr}
        \left(
            (\Sigma_t^{\mathrm{ref}})^{-1}\Sigma_t
        \right)
        -d
        +
        \log
        \frac{\det\Sigma_t^{\mathrm{ref}}}{\det\Sigma_t}
    \right].
    \label{eq:ghost_covariance_cost_main}
\end{equation}
\end{proposition}

\begin{proof}
The ghost trajectory is generated from \(S^\circ\) and ghost randomness
independent of the original sample \(S\). Therefore
\(\Sigma_t^{\mathrm{ref}}=\Phi_t(\mathcal H_{t-1}^{\circ})\) does not condition
on \(S\). It is an admissible reference covariance. The actual and reference
covariances need not coincide, so the Gaussian covariance-mismatch cost remains
and is exactly \(\mathcal C_{t,\mathrm{ghost}}^{\mathrm{cov}}\).
Proposition~\ref{prop:certified_onestep_reference} gives the one-step
comparison.
\end{proof}

\begin{remark}
The ghost-adaptive construction provides a valid reference geometry even when
the actual adaptive covariance is not observable from the virtual prefix. The
price is the explicit mismatch cost
\(\mathcal C_{t,\mathrm{ghost}}^{\mathrm{cov}}\). If the selector \(\Phi_t\) is
stable under replacement of the original trajectory by an independent ghost
trajectory, this cost can be small. If the selector is highly sample-sensitive,
the cost may dominate the bound.
\end{remark}

\subsection{Main General Covariance-Comparison Bound}
\label{subsec:main_general_covariance_comparison_bound}

We can now state the main theorem. The theorem separates three effects:
the information accumulated through stochastic gradients, the covariance
mismatch between actual and reference smoothing kernels, and the output
sensitivity required to transfer the guarantee from the perturbed output
\(\widetilde W_T\) back to the original SGD output \(W_T\).

Define the adaptive output-sensitivity functional by
\begin{equation}
    \Delta_{\Sigma_{1:T}}^{\mathrm{ad}}(W_T,s)
    =
    \mathbb E
    \left[
        L(W_T,s)-L(W_T+\zeta_T,s)
        \,\middle|\,
        \mathcal H_{T-1}
    \right],
    \label{eq:adaptive_output_sensitivity_functional_main}
\end{equation}
where
\[
    \zeta_T\mid\mathcal H_{T-1}
    \sim
    \mathcal N(0,\Sigma_{1:T}),
    \qquad
    \Sigma_{1:T}=\sum_{t=1}^{T-1}\Sigma_t.
\]
The adaptive output-sensitivity penalty is
\begin{equation}
    \mathcal R_{\Delta}^{\mathrm{ad}}
    =
    \left|
    \mathbb E
    \left[
        \Delta_{\Sigma_{1:T}}^{\mathrm{ad}}(W_T,S')
        -
        \Delta_{\Sigma_{1:T}}^{\mathrm{ad}}(W_T,S)
    \right]
    \right|.
    \label{eq:adaptive_output_sensitivity_penalty_main}
\end{equation}

\begin{assumption}[Basic regularity and certified reference conditions]
\label{ass:basic_conditions_main}
The following conditions hold.
\begin{enumerate}
    \item \(W_1\) is independent of \(S\).

    \item For every fixed \(w\), the random variable \(\ell(w,Z)\) is
    \(R\)-sub-Gaussian under the population distribution.

    \item The loss is differentiable in \(w\), and all gradient-deviation,
    sensitivity, covariance-comparison, and output-sensitivity quantities
    appearing below are finite.

    \item The actual covariance process is predictable:
    \[
        \Sigma_t
        \text{ is }
        \mathcal H_{t-1}\text{-measurable},
        \qquad
        \Sigma_t\succ0
        \quad
        \text{almost surely}.
    \]

    \item The virtual perturbations satisfy
    \[
        \varepsilon_t\mid\mathcal H_{t-1}
        \sim
        \mathcal N(0,\Sigma_t),
    \]
    and are conditionally independent of the current minibatch randomness and
    future optimization randomness given the appropriate past history.

    \item For each \(t\), a positive definite \(S\)-admissible reference
    covariance \(\Sigma_t^{\mathrm{ref}}\) is chosen, and the reference kernel is
    the canonical Gaussian kernel
    \[
        Q_{t+1\mid\mathcal F_t^Q}
        =
        \mathcal N
        \left(
            \widetilde W_t-\eta_t\bar g(\widetilde W_t),
            \Sigma_t^{\mathrm{ref}}
        \right).
    \]
\end{enumerate}
\end{assumption}

\begin{theorem}[General covariance-comparison bound with certified references]
\label{thm:general_covariance_comparison_bound}
Under Assumption~\ref{ass:basic_conditions_main}, define
\[
    \Lambda_t=(\Sigma_t^{\mathrm{ref}})^{-1}
\]
and
\[
    \mathcal C_t^{\mathrm{cov}}
    =
    \frac12
    \mathbb E
    \left[
        \operatorname{Tr}
        \left(
            (\Sigma_t^{\mathrm{ref}})^{-1}\Sigma_t
        \right)
        -d
        +
        \log
        \frac{\det\Sigma_t^{\mathrm{ref}}}{\det\Sigma_t}
    \right].
\]
Then the expected generalization error of the original SGD output \(W_T\)
satisfies
\begin{equation}
    \left|
        \operatorname{gen}(W_T,S)
    \right|
    \leq
    \sqrt{
        \frac{2R^2}{n}
        \sum_{t=1}^{T-1}
        \left(
            2\eta_t^2
            \mathbb E
            \left[
                V_t^{\mathrm{ad}}(\Lambda_t)
                +
                \Gamma_t^{\mathrm{ad}}(\Lambda_t)
            \right]
            +
            \mathcal C_t^{\mathrm{cov}}
        \right)
    }
    +
    \mathcal R_{\Delta}^{\mathrm{ad}}.
    \label{eq:general_covariance_comparison_bound_main}
\end{equation}
\end{theorem}

\begin{proof}
Let
\[
    \zeta_T\mid\mathcal H_{T-1}
    \sim
    \mathcal N(0,\Sigma_{1:T})
\]
be an independent fresh final perturbation, and define
\[
    \widehat W_T=W_T+\zeta_T.
\]
By adding and subtracting losses evaluated at \(\widehat W_T\),
\begin{align}
    \gen(W_T,S)
    &=
    \mathbb E
    \left[
        L(W_T,S')-L(W_T,S)
    \right]
    \nonumber\\
    &=
    \mathbb E
    \left[
        L(\widehat W_T,S')-L(\widehat W_T,S)
    \right]
    \nonumber\\
    &\quad+
    \mathbb E
    \left[
        \Delta_{\Sigma_{1:T}}^{\mathrm{ad}}(W_T,S')
        -
        \Delta_{\Sigma_{1:T}}^{\mathrm{ad}}(W_T,S)
    \right].
\end{align}
Therefore,
\begin{equation}
    |\gen(W_T,S)|
    \leq
    |\gen(\widehat W_T,S)|
    +
    \mathcal R_{\Delta}^{\mathrm{ad}}.
    \label{eq:output_decomposition_main}
\end{equation}
Since \(\zeta_T\mid\mathcal H_{T-1}\) has the same conditional law as
\(\xi_T\mid\mathcal H_{T-1}\), the random variables
\[
    \widehat W_T=W_T+\zeta_T
    \quad\text{and}\quad
    \widetilde W_T=W_T+\xi_T
\]
have the same joint law with \(S\). Hence
\[
    I(\widehat W_T;S)
    =
    I(\widetilde W_T;S).
\]
The standard information-theoretic generalization inequality gives
\[
    |\gen(\widehat W_T,S)|
    \leq
    \sqrt{
        \frac{2R^2}{n}
        I(\widehat W_T;S)
    }
    =
    \sqrt{
        \frac{2R^2}{n}
        I(\widetilde W_T;S)
    }.
\]

By Lemma~\ref{lem:reference_kernel_mi_decomposition},
\[
    I(\widetilde W_T;S)
    \leq
    \sum_{t=1}^{T-1}
    \mathbb E
    \left[
        D_{\mathrm{KL}}
        \left(
            P_{\widetilde W_{t+1}\mid\mathcal A_t}
            \,\middle\|\,
            Q_{t+1\mid\mathcal F_t^Q}
        \right)
    \right].
\]
For each \(t\), Proposition~\ref{prop:certified_onestep_reference} bounds the
corresponding one-step KL term by
\[
    2\eta_t^2
    \mathbb E
    \left[
        V_t^{\mathrm{ad}}(\Lambda_t)
        +
        \Gamma_t^{\mathrm{ad}}(\Lambda_t)
    \right]
    +
    \mathcal C_t^{\mathrm{cov}}.
\]
Substituting into the mutual-information bound and then into the
sub-Gaussian generalization inequality proves
\eqref{eq:general_covariance_comparison_bound_main}.
\end{proof}

\begin{remark}
The theorem separates two roles that are often conflated. Predictability of
\(\Sigma_t\) is the causal condition needed to generate the actual virtual
Gaussian perturbation. \(S\)-admissibility of \(\Sigma_t^{\mathrm{ref}}\) is the
information-theoretic condition needed to construct a valid reference kernel.
When the two covariances coincide admissibly, the covariance-comparison cost is
zero. When they differ, \(\mathcal C_t^{\mathrm{cov}}\) is the explicit KL price
of using a valid reference geometry.
\end{remark}

\subsection{Readable Corollaries}
\label{subsec:readable_corollaries}

The general theorem is stated in terms of the reference precision
\[
    \Lambda_t=(\Sigma_t^{\mathrm{ref}})^{-1}.
\]
This is the correct level of generality for history-adaptive virtual
perturbations, because the actual covariance \(\Sigma_t\) and the covariance
available to the reference kernel need not coincide. The next corollary gives a
more readable form when the reference precision is comparable to the actual
precision.

\begin{corollary}[Comparable covariance geometries]
\label{cor:comparable_covariances}
Suppose the conditions of
Theorem~\ref{thm:general_covariance_comparison_bound} hold. Assume additionally
that there exists \(\kappa\geq1\) such that, for all
\(t=1,\ldots,T-1\),
\begin{equation}
    (\Sigma_t^{\mathrm{ref}})^{-1}
    \preceq
    \kappa\,\Sigma_t^{-1}
    \qquad
    \text{almost surely}.
    \label{eq:cor_precision_comparability}
\end{equation}
Then
\begin{equation}
    |\gen(W_T,S)|
    \leq
    \sqrt{
        \frac{2R^2}{n}
        \left[
            2\kappa
            \sum_{t=1}^{T-1}
            \eta_t^2
            \E
            \left[
                V_t^{\mathrm{ad}}
                +
                \Gamma_t^{\mathrm{ad}}
            \right]
            +
            \sum_{t=1}^{T-1}
            \mathcal C_t^{\mathrm{cov}}
        \right]
    }
    +
    \mathcal R_{\Delta}^{\mathrm{ad}}.
    \label{eq:comparable_covariance_bound}
\end{equation}
\end{corollary}

\begin{proof}
By \eqref{eq:cor_precision_comparability}, for every vector \(x\),
\[
    \|x\|_{\Lambda_t}^2
    =
    x^\top(\Sigma_t^{\mathrm{ref}})^{-1}x
    \leq
    \kappa x^\top\Sigma_t^{-1}x
    =
    \kappa\|x\|_{\Sigma_t^{-1}}^2.
\]
Therefore,
\[
    V_t^{\mathrm{ad}}(\Lambda_t)
    \leq
    \kappa V_t^{\mathrm{ad}},
    \qquad
    \Gamma_t^{\mathrm{ad}}(\Lambda_t)
    \leq
    \kappa\Gamma_t^{\mathrm{ad}}.
\]
Substituting these inequalities into
Theorem~\ref{thm:general_covariance_comparison_bound} proves the claim.
\end{proof}

The term \(V_t^{\mathrm{ad}}\) in
Corollary~\ref{cor:comparable_covariances} should be interpreted as a
conditional mean-square gradient-deviation term. It becomes a genuine
conditional variance term only under the conditional unbiasedness condition
\[
    \E[G_t\mid \cH_{t-1}]
    =
    \barg(W_t).
\]
Without this condition, \(V_t^{\mathrm{ad}}\) also contains the squared
conditional bias between the conditional mean gradient and the population
gradient.

\subsubsection{Admissible Synchronization}
\label{subsubsec:admissible_synchronization}

The clean fixed-noise-style form is recovered only when the reference kernel is
allowed to use the same covariance as the actual virtual perturbation. This
requires more than the algebraic identity
\[
    \Sigma_t^{\mathrm{ref}}=\Sigma_t.
\]
The reference kernel in the mutual-information decomposition cannot condition
directly on the sample \(S\). Hence the actual covariance must be reconstructible
from information available to the reference comparison. Recall the reference-visible sigma-field
\[
    \mathcal F_t^Q
    =
    \sigma(\wtW_{1:t})\vee\mathcal U,
\]
where \(\mathcal U\) denotes optional public, auxiliary, or ghost randomness
independent of the original training sample \(S\).

\begin{definition}[Admissible synchronization certificate]
\label{def:admissible_synchronization_certificate}
We say that the covariance process admits an \(S\)-admissible synchronized
reference at time \(t\) if there exists an \(\mathcal F_t^Q\)-measurable
positive definite random matrix \(\Sigma_t^{\mathrm{sync}}\) such that
\begin{equation}
    \Sigma_t^{\mathrm{sync}}
    =
    \Sigma_t
    \qquad
    \text{almost surely}.
    \label{eq:admissible_sync_certificate}
\end{equation}
Equivalently, the actual covariance \(\Sigma_t\) has a version that can be
computed from the virtual prefix \(\wtW_{1:t}\) and admissible
sample-independent auxiliary information, without direct access to \(S\).
\end{definition}

\begin{proposition}[Synchronization certificate implies zero covariance cost]
\label{prop:sync_certificate_zero_cost}
Suppose Definition~\ref{def:admissible_synchronization_certificate} holds at
time \(t\). Choose
\[
    \Sigma_t^{\mathrm{ref}}
    =
    \Sigma_t^{\mathrm{sync}}.
\]
Then \(\Sigma_t^{\mathrm{ref}}\) is \(S\)-admissible,
\[
    \Sigma_t^{\mathrm{ref}}=\Sigma_t
    \qquad
    \text{almost surely},
\]
and
\begin{equation}
    \mathcal C_t^{\mathrm{cov}}=0,
    \qquad
    \Lambda_t=\Sigma_t^{-1}.
    \label{eq:sync_zero_cost_and_precision}
\end{equation}
\end{proposition}

\begin{proof}
By Definition~\ref{def:admissible_synchronization_certificate},
\(\Sigma_t^{\mathrm{sync}}\) is measurable with respect to the information
available to the reference kernel. Therefore it can be used without conditioning
directly on \(S\). Since
\(\Sigma_t^{\mathrm{sync}}=\Sigma_t\) almost surely, the actual and reference
covariances coincide. Substituting
\(\Sigma_t^{\mathrm{ref}}=\Sigma_t\) into the covariance-comparison cost gives
\[
    \mathcal C_t^{\mathrm{cov}}
    =
    \frac12
    \E
    \left[
        \Tr(\Sigma_t^{-1}\Sigma_t)
        -d
        +
        \log
        \frac{\det\Sigma_t}{\det\Sigma_t}
    \right]
    =
    0.
\]
The identity \(\Lambda_t=\Sigma_t^{-1}\) follows immediately.
\end{proof}

\begin{remark}[Predictability is not synchronization]
\label{rem:predictability_not_synchronization}
Predictability means
\[
    \Sigma_t
    \text{ is }
    \cH_{t-1}\text{-measurable}.
\]
This is the causal condition needed to generate the actual virtual Gaussian
perturbation. Admissible synchronization is different: it asks whether the
reference kernel can reconstruct the same covariance from
\[
    \mathcal F_t^Q
    =
    \sigma(\wtW_{1:t})\vee\mathcal U.
\]
Therefore, a covariance of the form
\[
    \Sigma_t=\Phi_t(\cH_{t-1})
\]
may be predictable but not admissibly synchronized if \(\cH_{t-1}\) contains
sample-value information that is not recoverable from the virtual prefix or from
admissible auxiliary randomness.
\end{remark}

\begin{corollary}[Admissibly synchronized covariance bound]
\label{cor:synchronized_covariance_bound}
Suppose the conditions of
Theorem~\ref{thm:general_covariance_comparison_bound} hold. Assume additionally
that, for every \(t=1,\ldots,T-1\), the covariance process admits an
\(S\)-admissible synchronized reference in the sense of
Definition~\ref{def:admissible_synchronization_certificate}. Then
\[
    \mathcal C_t^{\mathrm{cov}}=0,
    \qquad
    \Lambda_t=\Sigma_t^{-1},
\]
and the expected generalization error satisfies
\begin{equation}
    |\gen(W_T,S)|
    \leq
    \sqrt{
        \frac{4R^2}{n}
        \sum_{t=1}^{T-1}
        \eta_t^2
        \E
        \left[
            V_t^{\mathrm{ad}}
            +
            \Gamma_t^{\mathrm{ad}}
        \right]
    }
    +
    \mathcal R_{\Delta}^{\mathrm{ad}}.
    \label{eq:main_adaptive_bound_clean}
\end{equation}
\end{corollary}

\begin{proof}
By Proposition~\ref{prop:sync_certificate_zero_cost}, admissible
synchronization gives
\[
    \mathcal C_t^{\mathrm{cov}}=0,
    \qquad
    \Lambda_t=\Sigma_t^{-1}.
\]
Substituting these identities into
Theorem~\ref{thm:general_covariance_comparison_bound} gives
\eqref{eq:main_adaptive_bound_clean}.
\end{proof}

\begin{remark}[When the synchronized corollary is automatic]
\label{rem:when_sync_automatic}
The admissibly synchronized corollary is automatic in the following cases.
\begin{enumerate}[leftmargin=2em]
    \item \textbf{Deterministic covariance schedules.}
    If \(\Sigma_t\) is deterministic, then it is available to both the actual
    virtual kernel and the reference kernel.

    \item \textbf{Public-random covariance schedules.}
    If
    \[
        \Sigma_t=\Psi_t(U_{t-1}),
    \]
    where \(U_{t-1}\) is independent of the sample values, then synchronization
    is admissible after conditioning on \(U_{t-1}\).

    \item \textbf{Prefix-observable adaptive covariance.}
    If there exists a measurable map \(\psi_t\) such that
    \[
        \Sigma_t=\psi_t(\wtW_{1:t},\mathcal U)
        \qquad
        \text{almost surely},
    \]
    then the reference kernel can recover the actual covariance from the
    virtual prefix and admissible auxiliary information.

    \item \textbf{Admissible public or ghost augmentation.}
    If the covariance can be reconstructed from \(\wtW_{1:t}\) and auxiliary
    randomness independent of \(S\), then synchronization is admissible after
    conditioning on that auxiliary randomness.
\end{enumerate}
Outside these cases, one should use the general covariance-comparison theorem
rather than the clean synchronized corollary.
\end{remark}

\begin{remark}[Interpretation of the readable bounds]
\label{rem:readable_bounds_interpretation}
The comparable and admissibly synchronized corollaries have the same algebraic
form as fixed-noise virtual perturbation bounds, but their interpretation is
more subtle. The term \(V_t^{\mathrm{ad}}\) is a gradient-deviation term:
\[
    V_t^{\mathrm{ad}}
    =
    \E
    \left[
        \normSigma{
            G_t-\barg(W_t)
        }{\Sigma_t^{-1}}^2
        \given
        \cH_{t-1}
    \right].
\]
It reduces to a conditional variance only when
\[
    \E[G_t\mid\cH_{t-1}]
    =
    \barg(W_t).
\]
The term \(\Gamma_t^{\mathrm{ad}}\) measures sensitivity of the population
gradient to accumulated virtual perturbations. The penalty
\(\mathcal R_{\Delta}^{\mathrm{ad}}\) transfers the guarantee from the perturbed
output back to the original SGD output. Finally, the synchronized bound is valid
only under an admissible synchronization certificate; predictability alone is
not enough.
\end{remark}

\section{Special Cases and Recovering Existing Bounds}
\label{sec:special-cases}

The main theorem is stated in a general covariance-comparison form because, for
history-adaptive perturbations, the actual virtual covariance and the reference
covariance used in the mutual-information comparison need not coincide. This
section shows how the theorem specializes to familiar fixed-noise settings and
to several interpretable adaptive virtual geometries. The key distinction is the following. When the covariance sequence is
deterministic, synchronized reference comparisons are automatically admissible,
and the clean fixed-noise-style corollary applies directly. When the covariance
is history-adaptive and hence generally data-dependent, the general
covariance-comparison theorem should be used unless an admissible synchronized
comparison is explicitly available. Thus, adaptive examples below should be
interpreted as choices of virtual perturbation geometry, not as modifications of
the SGD optimizer. Throughout this section, the true SGD update remains unchanged:
\begin{equation}
    W_{t+1}
    =
    W_t-\eta_tG_t,
    \qquad
    G_t=g(W_t,B_t),
    \qquad
    t=1,\ldots,T-1.
\end{equation}
Only the analytical perturbation covariance \(\Sigma_t\) is varied.

\subsection{Fixed Isotropic Perturbations}
\label{subsec:fixed_isotropic_perturbations}

The simplest special case is deterministic isotropic virtual noise. Let
\begin{equation}
    \Sigma_t
    =
    \sigma_t^2 I,
    \qquad
    \sigma_t>0,
    \qquad
    t=1,\ldots,T-1,
    \label{eq:fixed_isotropic_sigma}
\end{equation}
where the sequence \(\{\sigma_t\}_{t=1}^{T-1}\) is fixed before observing the
SGD trajectory. Since \(\Sigma_t\) is deterministic, it is automatically
\(\cH_{t-1}\)-measurable. Moreover, the same covariance can be used in the
actual and reference one-step comparisons without depending on \(S\), so the
covariance-comparison term vanishes:
\begin{equation}
    \mathcal C_t^{\mathrm{cov}}=0.
\end{equation}
The accumulated covariance is deterministic:
\begin{equation}
    \Sigma_{1:t}
    =
    \sum_{k=1}^{t-1}
    \Sigma_k
    =
    \sigma_{1:t}^2 I,
    \qquad
    \sigma_{1:t}^2
    =
    \sum_{k=1}^{t-1}
    \sigma_k^2.
    \label{eq:fixed_isotropic_accum}
\end{equation}
Consequently,
\begin{equation}
    \normSigma{x}{\Sigma_t^{-1}}^2
    =
    \sigma_t^{-2}\norm{x}^2.
\end{equation}
Define the Euclidean history-conditioned gradient-deviation numerator
\begin{equation}
    V_t^{\mathrm{Euc}}
    =
    \E
    \left[
        \norm{
            G_t-\barg(W_t)
        }^2
        \given
        \cH_{t-1}
    \right].
    \label{eq:fixed_iso_euc_deviation}
\end{equation}
This is a variance numerator only under the conditional unbiasedness condition
\[
    \E[G_t\mid\cH_{t-1}]
    =
    \barg(W_t).
\]
Define the Euclidean sensitivity numerator
\begin{equation}
    \Gamma_t^{\mathrm{Euc}}
    =
    \E
    \left[
        \norm{
            \barg(W_t+\zeta_t)-\barg(W_t)
        }^2
        \given
        \cH_{t-1}
    \right],
    \qquad
    \zeta_t\given\cH_{t-1}
    \sim
    \N(0,\sigma_{1:t}^2I).
    \label{eq:fixed_iso_euc_sensitivity}
\end{equation}
Then
\begin{equation}
    V_t^{\mathrm{ad}}
    =
    \frac{1}{\sigma_t^2}V_t^{\mathrm{Euc}},
    \qquad
    \Gamma_t^{\mathrm{ad}}
    =
    \frac{1}{\sigma_t^2}\Gamma_t^{\mathrm{Euc}}.
\end{equation}

Substituting these quantities into the admissibly synchronized covariance
corollary yields
\begin{equation}
    |\gen(W_T,S)|
    \leq
    \sqrt{
        \frac{4R^2}{n}
        \sum_{t=1}^{T-1}
        \frac{\eta_t^2}{\sigma_t^2}
        \E
        \left[
            V_t^{\mathrm{Euc}}
            +
            \Gamma_t^{\mathrm{Euc}}
        \right]
    }
    +
    \mathcal R_{\Delta,\sigma}.
    \label{eq:fixed_iso_bound}
\end{equation}
Here \(\mathcal R_{\Delta,\sigma}\) is the output-sensitivity penalty
corresponding to the deterministic accumulated covariance
\(\sigma_{1:T}^2I\). Under the usual fixed-noise notation,
\eqref{eq:fixed_iso_bound} is the scalar isotropic virtual perturbation form.
Thus, the classical fixed isotropic setting is recovered as a special case of
the predictable covariance framework.

\subsection{Fixed Geometry-Aware Covariances}
\label{subsec:fixed_geometry_aware_covariances}

A more general fixed setting allows deterministic anisotropic covariance
matrices. Let
\begin{equation}
    \Sigma_t\succ0,
    \qquad
    t=1,\ldots,T-1,
    \label{eq:fixed_matrix_cov}
\end{equation}
be deterministic but not necessarily proportional to the identity. Again, each
\(\Sigma_t\) is automatically predictable, and a synchronized reference
comparison is admissible. Therefore,
\begin{equation}
    \mathcal C_t^{\mathrm{cov}}=0,
    \qquad
    \Lambda_t=\Sigma_t^{-1}.
\end{equation}
The accumulated covariance is deterministic:
\begin{equation}
    \Sigma_{1:t}
    =
    \sum_{k=1}^{t-1}
    \Sigma_k.
\end{equation}

Define the fixed geometry-aware gradient-deviation and sensitivity terms by
\begin{equation}
    V_{t,\Sigma_t}^{\mathrm{hist}}
    =
    \E
    \left[
        \normSigma{
            G_t-\barg(W_t)
        }{\Sigma_t^{-1}}^2
        \given
        \cH_{t-1}
    \right],
    \label{eq:fixed_matrix_deviation_hist}
\end{equation}
and
\begin{equation}
    \Gamma_{\Sigma_t,\Sigma_{1:t}}^{\mathrm{hist}}
    =
    \E
    \left[
        \normSigma{
            \barg(W_t+\zeta_t)-\barg(W_t)
        }{\Sigma_t^{-1}}^2
        \given
        \cH_{t-1}
    \right],
    \qquad
    \zeta_t\given\cH_{t-1}
    \sim
    \N(0,\Sigma_{1:t}).
    \label{eq:fixed_matrix_sensitivity_hist}
\end{equation}
The admissibly synchronized covariance corollary gives
\begin{equation}
    |\gen(W_T,S)|
    \leq
    \sqrt{
        \frac{4R^2}{n}
        \sum_{t=1}^{T-1}
        \eta_t^2
        \E
        \left[
            V_{t,\Sigma_t}^{\mathrm{hist}}
            +
            \Gamma_{\Sigma_t,\Sigma_{1:t}}^{\mathrm{hist}}
        \right]
    }
    +
    \mathcal R_{\Delta}.
    \label{eq:fixed_geometry_bound}
\end{equation}
This is the fixed geometry-aware virtual perturbation bound in
history-conditioned notation. Some fixed-noise presentations write the
corresponding quantities as functions of \(W_t\) alone. That reduction is valid
after taking expectations under an additional condition ensuring that the
conditional law of the current stochastic gradient given the past depends on
the past only through \(W_t\). In more general sampling settings, the
history-conditioned notation is the more precise analogue. Compared with isotropic perturbations, deterministic matrix-valued covariances
can encode a fixed anisotropic geometry of the parameter space. However, because
the covariance is chosen before observing the training trajectory, it still
cannot adapt to realized gradient scales, curvature proxies, or time-varying
stochasticity.

\subsection{Adaptive Scalar Perturbations}
\label{subsec:adaptive_scalar_perturbations}

The first genuinely history-adaptive case is scalar but time- and
path-dependent noise:
\begin{equation}
    \Sigma_t
    =
    \sigma_t^2(\cH_{t-1})I,
    \qquad
    \sigma_t(\cH_{t-1})>0,
    \qquad
    t=1,\ldots,T-1.
    \label{eq:adaptive_scalar_cov}
\end{equation}
Here \(\sigma_t\) may depend on any statistic of the past SGD trajectory, such
as moving averages of gradient norms, subbatch fluctuation proxies,
gradient-deviation diagnostics, or training-stage indicators. Because
\(\sigma_t\) is \(\cH_{t-1}\)-measurable, the covariance is predictable. The accumulated covariance is
\begin{equation}
    \Sigma_{1:t}
    =
    \sigma_{1:t}^2I,
    \qquad
    \sigma_{1:t}^2
    =
    \sum_{k=1}^{t-1}
    \sigma_k^2(\cH_{k-1}).
\end{equation}
Thus,
\begin{equation}
    \zeta_t\given\cH_{t-1}
    \sim
    \N(0,\sigma_{1:t}^2I).
\end{equation}

Under an admissibly synchronized comparison, the adaptive scalar
gradient-deviation and sensitivity terms are
\begin{equation}
    V_t^{\mathrm{ad}}
    =
    \frac{
        \E
        \left[
            \norm{
                G_t-\barg(W_t)
            }^2
            \given
            \cH_{t-1}
        \right]
    }{
        \sigma_t^2(\cH_{t-1})
    },
    \label{eq:adaptive_scalar_deviation}
\end{equation}
and
\begin{equation}
    \Gamma_t^{\mathrm{ad}}
    =
    \frac{
        \E
        \left[
            \norm{
                \barg(W_t+\zeta_t)-\barg(W_t)
            }^2
            \given
            \cH_{t-1}
        \right]
    }{
        \sigma_t^2(\cH_{t-1})
    }.
    \label{eq:adaptive_scalar_sensitivity}
\end{equation}
The term in \eqref{eq:adaptive_scalar_deviation} is a variance only under
conditional unbiasedness of \(G_t\) for \(\barg(W_t)\). Because \(\sigma_t(\cH_{t-1})\) is generally data-dependent, the synchronized
clean bound is not automatic. The general covariance-comparison theorem applies
with a reference scalar covariance
\begin{equation}
    \Sigma_t^{\mathrm{ref}}
    =
    (\sigma_t^{\mathrm{ref}})^2I,
\end{equation}
which yields
\begin{equation}
    \Lambda_t
    =
    (\Sigma_t^{\mathrm{ref}})^{-1}
    =
    \frac{1}{(\sigma_t^{\mathrm{ref}})^2}I.
\end{equation}
The covariance-comparison cost becomes
\begin{equation}
    \mathcal C_t^{\mathrm{cov}}
    =
    \frac{d}{2}
    \E
    \left[
        \frac{\sigma_t^2}{(\sigma_t^{\mathrm{ref}})^2}
        -
        1
        +
        \log
        \frac{(\sigma_t^{\mathrm{ref}})^2}{\sigma_t^2}
    \right].
    \label{eq:adaptive_scalar_cov_cost}
\end{equation}
If an admissible synchronized comparison is available, then
\(\sigma_t^{\mathrm{ref}}=\sigma_t\), the covariance cost vanishes, and the
clean bound uses \eqref{eq:adaptive_scalar_deviation} and
\eqref{eq:adaptive_scalar_sensitivity}. Otherwise, the reference-geometry
quantities and the cost \eqref{eq:adaptive_scalar_cov_cost} must be included. This example makes the basic trade-off explicit. Larger \(\sigma_t^2\) reduces
inverse-covariance weighted gradient-deviation and sensitivity terms, but it
also increases the accumulated covariance and can enlarge the final
output-sensitivity penalty.

\subsection{Adaptive Diagonal Covariances}
\label{subsec:adaptive_diagonal_covariances}
A richer adaptive family is obtained by choosing coordinate-wise perturbation
scales:
\begin{equation}
    \Sigma_t
    =
    \diag
    \left(
        s_{t,1}^2,\ldots,s_{t,d}^2
    \right),
    \qquad
    s_{t,j}>0,
    \label{eq:adaptive_diagonal_cov}
\end{equation}
where each \(s_{t,j}\) is \(\cH_{t-1}\)-measurable. For any vector
\(x\in\R^d\),
\begin{equation}
    \normSigma{x}{\Sigma_t^{-1}}^2
    =
    \sum_{j=1}^{d}
    \frac{x_j^2}{s_{t,j}^2}.
    \label{eq:diagonal_inverse_norm}
\end{equation}
Consequently, under an admissibly synchronized comparison, the adaptive
gradient-deviation and sensitivity terms take the coordinate-wise forms
\begin{equation}
    V_t^{\mathrm{ad}}
    =
    \E
    \left[
        \sum_{j=1}^{d}
        \frac{
            \left(
                G_{t,j}-\barg_j(W_t)
            \right)^2
        }{
            s_{t,j}^2
        }
        \given
        \cH_{t-1}
    \right],
    \label{eq:diagonal_deviation}
\end{equation}
and
\begin{equation}
    \Gamma_t^{\mathrm{ad}}
    =
    \E
    \left[
        \sum_{j=1}^{d}
        \frac{
            \left(
                \barg_j(W_t+\zeta_t)-\barg_j(W_t)
            \right)^2
        }{
            s_{t,j}^2
        }
        \given
        \cH_{t-1}
    \right],
    \label{eq:diagonal_sensitivity}
\end{equation}
where
\[
    \zeta_t\given\cH_{t-1}\sim\N(0,\Sigma_{1:t}).
\]
The term in \eqref{eq:diagonal_deviation} is a coordinate-wise variance only
under conditional unbiasedness. If the reference covariance is diagonal with entries
\((s_{t,j}^{\mathrm{ref}})^2\), then the covariance-comparison cost is
\begin{equation}
    \mathcal C_t^{\mathrm{cov}}
    =
    \frac12
    \E
    \left[
        \sum_{j=1}^{d}
        \left(
            \frac{s_{t,j}^2}{(s_{t,j}^{\mathrm{ref}})^2}
            -
            1
            +
            \log
            \frac{(s_{t,j}^{\mathrm{ref}})^2}{s_{t,j}^2}
        \right)
    \right].
    \label{eq:diagonal_cov_cost}
\end{equation}
This expression is the coordinate-wise KL cost of comparing Gaussian smoothing
kernels with different diagonal covariances. Diagonal adaptive perturbations
allow the analytical geometry to respond to coordinate-wise scale information.
Directions assigned larger virtual variance contribute less to the information
term, but may contribute more to the accumulated output-sensitivity penalty.
Thus, the coordinate-wise covariance exposes an information-sensitivity
trade-off rather than a uniformly better choice.

\subsection{Adam-Like Virtual Geometry}
\label{subsec:adam_like_virtual_geometry}

The predictable covariance framework can express virtual perturbation geometries
inspired by adaptive optimizers. This does not mean that the SGD update is
replaced by Adam. Rather, Adam-like moving statistics are used only to define
the covariance of the analytical perturbations.

For example, define a past-gradient second-moment proxy by
\begin{equation}
    v_t
    =
    \beta v_{t-1}
    +
    (1-\beta)
    G_{t-1}\odot G_{t-1},
    \qquad
    0<\beta<1,
    \label{eq:adam_like_second_moment}
\end{equation}
with an initialization \(v_1\) independent of the current update. Because
\(v_t\) is computed from gradients up to time \(t-1\), it is
\(\cH_{t-1}\)-measurable. Therefore, any positive definite diagonal covariance
constructed from \(v_t\) is predictable. One possible inverse-scale virtual covariance is
\begin{equation}
    \Sigma_t^{\mathrm{inv}}
    =
    \diag
    \left(
        \frac{\rho_t^2}{\sqrt{v_t}+\epsilon}
    \right)
    +
    \lambda_0 I,
    \label{eq:adam_inverse_scale_cov}
\end{equation}
where \(\rho_t>0\) is predictable, \(\epsilon>0\) stabilizes the denominator,
and \(\lambda_0>0\) ensures positive definiteness. This choice assigns smaller
virtual perturbation variance to coordinates with larger historical second
moment.

Alternatively, one may choose a proportional-scale virtual covariance
\begin{equation}
    \Sigma_t^{\mathrm{prop}}
    =
    \diag
    \left(
        \rho_t^2(\sqrt{v_t}+\epsilon)
    \right)
    +
    \lambda_0 I.
    \label{eq:adam_proportional_scale_cov}
\end{equation}
This choice assigns larger virtual smoothing to coordinates with larger
historical gradient scale. It may reduce the inverse-covariance weighted
information contribution of discrepancies in high-fluctuation or
high-gradient-deviation directions, while potentially increasing the
output-sensitivity penalty along those same directions.

Both \eqref{eq:adam_inverse_scale_cov} and
\eqref{eq:adam_proportional_scale_cov} are valid choices of actual virtual
covariance whenever the defining statistics are predictable and the covariance
is positive definite. For the generalization bound, one must still specify the
corresponding admissible reference comparison. If the reference covariance
matches the actual covariance in an admissible synchronized comparison, the
clean corollary applies. Otherwise, the general covariance-comparison theorem
applies and includes the appropriate covariance-comparison term. Most importantly, neither \eqref{eq:adam_inverse_scale_cov} nor
\eqref{eq:adam_proportional_scale_cov} changes the true SGD recursion:
\begin{equation}
    W_{t+1}=W_t-\eta_tG_t.
\end{equation}
The Adam-like quantities appear only in the virtual perturbation geometry used
for analysis.

\subsection{Summary}
\label{subsec:special_cases_summary}

The examples above show that the predictable perturbation theorem unifies
several regimes. Fixed deterministic covariances recover existing fixed-noise
bounds through the admissibly synchronized covariance corollary. In particular,
the choice \(\Sigma_t=\sigma_t^2 I\) corresponds to fixed isotropic virtual
noise, while a deterministic positive definite covariance
\(\Sigma_t\succ0\) represents a fixed geometry-aware virtual noise model.
History-adaptive covariances provide richer analytical geometries, but because
they are generally data-dependent, the general covariance-comparison theorem is
the correct default unless synchronized reference comparisons are admissible.
For example, \(\sigma_t^2(\cH_{t-1})I\) gives an adaptive scalar virtual-noise
model, diagonal choices such as
\(\diag(s_{t,1}^2,\ldots,s_{t,d}^2)\) encode coordinate-wise adaptive geometry,
and Adam-like diagonal covariances represent past-gradient virtual geometry.

\begin{table}[t]
\centering
\caption{Representative virtual covariance families. These choices affect only
the analytical perturbation process, not the SGD update.}
\label{tab:covariance-families}
\begin{tabular}{lll}
\toprule
\textbf{Choice of \(\Sigma_t\)}
&
\textbf{Predictability}
&
\textbf{Reference comparison}
\\
\midrule
\(\sigma_t^2I\), deterministic
&
automatic
&
admissibly synchronized
\\
deterministic \(\Sigma_t\succ0\)
&
automatic
&
admissibly synchronized
\\
\(\sigma_t^2(\cH_{t-1})I\)
&
history-adaptive
&
requires certificate or reference covariance
\\
\(\diag(s_{t,1}^2,\ldots,s_{t,d}^2)\)
&
history-adaptive
&
requires certificate or reference covariance
\\
Adam-like diagonal covariance
&
history-adaptive
&
requires certificate or reference covariance
\\
\bottomrule
\end{tabular}
\end{table}

The gain of the framework is analytical flexibility: the proof can use virtual
perturbations whose covariance reflects past optimization information. This
flexibility does not imply that every adaptive covariance yields a tighter
bound. Each choice induces a different balance between inverse-covariance
information control, covariance-comparison cost, and accumulated output
sensitivity.

\section{Controlling the Adaptive Sensitivity Penalty}
\label{sec:sensitivity-penalty}

The main covariance-comparison bound contains two conceptually distinct
components. The first is the information term, which is controlled through
one-step relative-entropy comparisons and may include the covariance-comparison
cost
\[
    \mathcal C_t^{\mathrm{cov}}.
\]
The second is the adaptive output-sensitivity penalty
\[
    \mathcal R_{\Delta}^{\mathrm{ad}},
\]
which accounts for transferring the generalization guarantee from the virtually
perturbed output back to the original SGD output \(W_T\). This section studies
only the latter term. The key point is that the virtual perturbation used in the
information-theoretic argument should not move the final output too much in
loss. If the loss landscape around \(W_T\) is insensitive to perturbations drawn
from the accumulated adaptive covariance, then the output-sensitivity penalty is
small. We first give a global smoothness control, which is safe but often
pessimistic, and then give a local curvature refinement that makes the geometry
of the accumulated covariance explicit.

\subsection{General Form of the Penalty}
\label{subsec:general_form_penalty}

Recall that the accumulated adaptive covariance at the final output is
\begin{equation}
    \Sigma_{1:T}
    =
    \sum_{t=1}^{T-1}
    \Sigma_t,
    \label{eq:penalty_accum_cov}
\end{equation}
where each \(\Sigma_t\) is predictable with respect to the real SGD history.
The final real history is \(\cH_{T-1}\), since \(W_T\) is obtained after
\(T-1\) SGD updates.

Let
\begin{equation}
    \zeta_T \given \cH_{T-1}
    \sim
    \N(0,\Sigma_{1:T})
    \label{eq:penalty_final_noise}
\end{equation}
be an independent final perturbation drawn conditionally on the accumulated
adaptive covariance. For a deterministic sample \(s\), define the adaptive
output sensitivity by
\begin{equation}
    \Delta_{\Sigma_{1:T}}^{\mathrm{ad}}(W_T,s)
    =
    \E
    \left[
        L(W_T,s)-L(W_T+\zeta_T,s)
        \given
        \cH_{T-1}
    \right].
    \label{eq:penalty_delta_def}
\end{equation}
The final penalty appearing in the main generalization bound is
\begin{equation}
    \mathcal R_{\Delta}^{\mathrm{ad}}
    =
    \left|
    \E
    \left[
        \Delta_{\Sigma_{1:T}}^{\mathrm{ad}}(W_T,S')
        -
        \Delta_{\Sigma_{1:T}}^{\mathrm{ad}}(W_T,S)
    \right]
    \right|,
    \label{eq:penalty_rad_def}
\end{equation}
where \(S'\) is an independent ghost sample drawn from the same population
distribution as \(S\).

The role of \(\mathcal R_{\Delta}^{\mathrm{ad}}\) is to quantify the cost of
replacing the original output \(W_T\) by a locally perturbed output
\(W_T+\zeta_T\). If both the training and ghost losses are insensitive to
perturbations drawn from \(\N(0,\Sigma_{1:T})\), then this penalty is small.
Unlike in fixed-noise analysis, the covariance \(\Sigma_{1:T}\) is generally
random and history-dependent. Thus, the penalty depends not only on the local
loss landscape around \(W_T\), but also on how the adaptive covariance process
allocates perturbation mass over the trajectory. This section controls only the output-sensitivity penalty. The
covariance-comparison cost in the information term is a separate effect.

\subsection{Control under Global Smoothness}
\label{subsec:global_smoothness_control}

We first give a simple upper bound under a uniform smoothness assumption.

\begin{assumption}[Uniform smoothness]
\label{ass:uniform_smoothness_output}
For every finite sample \(s\), the empirical risk \(L(\cdot,s)\) is
\(\mu\)-smooth. Equivalently, for all \(w,u\in\R^d\),
\begin{equation}
    \left|
        L(w+u,s)-L(w,s)-\inner{\nabla L(w,s)}{u}
    \right|
    \leq
    \frac{\mu}{2}\norm{u}^2.
    \label{eq:uniform_smoothness_assumption}
\end{equation}
\end{assumption}

\begin{proposition}[Global smoothness control of output sensitivity]
\label{prop:global_smoothness_output_sensitivity}
Under Assumption~\ref{ass:uniform_smoothness_output}, for every deterministic
sample \(s\),
\begin{equation}
    \left|
        \Delta_{\Sigma_{1:T}}^{\mathrm{ad}}(W_T,s)
    \right|
    \leq
    \frac{\mu}{2}
    \Tr(\Sigma_{1:T})
    \qquad
    \text{almost surely}.
    \label{eq:global_smoothness_delta_bound}
\end{equation}
Consequently,
\begin{equation}
    \mathcal R_{\Delta}^{\mathrm{ad}}
    \leq
    \mu\,
    \E
    \left[
        \Tr(\Sigma_{1:T})
    \right].
    \label{eq:global_smoothness_rad_bound}
\end{equation}
\end{proposition}

\begin{proof}
Condition on \(\cH_{T-1}\). By Assumption~\ref{ass:uniform_smoothness_output},
with \(w=W_T\) and \(u=\zeta_T\),
\[
    \left|
        L(W_T+\zeta_T,s)-L(W_T,s)
        -
        \inner{\nabla L(W_T,s)}{\zeta_T}
    \right|
    \leq
    \frac{\mu}{2}\norm{\zeta_T}^2.
\]
Taking conditional expectation and using
\[
    \E[\zeta_T\mid\cH_{T-1}]=0,
\]
we obtain
\[
    \left|
        \E[
            L(W_T,s)-L(W_T+\zeta_T,s)
            \mid
            \cH_{T-1}
        ]
    \right|
    \leq
    \frac{\mu}{2}
    \E[
        \norm{\zeta_T}^2
        \mid
        \cH_{T-1}
    ].
\]
Since
\[
    \zeta_T\mid\cH_{T-1}
    \sim
    \N(0,\Sigma_{1:T}),
\]
we have
\[
    \E[
        \norm{\zeta_T}^2
        \mid
        \cH_{T-1}
    ]
    =
    \Tr(\Sigma_{1:T}),
\]
which proves \eqref{eq:global_smoothness_delta_bound}. Applying the triangle
inequality to \eqref{eq:penalty_rad_def} gives
\[
    \mathcal R_{\Delta}^{\mathrm{ad}}
    \leq
    \E
    \left[
        \left|
        \Delta_{\Sigma_{1:T}}^{\mathrm{ad}}(W_T,S')
        \right|
        +
        \left|
        \Delta_{\Sigma_{1:T}}^{\mathrm{ad}}(W_T,S)
        \right|
    \right]
    \leq
    \mu
    \E[\Tr(\Sigma_{1:T})].
\]
\end{proof}
The global smoothness bound is robust and assumption-light, but it is generally
pessimistic in high-dimensional nonconvex problems. It treats every direction as
if it had curvature \(\mu\), ignoring the anisotropic geometry of the loss
around the final iterate. The next subsection gives a sharper local curvature
refinement.

\subsection{Local Curvature Refinement}
\label{subsec:local_curvature_refinement}

The global bound depends only on \(\Tr(\Sigma_{1:T})\). A more informative
description depends on how the accumulated covariance aligns with the local
Hessian of the empirical or ghost loss. For a finite sample \(s\), define the local Hessian at the final output by
\begin{equation}
    H_s(W_T)
    =
    \nabla^2 L(W_T,s).
    \label{eq:local_hessian_def}
\end{equation}

\begin{assumption}[Hessian-Lipschitz remainder control]
\label{ass:local_hessian_lipschitz}
For a given sample \(s\), assume that \(L(\cdot,s)\) is twice differentiable in
a neighborhood of \(W_T\). Assume further that there exists a constant
\(\rho_s\geq0\), possibly depending on \(s\) and \(W_T\), such that for the
perturbations considered below the Taylor remainder satisfies
\begin{equation}
    \left|
        L(W_T+u,s)-L(W_T,s)
        -\inner{\nabla L(W_T,s)}{u}
        -\frac12 u^\top H_s(W_T)u
    \right|
    \leq
    \frac{\rho_s}{6}\norm{u}^3.
    \label{eq:local_taylor_remainder_assumption}
\end{equation}
A sufficient condition is that the Hessian is globally \(\rho_s\)-Lipschitz in
operator norm. If the Hessian-Lipschitz condition is only local, then
\eqref{eq:local_taylor_remainder_assumption} should be understood on the event
where the Gaussian perturbation remains in the local neighborhood, with the
complement controlled separately.
\end{assumption}

\begin{proposition}[Second-order expansion of adaptive output sensitivity]
\label{prop:local_curvature_delta_expansion}
Suppose Assumption~\ref{ass:local_hessian_lipschitz} holds for sample \(s\). Then
\begin{equation}
    \Delta_{\Sigma_{1:T}}^{\mathrm{ad}}(W_T,s)
    =
    -
    \frac12
    \Tr
    \left(
        H_s(W_T)\Sigma_{1:T}
    \right)
    +
    R_s(W_T,\Sigma_{1:T}),
    \label{eq:local_curvature_delta_expansion}
\end{equation}
where the Taylor remainder satisfies
\begin{equation}
    \left|
        R_s(W_T,\Sigma_{1:T})
    \right|
    \leq
    \frac{\rho_s}{6}
    \E
    \left[
        \norm{\zeta_T}^3
        \given
        \cH_{T-1}
    \right].
    \label{eq:local_curvature_remainder}
\end{equation}
\end{proposition}

\begin{proof}
Condition on \(\cH_{T-1}\). Applying
Assumption~\ref{ass:local_hessian_lipschitz} with \(u=\zeta_T\), we have
\begin{align}
    L(W_T+\zeta_T,s)
    &=
    L(W_T,s)
    +
    \inner{\nabla L(W_T,s)}{\zeta_T}
    +
    \frac12
    \zeta_T^\top H_s(W_T)\zeta_T
    +
    r_s(\zeta_T),
    \label{eq:taylor_local_curvature}
\end{align}
where
\[
    |r_s(\zeta_T)|
    \leq
    \frac{\rho_s}{6}\norm{\zeta_T}^3.
\]
Taking conditional expectations and using
\[
    \E[\zeta_T\mid\cH_{T-1}]=0
\]
and
\[
    \E[
        \zeta_T^\top H_s(W_T)\zeta_T
        \mid
        \cH_{T-1}
    ]
    =
    \Tr(H_s(W_T)\Sigma_{1:T}),
\]
we obtain
\[
    \E[
        L(W_T+\zeta_T,s)-L(W_T,s)
        \mid
        \cH_{T-1}
    ]
    =
    \frac12
    \Tr(H_s(W_T)\Sigma_{1:T})
    +
    \E[r_s(\zeta_T)\mid\cH_{T-1}].
\]
Using
\[
    \Delta_{\Sigma_{1:T}}^{\mathrm{ad}}(W_T,s)
    =
    \E[
        L(W_T,s)-L(W_T+\zeta_T,s)
        \mid
        \cH_{T-1}
    ],
\]
we obtain \eqref{eq:local_curvature_delta_expansion} with
\[
    R_s(W_T,\Sigma_{1:T})
    =
    -\E[r_s(\zeta_T)\mid\cH_{T-1}].
\]
The bound \eqref{eq:local_curvature_remainder} follows immediately.
\end{proof}

\begin{lemma}[Gaussian third-moment control]
\label{lem:gaussian_third_moment_control}
If
\[
    \zeta\sim\N(0,\Sigma),
\]
then
\begin{equation}
    \E\norm{\zeta}^3
    \leq
    \sqrt{3}
    \left[
        \Tr(\Sigma)
    \right]^{3/2}.
    \label{eq:gaussian_third_moment_trace}
\end{equation}
More generally, one may use sharper dimension- or spectrum-sensitive Gaussian
moment bounds depending on the eigenvalues of \(\Sigma\).
\end{lemma}

\begin{proof}
By Cauchy--Schwarz,
\[
    \E\norm{\zeta}^3
    =
    \E\left[\norm{\zeta}^2\norm{\zeta}\right]
    \leq
    \left(\E\norm{\zeta}^4\right)^{1/2}
    \left(\E\norm{\zeta}^2\right)^{1/2}.
\]
For \(\zeta\sim\N(0,\Sigma)\),
\[
    \E\norm{\zeta}^2
    =
    \Tr(\Sigma),
\]
and
\[
    \E\norm{\zeta}^4
    =
    2\Tr(\Sigma^2)+[\Tr(\Sigma)]^2
    \leq
    3[\Tr(\Sigma)]^2.
\]
Therefore,
\[
    \E\norm{\zeta}^3
    \leq
    \sqrt{3}\,[\Tr(\Sigma)]^{3/2}.
\]
\end{proof}

Combining Proposition~\ref{prop:local_curvature_delta_expansion} and
Lemma~\ref{lem:gaussian_third_moment_control} gives the explicit remainder
control
\begin{equation}
    \left|
        R_s(W_T,\Sigma_{1:T})
    \right|
    \leq
    \frac{\sqrt{3}\rho_s}{6}
    \left[
        \Tr(\Sigma_{1:T})
    \right]^{3/2}.
    \label{eq:local_remainder_trace_bound}
\end{equation}

\subsection{Curvature-Mismatch Control of \texorpdfstring{\(\mathcal R_{\Delta}^{\mathrm{ad}}\)}{RDelta}}
\label{subsec:curvature_mismatch_control}

The local expansion reveals that the output-sensitivity penalty depends on the
difference between the curvature of the ghost loss and the curvature of the
training loss, weighted by the accumulated adaptive covariance.

\begin{proposition}[Local curvature control of the adaptive output penalty]
\label{prop:local_curvature_rad_control}
Suppose the local expansion
\eqref{eq:local_curvature_delta_expansion} holds for both \(S\) and \(S'\).
Then
\begin{align}
    \mathcal R_{\Delta}^{\mathrm{ad}}
    &\leq
    \frac12
    \left|
    \E
    \left[
        \Tr
        \left(
            \left[
                H_{S'}(W_T)-H_S(W_T)
            \right]
            \Sigma_{1:T}
        \right)
    \right]
    \right|
    \nonumber\\
    &\quad+
    \E
    \left[
        \left|
            R_{S'}(W_T,\Sigma_{1:T})
        \right|
        +
        \left|
            R_S(W_T,\Sigma_{1:T})
        \right|
    \right].
    \label{eq:local_curvature_rad_control}
\end{align}
In particular, if the local remainder constants for \(S\) and \(S'\) are
\(\rho_S\) and \(\rho_{S'}\), respectively, then
\begin{align}
    \mathcal R_{\Delta}^{\mathrm{ad}}
    &\leq
    \frac12
    \left|
    \E
    \left[
        \Tr
        \left(
            \left[
                H_{S'}(W_T)-H_S(W_T)
            \right]
            \Sigma_{1:T}
        \right)
    \right]
    \right|
    \nonumber\\
    &\quad+
    \frac{\sqrt{3}}{6}
    \E
    \left[
        (\rho_S+\rho_{S'})
        \left[
            \Tr(\Sigma_{1:T})
        \right]^{3/2}
    \right].
    \label{eq:local_curvature_rad_control_trace}
\end{align}
\end{proposition}

\begin{proof}
Using Proposition~\ref{prop:local_curvature_delta_expansion},
\[
    \Delta_{\Sigma_{1:T}}^{\mathrm{ad}}(W_T,S')
    =
    -
    \frac12
    \Tr(H_{S'}(W_T)\Sigma_{1:T})
    +
    R_{S'}(W_T,\Sigma_{1:T}),
\]
and
\[
    \Delta_{\Sigma_{1:T}}^{\mathrm{ad}}(W_T,S)
    =
    -
    \frac12
    \Tr(H_S(W_T)\Sigma_{1:T})
    +
    R_S(W_T,\Sigma_{1:T}).
\]
Subtracting gives
\begin{align*}
    &
    \Delta_{\Sigma_{1:T}}^{\mathrm{ad}}(W_T,S')
    -
    \Delta_{\Sigma_{1:T}}^{\mathrm{ad}}(W_T,S)
    \\
    &\qquad=
    -\frac12
    \Tr
    \left(
        [H_{S'}(W_T)-H_S(W_T)]
        \Sigma_{1:T}
    \right)
    +
    R_{S'}(W_T,\Sigma_{1:T})
    -
    R_S(W_T,\Sigma_{1:T}).
\end{align*}
Taking expectation, absolute values, and applying the triangle inequality proves
\eqref{eq:local_curvature_rad_control}. The trace-form remainder bound follows
from \eqref{eq:local_remainder_trace_bound}.
\end{proof}

This proposition gives a sharper interpretation than the global smoothness
bound. The global bound depends only on the total perturbation mass
\(\Tr(\Sigma_{1:T})\). The local bound depends on where this mass is placed
relative to the curvature mismatch
\[
    H_{S'}(W_T)-H_S(W_T).
\]
If the accumulated covariance places perturbation mass mostly in directions
where the training and ghost losses have similar and small curvature, then the
output-sensitivity penalty can be small even when \(\Tr(\Sigma_{1:T})\) is not
tiny.

\subsection{Positive-Curvature and Flatness Interpretation}
\label{subsec:positive_curvature_flatness}

Near a well-trained solution, local Hessians are often interpreted as measuring
sharpness or flatness. The local expansion formalizes this connection.

\begin{corollary}[Covariance-weighted local flatness]
\label{cor:covariance_weighted_flatness}
Suppose the Taylor remainders in
\eqref{eq:local_curvature_delta_expansion} are negligible and the local Hessian
\(H_s(W_T)\) is positive semidefinite. Then, under the sign convention used in
\(\Delta_{\Sigma_{1:T}}^{\mathrm{ad}}\),
\begin{equation}
    \Delta_{\Sigma_{1:T}}^{\mathrm{ad}}(W_T,s)
    \approx
    -
    \frac12
    \Tr
    \left(
        H_s(W_T)\Sigma_{1:T}
    \right),
\end{equation}
and therefore
\begin{equation}
    \left|
        \Delta_{\Sigma_{1:T}}^{\mathrm{ad}}(W_T,s)
    \right|
    \approx
    \frac12
    \Tr
    \left(
        H_s(W_T)\Sigma_{1:T}
    \right).
    \label{eq:covariance_weighted_flatness}
\end{equation}
Thus, output sensitivity is controlled by a covariance-weighted local curvature
rather than by the isotropic trace \(\Tr(\Sigma_{1:T})\) alone.
\end{corollary}

This expression clarifies the role of anisotropic adaptive perturbations.
Perturbation mass in low-curvature directions contributes little to output
sensitivity, while perturbation mass aligned with high-curvature directions can
increase the penalty. Therefore, a useful adaptive virtual covariance should
not merely be large or small globally; it should allocate perturbation mass in
directions that improve information control without producing excessive
curvature-weighted output sensitivity.

\subsection{Information--Sensitivity Trade-Off}
\label{subsec:information_sensitivity_tradeoff}

The main theorem exposes a trade-off between the information term and the
output-sensitivity penalty. Increasing actual or reference covariance in
appropriate directions can reduce inverse-covariance weighted information
contributions such as
\[
    V_t^{\mathrm{ad}}(\Lambda_t)
    +
    \Gamma_t^{\mathrm{ad}}(\Lambda_t),
    \qquad
    \Lambda_t=(\Sigma_t^{\mathrm{ref}})^{-1}.
\]
From the perspective of mutual information, larger virtual smoothing can make
conditional transition kernels harder to distinguish.

At the same time, increasing the actual adaptive covariance increases the
accumulated covariance
\[
    \Sigma_{1:T}
    =
    \sum_{t=1}^{T-1}\Sigma_t,
\]
which can increase the output-sensitivity penalty. Under global smoothness this
is captured by
\[
    \mathcal R_{\Delta}^{\mathrm{ad}}
    \leq
    \mu
    \E[\Tr(\Sigma_{1:T})].
\]
Under the local curvature refinement, the leading term behaves instead like
\[
    \frac12
    \left|
    \E
    \left[
        \Tr
        \left(
            [H_{S'}(W_T)-H_S(W_T)]
            \Sigma_{1:T}
        \right)
    \right]
    \right|
\]
up to higher-order Taylor remainders. Thus, the relevant question is not only
how large \(\Sigma_{1:T}\) is, but how it aligns with local curvature and
curvature mismatch. Schematically, the adaptive perturbation bound balances
\begin{equation}
    \text{information term}
    \sim
    \sum_{t=1}^{T-1}
    \eta_t^2
    \E
    \left[
        \normSigma{\cdot}{\Lambda_t}^2
    \right],
    \qquad
    \text{output sensitivity}
    \sim
    \E
    \left[
        \Tr(H_T\Sigma_{1:T})
    \right],
    \label{eq:schematic_tradeoff}
\end{equation}
where \(H_T\) denotes a local Hessian or curvature-mismatch proxy near the
final iterate. Larger covariance can improve smoothing for information control,
but may worsen perturbation robustness at the output. Smaller covariance can
reduce the final sensitivity penalty, but may make the information term larger.

The purpose of the history-adaptive framework is not to claim that every
adaptive covariance improves the bound. Rather, it provides a language for
evaluating covariance choices that respond to the optimization path. A useful
adaptive covariance should reduce the information term in directions with large
gradient deviation, stochastic fluctuation, or gradient sensitivity, while
avoiding excessive accumulated perturbation in high-curvature directions near
the final iterate. This is the central information--sensitivity trade-off
exposed by predictable virtual perturbations.

\begin{remark}[Scope of this section]
The results in this section control only
\(\mathcal R_{\Delta}^{\mathrm{ad}}\), the output-sensitivity penalty. They do
not control the covariance-comparison cost \(\mathcal C_t^{\mathrm{cov}}\),
which belongs to the information term and arises when the actual and reference
virtual perturbation covariances differ.
\end{remark}



\section{Discussion: Long-Term Dependence and Predictability}
\label{sec:discussion}

The central conceptual point of this paper is that virtual perturbation analysis
need not restrict the perturbation geometry to be fixed independently of the
optimization history. A covariance chosen at update \(t\) may depend on the past
SGD trajectory, provided that this dependence is causal. In our framework, this
causal condition is predictability with respect to the real SGD history. At the
same time, predictability should not be confused with a complete solution to the
mutual-information comparison. Predictability makes the local conditional
Gaussian smoothing step valid; the full information-theoretic bound also
requires a valid reference-kernel comparison, and in the history-adaptive case
this may introduce an explicit covariance-comparison cost.

This section clarifies the role of long-term dependence, the precise meaning of
predictability, the limitations of the present analysis, and the directions
opened by the framework.

\subsection{Long-Term Dependence in Virtual Perturbation Geometry}
\label{subsec:discussion_long_term_dependence}

Fixed virtual perturbation analysis avoids long-term dependence by choosing the
perturbation covariance sequence before the optimization trajectory is observed.
For example, one may set
\begin{equation}
    \Sigma_t=\sigma_t^2I,
    \qquad
    t=1,\ldots,T-1,
\end{equation}
or choose a deterministic matrix-valued sequence
\begin{equation}
    \Sigma_1,\ldots,\Sigma_{T-1}.
\end{equation}
This makes the perturbation process technically convenient because the noise law
is independent of the realized SGD path and of the training sample values. However, this independence is restrictive if the purpose of the perturbation is
to probe the geometry encountered along training. In realistic stochastic
optimization dynamics, many useful scale and geometry estimates are
history-dependent. Examples include moving averages of gradient norms,
coordinate-wise second-moment estimates, gradient-deviation diagnostics,
subbatch fluctuation proxies, curvature surrogates, and preconditioner-like
quantities. A natural adaptive virtual perturbation covariance may therefore
take the form
\begin{equation}
    \Sigma_t
    =
    \Phi_t(G_1,\ldots,G_{t-1}),
    \label{eq:discussion_gradient_history_cov}
\end{equation}
or more generally
\begin{equation}
    \Sigma_t
    =
    \Phi_t(\cH_{t-1}),
    \label{eq:discussion_history_cov}
\end{equation}
where \(\cH_{t-1}\) denotes the real SGD history available before update \(t\).

Equations \eqref{eq:discussion_gradient_history_cov} and
\eqref{eq:discussion_history_cov} explicitly introduce long-term dependence into
the virtual perturbation process. The covariance at time \(t\) may depend on
information accumulated over many previous iterations. Thus, the virtual
perturbation geometry can reflect the pathwise behavior of SGD rather than a
fixed geometry chosen before training. Importantly, this does not change the actual SGD update:
\begin{equation}
    W_{t+1}
    =
    W_t-\eta_tG_t,
    \qquad
    t=1,\ldots,T-1.
\end{equation}
The long-term dependence considered here belongs to the analytical perturbation
process, not to the optimizer itself. The true SGD trajectory remains
unchanged; only the virtual noise used to analyze it is allowed to adapt to the
past.

\subsection{Predictability and Reference Comparisons}
\label{subsec:discussion_predictability_reference}

The key local observation is that full independence from the optimization
history is stronger than necessary for the conditional Gaussian smoothing step.
The conditional relative-entropy argument does not require \(\Sigma_t\) to be
deterministic. It requires that \(\Sigma_t\) be known before the current virtual
perturbation is sampled. Formally, the predictability condition is
\begin{equation}
    \Sigma_t
    \text{ is }
    \cH_{t-1}\text{-measurable}.
    \label{eq:discussion_predictability}
\end{equation}
Under this condition,
\begin{equation}
    \eps_t\given \cH_{t-1}
    \sim
    \N(0,\Sigma_t).
    \label{eq:discussion_conditional_gaussian}
\end{equation}
Once we condition on \(\cH_{t-1}\), the covariance \(\Sigma_t\) is fixed, even
though it may be random before conditioning. This is the structure needed to
apply conditional Gaussian smoothing and conditional relative-entropy bounds at
a single step.

Predictability separates allowed dependence from forbidden dependence. Allowed
dependence includes functions of past iterates, past minibatch indices, past
minibatches, past gradients, and statistics computed from the past trajectory.
Forbidden dependence includes the current virtual perturbation, current
unrevealed minibatch randomness, the current stochastic gradient \(G_t\), future
gradients, future iterates, or the final output. In this sense, predictability
is the causal condition that permits the perturbation geometry to adapt to
long-term historical information while preserving the local innovation
structure of the proof.

This viewpoint is analogous to predictable processes in stochastic analysis and
martingale theory. A predictable process may depend on the past but not on the
next innovation. Here, \(\Sigma_t\) plays the role of a predictable process,
while \(\eps_t\) is the next virtual innovation. The conditional Gaussian law in
\eqref{eq:discussion_conditional_gaussian} is therefore the
virtual-perturbation analogue of a martingale-style innovation structure.

However, predictability alone does not automatically yield the clean
fixed-noise-style mutual-information bound. Because \(\Sigma_t\) may depend on
the training data through \(\cH_{t-1}\), the reference kernel used in the
chain-rule decomposition cannot be assumed to have access to the same covariance
without justification. The general theorem therefore compares the actual
virtual transition kernel with an admissible reference kernel. If the reference
covariance is denoted by \(\Sigma_t^{\mathrm{ref}}\), then the mean-discrepancy
part of the one-step comparison is measured in the reference precision
\begin{equation}
    \Lambda_t
    =
    (\Sigma_t^{\mathrm{ref}})^{-1},
\end{equation}
and the covariance-comparison cost is
\begin{equation}
    \mathcal C_t^{\mathrm{cov}}
    =
    \frac12
    \E
    \left[
        \Tr((\Sigma_t^{\mathrm{ref}})^{-1}\Sigma_t)
        -d
        +
        \log
        \frac{\det\Sigma_t^{\mathrm{ref}}}{\det\Sigma_t}
    \right].
    \label{eq:discussion_covariance_cost}
\end{equation}

When the comparison is admissibly synchronized, meaning that the reference
kernel can reconstruct the same covariance from the reference-visible
information
\[
    \mathcal F_t^Q
    =
    \sigma(\wtW_{1:t})\vee\mathcal U,
\]
one may choose
\[
    \Sigma_t^{\mathrm{ref}}=\Sigma_t.
\]
In that case, \(\mathcal C_t^{\mathrm{cov}}\) vanishes and the clean
fixed-noise-style corollary is recovered. This synchronization is automatic for
deterministic fixed covariances, and may also hold for public or
prefix-observable covariance constructions. It is not automatic for arbitrary
data-dependent history-adaptive covariances. For fully history-adaptive
perturbations, the general covariance-comparison theorem is the safe default.

\subsection{Output Sensitivity and Local Curvature}
\label{subsec:discussion_output_sensitivity_curvature}

The covariance process also affects the output-sensitivity penalty through the
accumulated covariance
\[
    \Sigma_{1:T}
    =
    \sum_{t=1}^{T-1}\Sigma_t.
\]
The penalty
\[
    \mathcal R_{\Delta}^{\mathrm{ad}}
\]
measures the cost of transferring the guarantee from a virtually perturbed
output back to the original SGD output. Under global smoothness, this term is
controlled by a trace bound of the form
\[
    \mathcal R_{\Delta}^{\mathrm{ad}}
    \leq
    \mu\E[\Tr(\Sigma_{1:T})].
\]
This control is robust but often pessimistic because it treats all directions as
having worst-case curvature. The local curvature refinement gives a sharper interpretation. Up to
higher-order Taylor remainders, the output-sensitivity term is governed by a
covariance-weighted curvature mismatch:
\[
    \frac12
    \left|
    \E
    \left[
        \Tr
        \left(
            [H_{S'}(W_T)-H_S(W_T)]
            \Sigma_{1:T}
        \right)
    \right]
    \right|.
\]
Thus, the accumulated covariance should not be judged only by its trace. Its
alignment with the local curvature of the training and ghost losses matters.
Perturbation mass in low-curvature or curvature-matched directions may have a
small output-sensitivity cost, while perturbation mass in high-curvature
directions may increase the penalty substantially. This observation complements the covariance-comparison discussion above. The
information term is influenced by inverse-covariance weighted
gradient-deviation and gradient-sensitivity terms, whereas the output penalty is
influenced by the accumulated covariance itself and its alignment with local
curvature. The framework therefore exposes an information--sensitivity
trade-off rather than a universally optimal adaptive covariance rule.

\subsection{What Is Still Not Covered}
\label{subsec:discussion_not_covered}

The predictable framework is deliberately limited. It relaxes fixed
perturbation geometry, but it does not cover all forms of adaptivity. Several
important cases remain outside the direct scope of the present theorem.

First, if \(\Sigma_t\) depends on the current minibatch \(B_t\), the current
minibatch index \(J_t\), or the current stochastic gradient \(G_t\), then it is
generally not \(\cH_{t-1}\)-measurable under the timing convention of this
paper. In that case, the covariance is not fixed before the randomness of the
current step is revealed, and the conditional Gaussian comparison used in the
proof no longer applies directly. Analyzing such choices would require
enlarging the filtration, changing the timing of the update, or introducing
additional correction terms.

Second, if \(\Sigma_t\) depends on future iterates, the final output, validation
data, or test data, then the causal structure is broken. Such choices may
introduce uncontrolled dependence between the perturbation process and
quantities that the proof treats as unrevealed. The theorem therefore does not
apply to noncausal covariance rules. Third, the current main theorem assumes positive definite covariances. If
\(\Sigma_t\) is singular or genuinely low-rank, then the inverse norm
\begin{equation}
    \normSigma{x}{\Sigma_t^{-1}}^2
\end{equation}
is not directly defined. One may handle such cases by adding a ridge term,
\begin{equation}
    \Sigma_t+\lambda I,
    \qquad
    \lambda>0,
\end{equation}
or by using Moore--Penrose pseudo-inverses together with support conditions.
These extensions are natural but require separate technical treatment. Fourth, when the actual and reference covariance processes differ in the
one-step comparison, the covariance-comparison term in
\eqref{eq:discussion_covariance_cost} must be included. Moreover, the
corresponding gradient-deviation and sensitivity terms are measured in the
reference geometry
\[
    \Lambda_t=(\Sigma_t^{\mathrm{ref}})^{-1}.
\]
Ignoring this distinction would incorrectly treat a data-dependent adaptive
covariance as if it were a deterministic fixed covariance. The
covariance-comparison term is not a defect of the framework; it is the correct
KL cost of comparing Gaussian kernels with different covariances.

Finally, our analysis does not by itself prove generalization bounds for fully
adaptive optimizers such as Adam or AdaGrad. Those methods introduce long-term
dependence into the actual update rule, not only into the virtual perturbation
geometry. The present paper addresses a specific analytical subproblem: how to
allow the virtual noise geometry in SGD analysis to depend predictably on the
past. Combining this tool with a full analysis of adaptive update dynamics is an
important direction for future work.

\subsection{Why This Matters}
\label{subsec:discussion_why_matters}

The predictable perturbation framework provides a modular way to introduce
history dependence into information-theoretic generalization analysis. Although
the present paper focuses on SGD with adaptive virtual perturbations, the same
principle suggests several future directions.

First, adaptive learning-rate analyses may benefit from separating the
adaptivity of the true update from the adaptivity of the virtual perturbation
geometry. Even when the optimizer uses a fixed update rule, the proof may choose
a perturbation covariance that reflects past gradient scales. This offers a
bridge between fixed-noise virtual perturbation bounds and analyses of adaptive
stochastic dynamics.

Second, preconditioned SGD and Adam-like methods naturally generate geometry
through accumulated gradient statistics. The predictable covariance framework
shows how such statistics can be used analytically without necessarily changing
the optimizer. For example, a diagonal covariance based on past second-moment
estimates can be used as a virtual geometry, while the actual SGD recursion
remains unchanged. A separate analysis would be needed to handle optimizers
whose actual updates also use such preconditioners.

Third, fractional-memory and history-dependent stochastic methods involve
explicit dependence on past gradients. Although the present theorem does not
analyze fractional updates directly, it provides a template for handling
predictable objects generated from long histories. A natural next step is to
combine predictable virtual perturbations with memory-dependent update maps.

Fourth, differentially private SGD and related private learning algorithms often
calibrate noise to sensitivity, clipping norms, or gradient-deviation and
fluctuation estimates. A predictable perturbation viewpoint may help separate
the role of actual privacy noise from auxiliary analytical perturbations. This
could be useful for studying privacy, stability, and generalization under
adaptive noise calibration.

Fifth, sharper covariance-comparison bounds may improve the usefulness of
history-adaptive perturbation analysis. The term
\(\mathcal C_t^{\mathrm{cov}}\) is the exact Gaussian KL cost for comparing
covariances, but particular covariance families may allow sharper or more
interpretable controls. Similarly, low-rank or structured perturbation
geometries may reduce dimensional dependence if the corresponding singular or
nearly singular covariance analysis can be developed carefully.

The broader message is that fixed perturbation schedules are not the only
analytically tractable option. History-adaptive perturbations can be
incorporated into virtual perturbation analysis when their dependence on the
trajectory is predictable, and when the reference-kernel comparison is handled
with the appropriate covariance-comparison cost. This opens a controlled path
toward information-theoretic analyses of optimization procedures where
geometry, scale, and sensitivity evolve over time.

At the same time, the framework should not be interpreted as solving all
adaptive optimization generalization problems. It resolves one specific
obstacle: allowing virtual perturbation covariances to depend on past SGD
history while preserving the local conditional relative-entropy argument and
making explicit the additional cost required by the global mutual-information
comparison. This modular contribution can be combined with future analyses of
adaptive updates, preconditioning, memory, and privacy mechanisms.

\section{Numerical Diagnostic Illustration}
\label{sec:numerical}

This section describes an optional numerical diagnostic framework for visualizing
the quantities that appear in the proposed information-theoretic bounds. The
purpose is not to empirically prove the theorem; the theorem is a mathematical
statement. Rather, the goal is to make the components of the bound visible along
a recorded SGD trajectory. Throughout this section, the true training algorithm remains vanilla SGD:
\begin{equation}
    W_{t+1}
    =
    W_t-\eta_tG_t,
    \qquad
    G_t=g(W_t,B_t),
    \qquad
    t=1,\ldots,T-1.
    \label{eq:numerical_sgd}
\end{equation}
No virtual perturbation is injected into the update. The perturbation
covariances considered below are analytical objects used only to estimate proxy
terms for the theoretical quantities. Thus, the numerical quantities in this
section should be interpreted as diagnostics, not as certified numerical
generalization bounds.

\subsection{Purpose of the Diagnostic Illustration}
\label{subsec:purpose_diagnostic_illustration}

The main theorem involves the reference-geometry quantities
\begin{equation}
    V_t^{\mathrm{ad}}(\Lambda_t),
    \qquad
    \Gamma_t^{\mathrm{ad}}(\Lambda_t),
    \qquad
    \mathcal C_t^{\mathrm{cov}},
    \qquad
    \mathcal R_{\Delta}^{\mathrm{ad}},
\end{equation}
where
\begin{equation}
    \Lambda_t=(\Sigma_t^{\mathrm{ref}})^{-1}.
\end{equation}
In the admissibly synchronized case,
\[
    \Sigma_t^{\mathrm{ref}}=\Sigma_t,
    \qquad
    \Lambda_t=\Sigma_t^{-1},
    \qquad
    \mathcal C_t^{\mathrm{cov}}=0.
\]
In the general covariance-comparison theorem, the admissible reference
covariance \(\Sigma_t^{\mathrm{ref}}\) may differ from the actual predictable
covariance \(\Sigma_t\), and the covariance-comparison cost must be included.

The exact theoretical quantities are generally not computable in finite
experiments. They involve population gradients, conditional expectations,
Gaussian averages over accumulated adaptive covariances, and admissibility
conditions for reference kernels. The goal of this section is therefore to
define practical proxies that approximate the behavior of these terms.

The diagnostics have four purposes:
\begin{enumerate}[leftmargin=2em]
    \item to visualize gradient-deviation or fluctuation proxies along training;
    \item to visualize gradient-sensitivity proxies under accumulated virtual
    perturbations;
    \item to compare fixed and predictable adaptive virtual covariance schedules;
    \item to display the trade-off between inverse-covariance weighted
    information proxies, covariance-comparison cost, and output sensitivity.
\end{enumerate}

These diagnostics should not be interpreted as evidence that one covariance
schedule is universally superior, nor as evidence that the optimizer itself has
been improved. The optimizer is unchanged; only the analytical perturbation
geometry used in the diagnostic computation changes.

\subsection{Diagnostic Experimental Setup}
\label{subsec:diagnostic_experimental_setup}

A simple setup is sufficient. One may train a small multilayer perceptron or
convolutional neural network on a standard dataset such as MNIST or CIFAR-10
using vanilla SGD. During training, record at selected checkpoints:
\[
    W_t,
    \qquad
    B_t,
    \qquad
    G_t,
    \qquad
    \eta_t,
    \qquad
    L(W_t,S),
    \qquad
    L(W_t,S^{\mathrm{eval}}),
\]
where \(S^{\mathrm{eval}}\) is an independent validation, test, or ghost sample
used only for diagnostics. The virtual covariance schedules are evaluated alongside the SGD trajectory but
are not used in the optimizer. Representative choices include:
\begin{align}
    \text{fixed isotropic:}
    \qquad
    &\Sigma_t=\sigma^2I,
    \label{eq:num_fixed_iso}
    \\
    \text{adaptive scalar:}
    \qquad
    &\Sigma_t=\sigma_t^2I,
    \qquad
    \sigma_t^2=\sigma_0^2(1+c\,\widehat q_{t-1}),
    \label{eq:num_adaptive_scalar}
    \\
    \text{adaptive diagonal:}
    \qquad
    &\Sigma_t=
    \diag(s_{t,1}^2,\ldots,s_{t,d}^2),
    \label{eq:num_adaptive_diag}
    \\
    \text{Adam-like virtual geometry:}
    \qquad
    &v_t=\beta v_{t-1}+(1-\beta)G_{t-1}\odot G_{t-1}.
    \label{eq:num_adam_like_v}
\end{align}
Here \(\widehat q_{t-1}\) denotes any scalar pathwise statistic computed only
from the past history \(\cH_{t-1}\), such as a moving average of gradient norms
or a fluctuation proxy from previous updates. We use \(\widehat q_{t-1}\) rather
than \(\widehat v_{t-1}\) to avoid confusing this pathwise scale statistic with
the theoretical gradient-deviation term \(V_t^{\mathrm{ad}}\).

For the Adam-like case, one may use either a proportional covariance
\begin{equation}
    \Sigma_t
    =
    \diag
    \left(
        \rho_t^2(\sqrt{v_t}+\epsilon)
    \right)
    +
    \lambda_0 I,
    \label{eq:num_adam_prop}
\end{equation}
or an inverse-scale covariance
\begin{equation}
    \Sigma_t
    =
    \diag
    \left(
        \frac{\rho_t^2}{\sqrt{v_t}+\epsilon}
    \right)
    +
    \lambda_0 I.
    \label{eq:num_adam_inv}
\end{equation}
The ridge term \(\lambda_0I\) ensures positive definiteness. All adaptive
schedules must be predictable: \(\Sigma_t\) must be computed only from
information available before update \(t\), namely from \(\cH_{t-1}\). For
example, \(v_t\) in \eqref{eq:num_adam_like_v} uses \(G_{t-1}\), not \(G_t\). For a fixed deterministic covariance, the synchronized reference covariance
\[
    \Sigma_t^{\mathrm{ref}}=\Sigma_t
\]
is admissible. For a data-dependent adaptive covariance, synchronized
diagnostics may still be visually informative, but the clean admissibly
synchronized theoretical bound should be invoked only when an admissible
synchronization certificate is available. Otherwise, one should specify an
admissible reference covariance \(\Sigma_t^{\mathrm{ref}}\) and use the general
covariance-comparison diagnostic.

\subsection{Estimating the Bound Proxies}
\label{subsec:estimating-bound-proxies}

The quantities appearing in the theoretical bounds are generally not exactly
computable in finite experiments. They involve population gradients, conditional
expectations, Gaussian averages over accumulated adaptive covariance, and, in
the general covariance-comparison theorem, admissible reference geometries.
Therefore, the goal of this subsection is not to compute the bound exactly, but
to define diagnostics that approximate its main components.

A key distinction is important. The theoretical term
\[
    V_t^{\mathrm{ad}}
    =
    \E
    \left[
        \normSigma{
            G_t-\barg(W_t)
        }{\Sigma_t^{-1}}^2
        \given
        \cH_{t-1}
    \right]
\]
is a conditional mean-square deviation from the population gradient. It is a
conditional variance only under the unbiasedness condition
\[
    \E[G_t\mid\cH_{t-1}]
    =
    \barg(W_t).
\]
Consequently, numerical proxies based only on subbatch differences estimate
conditional stochastic fluctuation, not the full gradient-deviation term when
conditional bias is present.

\paragraph{Population-gradient or ghost-batch deviation proxy.}
If a large evaluation batch or an independent ghost batch
\[
    B_t^{\mathrm{eval}}
\]
is available, approximate the population gradient by
\begin{equation}
    \widehat{\barg}_t
    =
    g(W_t,B_t^{\mathrm{eval}}).
    \label{eq:population_gradient_proxy}
\end{equation}
Then a synchronized total gradient-deviation proxy is
\begin{equation}
    \widehat V_t^{\mathrm{dev}}
    =
    \normSigma{
        G_t-\widehat{\barg}_t
    }{\Sigma_t^{-1}}^2.
    \label{eq:gradient_deviation_proxy}
\end{equation}
If a reference covariance \(\Sigma_t^{\mathrm{ref}}\) is used, define
\[
    \Lambda_t=(\Sigma_t^{\mathrm{ref}})^{-1},
\]
and use the reference-geometry deviation proxy
\begin{equation}
    \widehat V_t^{\mathrm{dev,ref}}
    =
    \normSigma{
        G_t-\widehat{\barg}_t
    }{\Lambda_t}^2.
    \label{eq:gradient_deviation_proxy_ref}
\end{equation}
These proxies estimate the total discrepancy from a population-gradient
surrogate and therefore capture both stochastic fluctuation and possible bias
relative to the evaluation-batch approximation.

\paragraph{Subbatch fluctuation-only proxy.}
When an independent evaluation or ghost batch is unavailable, one may use
subbatch differences as a fluctuation-only diagnostic. Split the current
minibatch into two disjoint subbatches,
\[
    B_t^{(1)}
    \quad\text{and}\quad
    B_t^{(2)},
\]
and define
\[
    G_t^{(1)}=g(W_t,B_t^{(1)}),
    \qquad
    G_t^{(2)}=g(W_t,B_t^{(2)}).
\]
A synchronized fluctuation-only proxy is
\begin{equation}
    \widehat V_t^{\mathrm{fluc}}
    =
    \frac12
    \normSigma{
        G_t^{(1)}-G_t^{(2)}
    }{\Sigma_t^{-1}}^2.
    \label{eq:variance_proxy}
\end{equation}
The factor \(1/2\) is appropriate when the two subbatch gradients are
conditionally independent and identically distributed with the same covariance.
For unequal subbatch sizes, the scaling should be adjusted according to the
known covariance scaling of the two subbatch estimators. Similarly, the reference-geometry fluctuation-only proxy is
\begin{equation}
    \widehat V_t^{\mathrm{fluc,ref}}
    =
    \frac12
    \normSigma{
        G_t^{(1)}-G_t^{(2)}
    }{\Lambda_t}^2.
    \label{eq:variance_proxy_ref}
\end{equation}
The labels \eqref{eq:variance_proxy} and \eqref{eq:variance_proxy_ref} are kept
for compatibility with earlier references, but the quantities should be
interpreted as fluctuation-only proxies rather than full gradient-deviation
proxies. With more than two subbatches, one may use the empirical covariance around the
subbatch mean. If \(B_t^{(1)},\ldots,B_t^{(K)}\) are disjoint subbatches and
\[
    G_t^{(k)}=g(W_t,B_t^{(k)}),
    \qquad
    \bar G_t^{\mathrm{sub}}
    =
    \frac1K\sum_{k=1}^K G_t^{(k)},
\]
then a synchronized empirical fluctuation proxy is
\begin{equation}
    \widehat V_t^{\mathrm{fluc},K}
    =
    \frac{1}{K-1}
    \sum_{k=1}^K
    \normSigma{
        G_t^{(k)}-\bar G_t^{\mathrm{sub}}
    }{\Sigma_t^{-1}}^2.
    \label{eq:multi_subbatch_fluctuation_proxy}
\end{equation}
The corresponding reference-geometry version replaces \(\Sigma_t^{-1}\) by
\(\Lambda_t\).

\paragraph{Interpretation of deviation versus fluctuation proxies.}
The direct deviation proxy \(\widehat V_t^{\mathrm{dev}}\) targets
\[
    \normSigma{
        G_t-\barg(W_t)
    }{\Sigma_t^{-1}}^2
\]
through a population-gradient approximation. In contrast, the subbatch proxy
\(\widehat V_t^{\mathrm{fluc}}\) estimates only the conditional fluctuation of
minibatch gradients around their conditional mean. It does not estimate the
squared conditional bias term
\[
    \normSigma{
        \E[G_t\mid\cH_{t-1}]-\barg(W_t)
    }{\Sigma_t^{-1}}^2.
\]
Thus, fluctuation proxies are faithful proxies for \(V_t^{\mathrm{ad}}\) only
when conditional unbiasedness is justified or when the conditional bias is
negligible.

\paragraph{Gradient-sensitivity proxy.}
The adaptive gradient-sensitivity term measures how the population gradient
changes under accumulated virtual perturbations. Let
\[
    \Sigma_{1:t}
    =
    \sum_{k=1}^{t-1}\Sigma_k.
\]
At selected checkpoints, draw Monte Carlo perturbations
\begin{equation}
    \zeta_{t,j}
    \sim
    \N(0,\Sigma_{1:t}),
    \qquad
    j=1,\ldots,m.
    \label{eq:num_intermediate_perturbations}
\end{equation}
Using the same evaluation batch \(B_t^{\mathrm{eval}}\) to approximate
population gradients, define the synchronized sensitivity proxy
\begin{equation}
    \widehat\Gamma_t
    =
    \frac1m
    \sum_{j=1}^{m}
    \normSigma{
        g(W_t+\zeta_{t,j},B_t^{\mathrm{eval}})
        -
        g(W_t,B_t^{\mathrm{eval}})
    }{\Sigma_t^{-1}}^2.
    \label{eq:sensitivity_proxy}
\end{equation}
If a reference covariance is used, the reference-geometry sensitivity proxy is
\begin{equation}
    \widehat\Gamma_t^{\mathrm{ref}}
    =
    \frac1m
    \sum_{j=1}^{m}
    \normSigma{
        g(W_t+\zeta_{t,j},B_t^{\mathrm{eval}})
        -
        g(W_t,B_t^{\mathrm{eval}})
    }{\Lambda_t}^2.
    \label{eq:sensitivity_proxy_ref}
\end{equation}
If no evaluation or ghost batch is available, this sensitivity proxy should not
be reported unless a separate held-out minibatch or population-gradient surrogate
is specified.

\paragraph{Covariance-comparison proxy.}
For the general covariance-comparison theorem, the actual covariance
\(\Sigma_t\) and the admissible reference covariance
\(\Sigma_t^{\mathrm{ref}}\) may differ. The corresponding one-step Gaussian
covariance-comparison cost is
\[
    \mathcal C_t^{\mathrm{cov}}
    =
    \frac12
    \E
    \left[
        \Tr
        \left(
            (\Sigma_t^{\mathrm{ref}})^{-1}\Sigma_t
        \right)
        -d
        +
        \log
        \frac{\det\Sigma_t^{\mathrm{ref}}}{\det\Sigma_t}
    \right].
\]
A pointwise diagnostic at iteration \(t\) is
\begin{equation}
    \widehat{\mathcal C}_t^{\mathrm{cov}}
    =
    \frac12
    \left[
        \Tr
        \left(
            (\Sigma_t^{\mathrm{ref}})^{-1}\Sigma_t
        \right)
        -d
        +
        \log
        \frac{\det\Sigma_t^{\mathrm{ref}}}{\det\Sigma_t}
    \right].
    \label{eq:covariance_mismatch_proxy}
\end{equation}
When the covariance is admissibly synchronized,
\[
    \Sigma_t^{\mathrm{ref}}=\Sigma_t,
\]
this proxy is exactly zero. If synchronization is used only as a diagnostic
without an admissibility certificate, then
\(\widehat{\mathcal C}_t^{\mathrm{cov}}=0\) should not be interpreted as a valid
theoretical covariance-comparison cost.

\paragraph{Output-sensitivity proxy.}
The adaptive output-sensitivity penalty transfers the generalization guarantee
from the perturbed output back to the original SGD output. Let
\[
    \Sigma_{1:T}
    =
    \sum_{k=1}^{T-1}\Sigma_k.
\]
Draw final perturbations
\[
    \zeta_{T,j}
    \sim
    \N(0,\Sigma_{1:T}),
    \qquad
    j=1,\ldots,m_T.
\]
For a validation or ghost sample \(S^{\mathrm{eval}}\), define
\begin{equation}
    \widehat\Delta_T(S^{\mathrm{eval}})
    =
    \frac1{m_T}
    \sum_{j=1}^{m_T}
    \left[
        L(W_T,S^{\mathrm{eval}})
        -
        L(W_T+\zeta_{T,j},S^{\mathrm{eval}})
    \right].
    \label{eq:output_sensitivity_proxy}
\end{equation}
This proxy estimates the local loss sensitivity of the final output under the
accumulated adaptive virtual covariance. To approximate the difference between
training and population output sensitivity, one may compute the same quantity on
both the training sample and an independent evaluation or ghost sample and use
the penalty-difference proxy
\begin{equation}
    \widehat{\mathcal R}_{\Delta}^{\mathrm{ad}}
    =
    \left|
        \widehat\Delta_T(S^{\mathrm{eval}})
        -
        \widehat\Delta_T(S)
    \right|.
    \label{eq:output_penalty_difference_proxy}
\end{equation}
If only one evaluation sample is available, then
\(\widehat\Delta_T(S^{\mathrm{eval}})\) should be interpreted as an
output-sensitivity diagnostic rather than as an estimate of the exact
train--ghost penalty difference.

\paragraph{Information-sensitivity diagnostic.}
A practical diagnostic corresponding to the dominant terms in the bound is
\begin{equation}
    \widehat{\mathcal B}
    =
    \sum_{t\in\mathcal T}
    \left[
        2\eta_t^2
        \left(
            \widehat V_t^{\star}
            +
            \widehat\Gamma_t^{\star}
        \right)
        +
        \widehat{\mathcal C}_t^{\mathrm{cov}}
    \right]
    +
    \widehat{\mathcal R}_{\Delta}^{\mathrm{ad}},
    \label{eq:bound_proxy_summary}
\end{equation}
where \(\mathcal T\) denotes the set of checkpoints at which the proxies are
computed. The symbol \(\widehat V_t^{\star}\) denotes whichever
gradient-deviation or fluctuation proxy is used:
\[
    \widehat V_t^{\star}
    \in
    \left\{
        \widehat V_t^{\mathrm{dev}},
        \widehat V_t^{\mathrm{dev,ref}},
        \widehat V_t^{\mathrm{fluc}},
        \widehat V_t^{\mathrm{fluc,ref}}
    \right\}.
\]
Similarly, \(\widehat\Gamma_t^{\star}\) denotes either the synchronized
sensitivity proxy \(\widehat\Gamma_t\) or the reference-geometry sensitivity
proxy \(\widehat\Gamma_t^{\mathrm{ref}}\), depending on whether the synchronized
or general covariance-comparison diagnostic is being reported. The choice of \(\widehat V_t^{\star}\), \(\widehat\Gamma_t^{\star}\), and
\(\widehat{\mathcal R}_{\Delta}^{\mathrm{ad}}\) should be reported explicitly.
In particular, fluctuation-only proxies should not be described as estimating
the full gradient-deviation term unless conditional unbiasedness or negligible
conditional bias is justified. The diagnostic \(\widehat{\mathcal B}\) should not be interpreted as a certified
numerical generalization bound. It is a visualization tool for comparing the
relative behavior of gradient deviation, gradient sensitivity,
covariance-comparison cost, and output sensitivity under different virtual
perturbation geometries.

\begin{algorithm}[t]
\scriptsize
\caption{Proxy estimation for adaptive virtual perturbation diagnostics}
\label{alg:proxy_estimation}
\begin{algorithmic}[1]
\Require SGD trajectory \(\{W_t,G_t,B_t\}_{t=1}^{T}\), learning rates \(\{\eta_t\}_{t=1}^{T-1}\), actual covariance rule \(\Sigma_t=\Phi_t(\cH_{t-1})\), optional admissible reference covariance rule \(\Sigma_t^{\mathrm{ref}}\), optional evaluation or ghost batches \(B_t^{\mathrm{eval}}\), Monte Carlo sample size \(m\)
\State Initialize \(\Sigma_{1:1}=0\)
\For{\(t=1,\ldots,T-1\)}
    \State Compute predictable actual covariance \(\Sigma_t=\Phi_t(\cH_{t-1})\)
    \If{using a general reference geometry}
        \State Compute admissible \(\Sigma_t^{\mathrm{ref}}\) and set \(\Lambda_t=(\Sigma_t^{\mathrm{ref}})^{-1}\)
    \Else
        \State Set \(\Lambda_t=\Sigma_t^{-1}\)
    \EndIf

    \If{an evaluation or ghost batch \(B_t^{\mathrm{eval}}\) is available}
        \State Compute \(\widehat{\barg}_t=g(W_t,B_t^{\mathrm{eval}})\)
        \State Compute total gradient-deviation proxy
        \[
            \widehat V_t^{\mathrm{dev}}
            =
            \normSigma{
                G_t-\widehat{\barg}_t
            }{\Sigma_t^{-1}}^2
        \]
        \If{using a reference geometry}
            \State Compute reference-geometry total gradient-deviation proxy
            \[
                \widehat V_t^{\mathrm{dev,ref}}
                =
                \normSigma{
                    G_t-\widehat{\barg}_t
                }{\Lambda_t}^2
            \]
        \EndIf
    \Else
        \State Split \(B_t\) into \(B_t^{(1)},B_t^{(2)}\) and compute \(G_t^{(1)},G_t^{(2)}\)
        \State Compute fluctuation-only proxy
        \[
            \widehat V_t^{\mathrm{fluc}}
            =
            \frac12
            \normSigma{
                G_t^{(1)}-G_t^{(2)}
            }{\Sigma_t^{-1}}^2
        \]
        \If{using a reference geometry}
            \State Compute reference-geometry fluctuation-only proxy
            \[
                \widehat V_t^{\mathrm{fluc,ref}}
                =
                \frac12
                \normSigma{
                    G_t^{(1)}-G_t^{(2)}
                }{\Lambda_t}^2
            \]
        \EndIf
    \EndIf

    \State Draw \(\zeta_{t,j}\sim\N(0,\Sigma_{1:t})\), \(j=1,\ldots,m\)

    \If{an evaluation or ghost batch \(B_t^{\mathrm{eval}}\) is available}
        \State Estimate synchronized gradient-sensitivity proxy using \eqref{eq:sensitivity_proxy}
        \If{using a reference geometry}
            \State Estimate reference-geometry gradient-sensitivity proxy using \eqref{eq:sensitivity_proxy_ref}
        \EndIf
    \Else
        \State Report that the gradient-sensitivity proxy was not computed, or compute it using a separately specified held-out minibatch
    \EndIf

    \If{using a general reference geometry}
        \State Compute \(\widehat{\mathcal C}_t^{\mathrm{cov}}\) using \eqref{eq:covariance_mismatch_proxy}
    \Else
        \State Set \(\widehat{\mathcal C}_t^{\mathrm{cov}}=0\) only if synchronization is admissible
    \EndIf

    \State Update accumulated covariance \(\Sigma_{1:t+1}=\Sigma_{1:t}+\Sigma_t\)
\EndFor
\State Draw \(\zeta_{T,j}\sim\N(0,\Sigma_{1:T})\), \(j=1,\ldots,m\)
\State Estimate output-sensitivity proxy \(\widehat\Delta_T\) using \eqref{eq:output_sensitivity_proxy}
\State If both training and evaluation/ghost samples are available, estimate \(\widehat{\mathcal R}_{\Delta}^{\mathrm{ad}}\) using \eqref{eq:output_penalty_difference_proxy}
\State Compute \(\widehat{\mathcal B}\) using \eqref{eq:bound_proxy_summary}, reporting whether \(\widehat V_t^{\star}\) is a deviation proxy or a fluctuation-only proxy and whether \(\widehat\Gamma_t^{\star}\) is synchronized or reference-geometry
\end{algorithmic}
\end{algorithm}

\subsection{Fixed versus Adaptive Perturbation Schedules}
\label{subsec:fixed_versus_adaptive_schedules}

The numerical comparison can evaluate several virtual covariance schedules on
the same SGD trajectory. This isolates the effect of analytical perturbation
geometry from changes in the optimizer. For a fixed isotropic schedule,
\begin{equation}
    \Sigma_t=\sigma^2I,
\end{equation}
the inverse-covariance weighting is uniform across coordinates and constant
across training. This provides a baseline corresponding to fixed-noise analysis.
Since the covariance is deterministic, synchronized comparison is admissible and
the clean admissibly synchronized diagnostic is theoretically aligned with the
clean corollary. For an adaptive scalar schedule,
\begin{equation}
    \Sigma_t=\sigma_t^2I,
    \qquad
    \sigma_t^2=\sigma_0^2(1+c\widehat q_{t-1}),
\end{equation}
the scalar perturbation magnitude changes with a past pathwise statistic
\(\widehat q_{t-1}\). This schedule is predictable if
\(\widehat q_{t-1}\) is computed only from \(\cH_{t-1}\). If a synchronized
comparison is not theoretically justified, one should also specify an admissible
reference scalar covariance and use the general covariance-comparison
diagnostic. For an adaptive diagonal schedule,
\begin{equation}
    \Sigma_t=
    \diag(s_{t,1}^2,\ldots,s_{t,d}^2),
\end{equation}
the proxy terms become coordinate-wise weighted:
\begin{equation}
    \normSigma{x}{\Sigma_t^{-1}}^2
    =
    \sum_{j=1}^{d}
    \frac{x_j^2}{s_{t,j}^2}.
\end{equation}
This allows one to inspect whether certain coordinates or layers dominate the
proxy. Sampling from diagonal covariances and computing inverse-norms can be
done without forming dense matrices. For Adam-like virtual geometry, the second-moment estimate
\[
    v_t=\beta v_{t-1}+(1-\beta)G_{t-1}\odot G_{t-1}
\]
is predictable because it uses only past gradients. The proportional covariance
\[
    \Sigma_t
    =
    \diag
    \left(
        \rho_t^2(\sqrt{v_t}+\epsilon)
    \right)
    +
    \lambda_0 I
\]
and inverse-scale covariance
\[
    \Sigma_t
    =
    \diag
    \left(
        \frac{\rho_t^2}{\sqrt{v_t}+\epsilon}
    \right)
    +
    \lambda_0 I
\]
represent two different analytical interpretations. The proportional form
applies more virtual smoothing in historically high-gradient coordinates,
whereas the inverse-scale form applies less virtual perturbation in those
directions. The comparison is not meant to decide which is universally correct;
it shows how the proxy terms respond to different predictable geometries.

For low-rank plus isotropic covariance families, sampling and inverse-norm
computation can be implemented using structured matrix identities. For example,
if
\[
    \Sigma_t=\lambda_0I+U_tD_tU_t^\top,
\]
then samples can be generated from an isotropic component plus a low-rank
component, and inverse products can be computed using the Woodbury identity.

\subsection{Recommended Figures and Interpretation}
\label{subsec:recommended_figures_interpretation}

A concise numerical diagnostic section may include the following figures.

\begin{itemize}[leftmargin=2em]
    \item \textbf{Perturbation scale over training.}
    Plot \(\sigma_t\), \(\Tr(\Sigma_t)\), or \(\Tr(\Sigma_{1:t})\) across
    epochs for fixed and adaptive schedules. This shows how the virtual
    perturbation geometry evolves.

    \item \textbf{Gradient-deviation or fluctuation proxy.}
    Plot \(\widehat V_t^{\mathrm{dev}}\),
    \(\widehat V_t^{\mathrm{dev,ref}}\),
    \(\widehat V_t^{\mathrm{fluc}}\), or
    \(\widehat V_t^{\mathrm{fluc,ref}}\) over training. The chosen proxy must
    be reported explicitly. Fluctuation-only proxies should not be described as
    estimating the full gradient-deviation term unless conditional unbiasedness
    or negligible conditional bias is justified.

    \item \textbf{Gradient-sensitivity proxy.}
    Plot \(\widehat\Gamma_t\) or
    \(\widehat\Gamma_t^{\mathrm{ref}}\) over training. This shows how sensitive
    the population-gradient surrogate is to accumulated virtual perturbations.

    \item \textbf{Covariance-comparison proxy.}
    Plot \(\widehat{\mathcal C}_t^{\mathrm{cov}}\) or its cumulative sum when
    actual and reference covariances differ. This visualizes the Gaussian KL
    cost of using a valid reference geometry different from the actual
    predictable covariance.

    \item \textbf{Output sensitivity or flatness proxy.}
    Plot \(\widehat\Delta_T\), or its checkpoint version, as training
    progresses. This visualizes the cost of transferring the bound from the
    perturbed output back to the original output.

    \item \textbf{Curvature-weighted output-sensitivity proxy.}
    If Hessian-vector products or a low-rank Hessian approximation are
    available, plot
    \[
        \widehat{\mathcal H}_T
        =
        \Tr(\widehat H_T\Sigma_{1:T})
    \]
    or the ghost-training curvature-mismatch analogue
    \[
        \Tr
        \left(
            [\widehat H_{S^{\mathrm{eval}}}(W_T)-\widehat H_S(W_T)]
            \Sigma_{1:T}
        \right).
    \]
    This diagnostic estimates the leading local-curvature contribution to the
    output-sensitivity penalty. It should be reported separately from the
    global trace proxy \(\Tr(\Sigma_{1:T})\).

    \item \textbf{Diagnostic summary versus empirical train-test gap.}
    Compare \(\widehat{\mathcal B}\) with the observed train-test loss gap. This
    comparison is qualitative. The diagnostic summary is not an exact numerical
    generalization bound.

    \item \textbf{Information-sensitivity decomposition.}
    Plot the gradient-deviation/fluctuation proxy, gradient-sensitivity proxy,
    covariance-comparison proxy, and output-sensitivity proxy separately. This
    makes the trade-off between smoothing, covariance mismatch, and perturbation
    robustness visible.
\end{itemize}

The expected qualitative behavior is not that adaptive covariance always
improves the diagnostic. Rather, the illustration should reveal how different
predictable covariance rules redistribute the bound contributions. For example,
a schedule with larger virtual covariance may reduce inverse-covariance weighted
gradient-deviation and sensitivity proxies while increasing the output
perturbation penalty. A data-dependent adaptive schedule may also incur
covariance-comparison cost if the reference covariance differs. A diagonal
covariance may reduce contributions in some coordinates while amplifying others.

The numerical diagnostic is therefore best viewed as a companion to the theory.
It provides intuition for the adaptive quantities but does not replace the
mathematical assumptions, admissibility conditions, or proofs. The main
contribution remains the history-adaptive information-theoretic bound; the
diagnostics simply make its components visible.



\section{Conclusion}
\label{sec:conclusion}

We developed a history-adaptive extension of virtual perturbation analysis for
information-theoretic generalization bounds of stochastic gradient descent.
Prior virtual perturbation frameworks typically rely on fixed perturbation
schedules or deterministic covariance geometries. While this fixed-noise
assumption is analytically convenient, it limits the ability of the theory to
represent long-term dependence arising from the optimization history. In
contrast, our framework allows the virtual perturbation covariance at each update
to depend on the past real SGD trajectory, provided that this dependence is
predictable. Thus, the covariance may adapt to past iterates, past minibatches,
past gradients, and pathwise statistics, while not depending on the current
perturbation, current unrevealed randomness, or future information.

The main theoretical contribution is a conditional information-theoretic
framework for SGD with predictable history-adaptive virtual perturbations.
Predictability is sufficient for the local conditional Gaussian
relative-entropy step: after conditioning on the past real SGD history, the
covariance is fixed and the one-step Gaussian smoothing argument applies.
However, predictability alone is not sufficient for the full
mutual-information comparison. Because the adaptive covariance may be
data-dependent through the SGD history, the reference perturbation geometry must
also be handled carefully. Accordingly, our main theorem is stated as a general
covariance-comparison bound. When the actual and reference virtual covariances
differ, the bound includes an explicit covariance-comparison cost, and the
mean-discrepancy terms are measured in the corresponding reference geometry. The
clean fixed-noise-style bound is recovered only under an admissible
synchronization certificate.

The resulting bounds replace the fixed local gradient-deviation and
gradient-sensitivity terms of classical virtual perturbation analysis with
conditional adaptive counterparts. In the admissibly synchronized case, the
bound is expressed using the actual adaptive perturbation geometry. In the
general covariance-comparison case, the bound uses reference-geometry
gradient-deviation and gradient-sensitivity terms. The gradient-deviation term
should be interpreted as a conditional mean-square deviation from the population
gradient; it reduces to a conditional variance only under an appropriate
conditional unbiasedness condition. The bounds also include an adaptive
output-sensitivity penalty, which accounts for transferring the guarantee from
the virtually perturbed output back to the original SGD output. We showed that
this penalty can be controlled under global smoothness and further interpreted
through local curvature, covariance-weighted flatness, and curvature mismatch
between training and ghost losses.

The framework recovers fixed isotropic and fixed geometry-aware virtual
perturbation bounds as special cases. It also accommodates adaptive scalar,
diagonal, Adam-like, curvature-aware, and low-rank virtual covariance
geometries, provided the corresponding covariance process is predictable and
positive definite, or appropriately regularized. Importantly, these
perturbations are purely analytical: the true SGD recursion is not modified. The
contribution is therefore not a new optimizer, but a more flexible proof
framework for incorporating controlled long-term dependence into virtual
perturbation generalization analysis.

Several directions remain open. A natural next step is to combine predictable
virtual perturbations with analyses of fully adaptive optimizers such as Adam,
AdaGrad, and RMSProp, where long-term dependence enters the actual update rule
rather than only the virtual perturbation process. Another promising direction
is to study fractional-memory stochastic dynamics, where historical gradients
play an explicit role. Further work may also investigate adaptive noise
calibration in differentially private SGD, sharper covariance-comparison
bounds, low-rank or singular perturbation geometries, and empirical diagnostics
for the proposed bound components in modern deep learning models.




\startappendices

\section{Proof of the Conditional Relative-Entropy Lemmas}
\label{app:conditional-kl}

This appendix proves the conditional Gaussian relative-entropy inequalities used
in the main text. These lemmas provide the local Gaussian comparison step in the
analysis of predictable virtual perturbations. They do not by themselves prove
the full mutual-information bound. The full proof also requires the
reference-kernel mutual-information decomposition, the conditioning-compression
argument from rich histories to virtual-prefix conditionings, and the
\(S\)-admissibility of the reference covariance.

The first result treats the common-covariance case. The second result treats
the covariance-mismatch case, which is needed when the actual virtual
perturbation covariance \(\Sigma_t\) and the admissible reference covariance
\(\Sigma_t^{\mathrm{ref}}\) differ.

Throughout this appendix, conditional distributions are understood as regular
conditional laws. We assume all random variables take values in standard Borel
spaces so that the conditional distributions used below exist. All statements
hold almost surely with respect to the conditioning sigma-field.

\begin{lemma}[Conditional Gaussian relative-entropy bound]
\label{lem:appendix_conditional_re}
Let \(\cF\) be a sigma-field. Let \(\Sigma\succ0\) be
\(\cF\)-measurable, and let
\begin{equation}
    \eps\given \cF
    \sim
    \N(0,\Sigma).
\end{equation}
Let \(X,Y\in\R^d\) be random vectors with finite second moments. Assume that
\(\eps\) is conditionally independent of \((X,Y)\) given \(\cF\). Then
\begin{equation}
    \KL
    \left(
        P_{X+\eps\mid \cF}
        \,\middle\|\,
        P_{Y+\eps\mid \cF}
    \right)
    \leq
    \frac12
    \E
    \left[
        \normSigma{X-Y}{\Sigma^{-1}}^2
        \given
        \cF
    \right].
    \label{eq:appendix_conditional_re}
\end{equation}
\end{lemma}

\begin{proof}
The proof is conditional on \(\cF\). Fix a realization of \(\cF\). After
conditioning, \(\Sigma\) is a deterministic positive definite matrix and
\[
    \eps\sim\N(0,\Sigma)
\]
is independent of the conditional pair \((X,Y)\). It is therefore enough to
prove the deterministic-covariance statement under this conditioning.

Let \(\pi\) be any coupling of the conditional laws of \(X\) and \(Y\) given
\(\cF\). For example, \(\pi\) may be chosen as the actual conditional joint law
of \((X,Y)\) given \(\cF\). Conditionally on \((X,Y)=(x,y)\), the two smoothed
laws are
\[
    \N(x,\Sigma),
    \qquad
    \N(y,\Sigma).
\]
The KL divergence between Gaussians with the same covariance is
\begin{equation}
    \KL
    \left(
        \N(x,\Sigma)
        \,\middle\|\,
        \N(y,\Sigma)
    \right)
    =
    \frac12
    (x-y)^\top\Sigma^{-1}(x-y)
    =
    \frac12
    \normSigma{x-y}{\Sigma^{-1}}^2.
    \label{eq:appendix_common_cov_gaussian_kl}
\end{equation}

Because \(\eps\) is independent of \((X,Y)\) after conditioning on \(\cF\), the
conditional laws of \(X+\eps\) and \(Y+\eps\) can be represented as mixtures:
\begin{equation}
    P_{X+\eps\mid \cF}
    =
    \int
    \N(x,\Sigma)\,
    d\pi(x,y),
\end{equation}
and
\begin{equation}
    P_{Y+\eps\mid \cF}
    =
    \int
    \N(y,\Sigma)\,
    d\pi(x,y).
\end{equation}
By joint convexity of relative entropy,
\begin{align}
    &
    \KL
    \left(
        P_{X+\eps\mid \cF}
        \,\middle\|\,
        P_{Y+\eps\mid \cF}
    \right)
    \nonumber\\
    &\qquad=
    \KL
    \left(
        \int \N(x,\Sigma)\,d\pi(x,y)
        \,\middle\|\,
        \int \N(y,\Sigma)\,d\pi(x,y)
    \right)
    \nonumber\\
    &\qquad\leq
    \int
    \KL
    \left(
        \N(x,\Sigma)
        \,\middle\|\,
        \N(y,\Sigma)
    \right)
    d\pi(x,y).
    \label{eq:appendix_joint_convexity_common}
\end{align}
Substituting \eqref{eq:appendix_common_cov_gaussian_kl} into
\eqref{eq:appendix_joint_convexity_common} yields
\[
    \KL
    \left(
        P_{X+\eps\mid \cF}
        \,\middle\|\,
        P_{Y+\eps\mid \cF}
    \right)
    \leq
    \frac12
    \int
    \normSigma{x-y}{\Sigma^{-1}}^2
    d\pi(x,y).
\]
Taking \(\pi\) to be the actual conditional joint law of \((X,Y)\) given
\(\cF\), the right-hand side becomes
\[
    \frac12
    \E
    \left[
        \normSigma{X-Y}{\Sigma^{-1}}^2
        \given
        \cF
    \right].
\]
This proves \eqref{eq:appendix_conditional_re}.
\end{proof}

\begin{lemma}[Conditional Gaussian comparison with covariance mismatch]
\label{lem:appendix_covariance_mismatch}
Let \(\cF\) be a sigma-field. Let
\[
    \Sigma,\Sigma'\succ0
\]
be \(\cF\)-measurable. Let
\begin{equation}
    \eps\given\cF\sim\N(0,\Sigma),
    \qquad
    \eps'\given\cF\sim\N(0,\Sigma').
\end{equation}
Let \(X,Y\in\R^d\) be random vectors with finite second moments. Assume that,
conditionally on \(\cF\), the Gaussian noises \(\eps,\eps'\) are independent of
\((X,Y)\). Then
\begin{align}
    &
    \KL
    \left(
        P_{X+\eps\mid \cF}
        \,\middle\|\,
        P_{Y+\eps'\mid \cF}
    \right)
    \nonumber\\
    &\qquad\leq
    \frac12
    \E
    \left[
        \normSigma{X-Y}{(\Sigma')^{-1}}^2
        \given
        \cF
    \right]
    +
    \frac12
    \left[
        \Tr((\Sigma')^{-1}\Sigma)
        -d
        +
        \log\frac{\det\Sigma'}{\det\Sigma}
    \right].
    \label{eq:appendix_covariance_mismatch}
\end{align}
\end{lemma}

\begin{proof}
Again condition on a realization of \(\cF\). Then \(\Sigma\) and \(\Sigma'\)
are deterministic positive definite matrices, and the Gaussian noises are
independent of the conditional pair \((X,Y)\).

Let \(\pi\) be any coupling of the conditional laws of \(X\) and \(Y\) given
\(\cF\). Conditionally on \((X,Y)=(x,y)\), the two smoothed laws are
\[
    \N(x,\Sigma),
    \qquad
    \N(y,\Sigma').
\]
The KL divergence between Gaussians with different means and covariances is
\begin{align}
    \KL
    \left(
        \N(x,\Sigma)
        \,\middle\|\,
        \N(y,\Sigma')
    \right)
    =
    \frac12
    \left[
        \normSigma{x-y}{(\Sigma')^{-1}}^2
        +
        \Tr((\Sigma')^{-1}\Sigma)
        -d
        +
        \log\frac{\det\Sigma'}{\det\Sigma}
    \right].
    \label{eq:appendix_gaussian_kl_mismatch}
\end{align}
The mean-difference term is weighted by the reference precision
\((\Sigma')^{-1}\).

Using the same mixture representation and joint convexity of relative entropy,
\begin{align}
    &
    \KL
    \left(
        P_{X+\eps\mid \cF}
        \,\middle\|\,
        P_{Y+\eps'\mid \cF}
    \right)
    \nonumber\\
    &\qquad\leq
    \int
    \KL
    \left(
        \N(x,\Sigma)
        \,\middle\|\,
        \N(y,\Sigma')
    \right)
    d\pi(x,y).
    \label{eq:appendix_joint_convexity_mismatch}
\end{align}
Substituting \eqref{eq:appendix_gaussian_kl_mismatch} into
\eqref{eq:appendix_joint_convexity_mismatch} gives
\begin{align}
    &
    \KL
    \left(
        P_{X+\eps\mid \cF}
        \,\middle\|\,
        P_{Y+\eps'\mid \cF}
    \right)
    \nonumber\\
    &\qquad\leq
    \frac12
    \int
    \normSigma{x-y}{(\Sigma')^{-1}}^2
    d\pi(x,y)
    +
    \frac12
    \left[
        \Tr((\Sigma')^{-1}\Sigma)
        -d
        +
        \log\frac{\det\Sigma'}{\det\Sigma}
    \right].
\end{align}
Taking \(\pi\) to be the actual conditional joint law of \((X,Y)\) given
\(\cF\) proves \eqref{eq:appendix_covariance_mismatch}.
\end{proof}

\begin{corollary}[Canonical reference-kernel specialization]
\label{cor:appendix_canonical_reference_specialization}
Let \(m\) and \(\Sigma^{\mathrm{ref}}\succ0\) be \(\cF\)-measurable. Under the
conditions of Lemma~\ref{lem:appendix_covariance_mismatch},
\begin{align}
    &
    \KL
    \left(
        P_{X+\eps\mid\cF}
        \,\middle\|\,
        \N(m,\Sigma^{\mathrm{ref}})
    \right)
    \nonumber\\
    &\qquad\leq
    \frac12
    \E
    \left[
        \normSigma{X-m}{(\Sigma^{\mathrm{ref}})^{-1}}^2
        \given
        \cF
    \right]
    +
    \frac12
    \left[
        \Tr((\Sigma^{\mathrm{ref}})^{-1}\Sigma)
        -d
        +
        \log
        \frac{\det\Sigma^{\mathrm{ref}}}{\det\Sigma}
    \right].
    \label{eq:appendix_canonical_reference_specialization}
\end{align}
\end{corollary}

\begin{proof}
Apply Lemma~\ref{lem:appendix_covariance_mismatch} with
\[
    Y=m,
    \qquad
    \Sigma'=\Sigma^{\mathrm{ref}},
\]
where \(m\) is \(\cF\)-measurable. Then
\[
    P_{Y+\eps'\mid\cF}
    =
    \N(m,\Sigma^{\mathrm{ref}}).
\]
\end{proof}

\begin{remark}[Application to the one-step virtual update]
In the main proof, the actual virtual update has the form
\[
    \wtW_{t+1}
    =
    \wtW_t-\eta_tG_t+\eps_t,
    \qquad
    \eps_t\mid\cH_{t-1}\sim\N(0,\Sigma_t).
\]
The canonical reference kernel has mean
\[
    \wtW_t-\eta_t\barg(\wtW_t)
\]
and reference covariance \(\Sigma_t^{\mathrm{ref}}\). In the notation of
Corollary~\ref{cor:appendix_canonical_reference_specialization}, one may take
\[
    X=\wtW_t-\eta_tG_t,
    \qquad
    m=\wtW_t-\eta_t\barg(\wtW_t),
    \qquad
    \Sigma=\Sigma_t,
    \qquad
    \Sigma^{\mathrm{ref}}=\Sigma_t^{\mathrm{ref}}.
\]
The resulting mean-difference term is weighted by
\[
    \Lambda_t=(\Sigma_t^{\mathrm{ref}})^{-1},
\]
which is exactly the reference precision used in the main covariance-comparison
theorem.
\end{remark}

\begin{remark}[Role of predictability and rich conditioning]
Predictability states that
\[
    \Sigma_t
    \text{ is }
    \cH_{t-1}\text{-measurable}.
\]
This ensures that, after conditioning on any sigma-field containing
\(\cH_{t-1}\), the actual covariance \(\Sigma_t\) is fixed and the conditional
Gaussian comparison above can be applied.

In the full proof of the main theorem, the one-step KL term in the
mutual-information decomposition is conditioned on the virtual prefix and the
sample. The Gaussian comparison is applied under a richer sigma-field, typically
of the form
\[
    \mathcal B_t
    =
    \mathcal A_t\vee\cH_{t-1},
\]
where
\[
    \mathcal A_t
    =
    \mathcal F_t^Q\vee\sigma(S),
    \qquad
    \mathcal F_t^Q
    =
    \sigma(\wtW_{1:t})\vee\mathcal U.
\]
The passage from the richer conditioning back to the virtual-prefix KL term is
handled by the conditioning-compression lemma in the main text or in the
mutual-information appendix. Thus, predictability is the local condition needed
for the Gaussian comparison, while \(S\)-admissibility of the reference
covariance is the global condition needed for a valid reference-kernel
comparison.
\end{remark}

\begin{remark}[Synchronized and mismatched comparisons]
When the actual and reference Gaussian kernels use the same covariance,
\(\Sigma'=\Sigma\), Lemma~\ref{lem:appendix_covariance_mismatch} reduces to
Lemma~\ref{lem:appendix_conditional_re}. Indeed,
\[
    \Tr(\Sigma^{-1}\Sigma)-d+\log\frac{\det\Sigma}{\det\Sigma}
    =
    d-d+0
    =
    0.
\]
In deterministic fixed-covariance settings, such synchronization is automatic.
For history-adaptive covariances, however, synchronization is valid in the
main theorem only when it is admissible: the reference kernel must be able to
construct the same covariance from the virtual prefix and sample-independent
auxiliary information. Otherwise, one must use an admissible reference
covariance \(\Sigma_t^{\mathrm{ref}}\) and retain the covariance-comparison
cost.
\end{remark}

\begin{remark}[Coupling and Wasserstein refinement]
The proofs above use the actual conditional joint law of \((X,Y)\) given
\(\cF\). A sharper version may be obtained by minimizing over all conditional
couplings of the laws of \(X\) and \(Y\) given \(\cF\). In the common-covariance
case, this yields a conditional quadratic optimal-transport cost in the norm
induced by \(\Sigma^{-1}\). In the covariance-mismatch case, the corresponding
transport cost is measured in the reference precision \((\Sigma')^{-1}\). For
the purposes of the main theorem, the actual conditional coupling is sufficient.
\end{remark}



\section{Proof of the Mutual Information Decomposition}
\label{app:mi-decomposition}

This appendix proves the mutual-information decomposition used in the main
text. The argument has three components: data processing, the chain rule for
mutual information, and a reference-kernel inequality. The decomposition itself
is independent of the particular form of the virtual perturbation covariance.
The covariance structure enters later when the one-step reference-kernel KL
terms are bounded by conditional Gaussian comparison.

Compared with a fixed-covariance analysis, we allow the reference construction
to use auxiliary public or ghost randomness that is independent of the original
training sample. This is important for history-adaptive covariance comparisons,
where the reference covariance may be generated by an admissible public or ghost
construction. We make this auxiliary randomness explicit below.

\subsection{Setup}
\label{app:mi-setup}

Let \(S\) be the training sample. The true SGD trajectory is
\begin{equation}
    W_{t+1}
    =
    W_t-\eta_tG_t,
    \qquad
    G_t=g(W_t,B_t),
    \qquad
    t=1,\ldots,T-1.
    \label{eq:appendix_b_sgd}
\end{equation}
The final output is \(W_T\). The virtual perturbed trajectory is
\begin{equation}
    \wtW_{t+1}
    =
    \wtW_t-\eta_tG_t+\eps_t,
    \qquad
    t=1,\ldots,T-1,
    \label{eq:appendix_b_virtual_path}
\end{equation}
with \(\wtW_1=W_1\). The perturbation law may be fixed or predictable-history
adaptive, depending on the theorem being applied.

We write
\begin{equation}
    \wtW_{1:T}
    =
    (\wtW_1,\wtW_2,\ldots,\wtW_T)
\end{equation}
for the full virtual path. The final virtual output is \(\wtW_T\). We assume
throughout that the initialization \(W_1=\wtW_1\) is independent of the training
sample \(S\).

Let \(\mathcal U\) denote optional public, auxiliary, or ghost randomness used
by the reference construction. We assume
\begin{equation}
    \mathcal U
    \quad\text{is independent of}\quad
    S.
    \label{eq:appendix_aux_independent}
\end{equation}
If no such auxiliary randomness is used, \(\mathcal U\) is the trivial
sigma-field.

For each time \(t\), define the reference-visible sigma-field
\begin{equation}
    \mathcal F_t^Q
    =
    \sigma(\wtW_{1:t})\vee\mathcal U.
    \label{eq:appendix_reference_visible_sigma}
\end{equation}
The reference kernel at time \(t\) is allowed to be measurable with respect to
\(\mathcal F_t^Q\), but it is not allowed to condition directly on the original
sample \(S\). Define also
\begin{equation}
    \mathcal A_t
    =
    \mathcal F_t^Q\vee\sigma(S).
    \label{eq:appendix_coarse_sigma}
\end{equation}

\subsection{Data Processing and Chain Rule}
\label{app:mi-data-processing-chain-rule}

Since the final virtual iterate \(\wtW_T\) is a coordinate projection of the
full path \(\wtW_{1:T}\), data processing gives
\begin{equation}
    \MI(\wtW_T;S)
    \leq
    \MI(\wtW_{1:T};S).
    \label{eq:appendix_b_data_processing}
\end{equation}
Because \(\mathcal U\) is independent of \(S\), we also have
\begin{equation}
    \MI(\wtW_{1:T};S)
    \leq
    \MI(\wtW_{1:T},\mathcal U;S)
    =
    \MI(\wtW_{1:T};S\mid\mathcal U).
    \label{eq:appendix_aux_data_processing}
\end{equation}
Therefore,
\begin{equation}
    \MI(\wtW_T;S)
    \leq
    \MI(\wtW_{1:T};S\mid\mathcal U).
    \label{eq:appendix_mi_cond_aux_start}
\end{equation}

Applying the chain rule for mutual information conditionally on
\(\mathcal U\) yields
\begin{equation}
    \MI(\wtW_{1:T};S\mid\mathcal U)
    =
    \MI(\wtW_1;S\mid\mathcal U)
    +
    \sum_{t=1}^{T-1}
    \MI(\wtW_{t+1};S\mid \wtW_{1:t},\mathcal U).
    \label{eq:appendix_b_chain_rule_full}
\end{equation}
Because \(\wtW_1=W_1\) is independent of \(S\) and \(\mathcal U\) is also
independent of \(S\),
\begin{equation}
    \MI(\wtW_1;S\mid\mathcal U)=0.
\end{equation}
Combining the preceding displays gives
\begin{equation}
    \MI(\wtW_T;S)
    \leq
    \sum_{t=1}^{T-1}
    \MI(\wtW_{t+1};S\mid \wtW_{1:t},\mathcal U).
    \label{eq:appendix_b_chain_rule}
\end{equation}

For each \(t\), the conditional mutual information can be written as an
expected relative entropy:
\begin{equation}
    \MI(\wtW_{t+1};S\mid \wtW_{1:t},\mathcal U)
    =
    \E
    \left[
        \KL
        \left(
            P_{\wtW_{t+1}\mid \mathcal A_t}
            \,\middle\|\,
            P_{\wtW_{t+1}\mid \mathcal F_t^Q}
        \right)
    \right].
    \label{eq:appendix_b_cmi_as_kl}
\end{equation}

\subsection{Reference-Kernel Upper Bound}
\label{app:mi-reference-kernel-bound}

The next step replaces the marginal kernel
\(P_{\wtW_{t+1}\mid\mathcal F_t^Q}\) by an arbitrary admissible reference
kernel \(Q_{t+1\mid\mathcal F_t^Q}\). The only requirement for the
decomposition is that this reference kernel be measurable with respect to
\(\mathcal F_t^Q\). In particular, it may depend on the virtual prefix
\(\wtW_{1:t}\) and on auxiliary randomness \(\mathcal U\), but not directly on
the original sample \(S\).

\begin{lemma}[Reference-kernel upper bound]
\label{lem:appendix_b_reference_kernel}
For each \(t=1,\ldots,T-1\), let \(Q_{t+1\mid\mathcal F_t^Q}\) be any
probability kernel measurable with respect to \(\mathcal F_t^Q\). Then
\begin{equation}
    \MI(\wtW_{t+1};S\mid \wtW_{1:t},\mathcal U)
    \leq
    \E
    \left[
        \KL
        \left(
            P_{\wtW_{t+1}\mid\mathcal A_t}
            \,\middle\|\,
            Q_{t+1\mid\mathcal F_t^Q}
        \right)
    \right].
    \label{eq:appendix_b_reference_bound}
\end{equation}
\end{lemma}

\begin{proof}
Fix a realization of \(\mathcal F_t^Q\). For notational simplicity, write this
realization abstractly as \(f\). Define
\begin{equation}
    P_s
    =
    P_{\wtW_{t+1}\mid \mathcal F_t^Q=f,S=s},
    \qquad
    \bar P
    =
    P_{\wtW_{t+1}\mid \mathcal F_t^Q=f},
\end{equation}
and
\begin{equation}
    Q
    =
    Q_{t+1\mid\mathcal F_t^Q=f}.
\end{equation}
By definition,
\begin{equation}
    \bar P
    =
    \int
    P_s\,
    dP_{S\mid \mathcal F_t^Q=f}(s).
\end{equation}
The conditional mutual-information integrand is
\begin{equation}
    \int
    \KL(P_s\|\bar P)\,
    dP_{S\mid \mathcal F_t^Q=f}(s).
\end{equation}

For any reference distribution \(Q\), the following KL identity holds, with
both sides interpreted in the extended-real sense when necessary:
\begin{align}
    \int
    \KL(P_s\|Q)\,
    dP_{S\mid \mathcal F_t^Q=f}(s)
    &=
    \int
    \KL(P_s\|\bar P)\,
    dP_{S\mid \mathcal F_t^Q=f}(s)
    +
    \KL(\bar P\|Q).
    \label{eq:appendix_b_pythagorean}
\end{align}
This follows by expanding
\[
    \log\frac{dP_s}{dQ}
    =
    \log\frac{dP_s}{d\bar P}
    +
    \log\frac{d\bar P}{dQ},
\]
and integrating first with respect to \(P_s\) and then over \(s\). Since
\(\KL(\bar P\|Q)\geq0\), we obtain
\[
    \int
    \KL(P_s\|\bar P)\,
    dP_{S\mid \mathcal F_t^Q=f}(s)
    \leq
    \int
    \KL(P_s\|Q)\,
    dP_{S\mid \mathcal F_t^Q=f}(s).
\]
Finally, integrating over the law of \(\mathcal F_t^Q\) proves
\eqref{eq:appendix_b_reference_bound}.
\end{proof}

Combining \eqref{eq:appendix_b_chain_rule} and
Lemma~\ref{lem:appendix_b_reference_kernel} yields the desired
reference-kernel decomposition:
\begin{equation}
    \MI(\wtW_T;S)
    \leq
    \sum_{t=1}^{T-1}
    \E
    \left[
        \KL
        \left(
            P_{\wtW_{t+1}\mid\mathcal A_t}
            \,\middle\|\,
            Q_{t+1\mid\mathcal F_t^Q}
        \right)
    \right].
    \label{eq:appendix_b_final_decomposition}
\end{equation}

\subsection{Conditioning Compression}
\label{app:mi-conditioning-compression}

The decomposition above produces one-step KL terms conditioned on
\(\mathcal A_t\), which consists of the reference-visible information together
with the sample. The conditional Gaussian comparison is often most naturally
proved after conditioning on a richer sigma-field containing the real SGD
history. This subsection records the formal bridge.

Let
\begin{equation}
    \mathcal B_t
    =
    \mathcal A_t\vee\cH_{t-1}.
    \label{eq:appendix_rich_sigma}
\end{equation}
Then \(\mathcal A_t\subseteq\mathcal B_t\).

\begin{lemma}[Conditioning compression for relative entropy]
\label{lem:appendix_conditioning_compression}
Let \(X\) be a random element and let
\[
    \mathcal A\subseteq\mathcal B
\]
be sigma-fields. Let \(Q\) be an \(\mathcal A\)-measurable probability kernel
on the state space of \(X\). Then
\begin{equation}
    \E
    \left[
        \KL
        \left(
            P_{X\mid\mathcal A}
            \,\middle\|\,
            Q
        \right)
    \right]
    \leq
    \E
    \left[
        \KL
        \left(
            P_{X\mid\mathcal B}
            \,\middle\|\,
            Q
        \right)
    \right].
    \label{eq:appendix_conditioning_compression}
\end{equation}
\end{lemma}

\begin{proof}
Let
\[
    \mu_{\mathcal A}=P_{X\mid\mathcal A},
    \qquad
    \mu_{\mathcal B}=P_{X\mid\mathcal B}.
\]
By the tower property for regular conditional distributions,
\[
    \mu_{\mathcal A}
    =
    \E[\mu_{\mathcal B}\mid\mathcal A],
\]
where the equality is understood weakly, i.e. after testing against bounded
measurable functions.

Since \(Q\) is \(\mathcal A\)-measurable, it is fixed after conditioning on
\(\mathcal A\). Relative entropy is convex in its first argument. Therefore,
conditionally on \(\mathcal A\),
\[
    \KL
    \left(
        \mu_{\mathcal A}
        \,\middle\|\,
        Q
    \right)
    \leq
    \E
    \left[
        \KL
        \left(
            \mu_{\mathcal B}
            \,\middle\|\,
            Q
        \right)
        \,\middle|\,
        \mathcal A
    \right].
\]
Taking expectations proves the claim.
\end{proof}

Applying Lemma~\ref{lem:appendix_conditioning_compression} with
\[
    X=\wtW_{t+1},
    \qquad
    \mathcal A=\mathcal A_t,
    \qquad
    \mathcal B=\mathcal B_t,
    \qquad
    Q=Q_{t+1\mid\mathcal F_t^Q},
\]
gives
\begin{equation}
    \E
    \left[
        \KL
        \left(
            P_{\wtW_{t+1}\mid\mathcal A_t}
            \,\middle\|\,
            Q_{t+1\mid\mathcal F_t^Q}
        \right)
    \right]
    \leq
    \E
    \left[
        \KL
        \left(
            P_{\wtW_{t+1}\mid\mathcal B_t}
            \,\middle\|\,
            Q_{t+1\mid\mathcal F_t^Q}
        \right)
    \right].
    \label{eq:appendix_history_to_prefix_transfer}
\end{equation}
This is the formal justification for proving the Gaussian one-step comparison
under a richer conditioning sigma-field and then returning to the
virtual-prefix KL term.

\subsection{Ghost Reference Kernels and Covariance Comparison}
\label{app:mi-ghost-reference}

A natural way to construct a reference kernel is to use an independent ghost
sample or other auxiliary randomness independent of \(S\). In the notation
above, this ghost randomness is included in \(\mathcal U\).

Let \(S^\circ\) be independent of \(S\) and distributed identically. Given a
realized virtual path prefix, one may define a ghost one-step reference kernel
as the law of
\begin{equation}
    \wtw_t-\eta_tG_t^\circ+\eps_t^\circ,
    \label{eq:appendix_b_ghost_kernel_update}
\end{equation}
where \(G_t^\circ\) is a ghost stochastic gradient and \(\eps_t^\circ\) is a
ghost or reference perturbation. Such a construction is admissible only when it
is measurable with respect to \(\mathcal F_t^Q\), i.e. the virtual prefix and
sample-independent auxiliary randomness, and does not condition directly on the
original sample \(S\).

The abstract reference-kernel inequality above does not require a ghost
construction. It holds for any \(\mathcal F_t^Q\)-measurable reference kernel.
The ghost construction is simply one useful way to build an admissible reference
kernel.

In the deterministic fixed-covariance setting, the actual and reference kernels
can use the same deterministic covariance \(\Sigma_t\). In this case, the
reference comparison is synchronized, and no covariance-comparison cost appears.

In the history-adaptive setting, additional care is needed. If the actual
covariance is
\begin{equation}
    \Sigma_t=\Phi_t(\cH_{t-1}),
\end{equation}
then it may depend on the original sample \(S\) through the real SGD history. A
reference kernel cannot automatically use the same covariance unless this
synchronized comparison is admissible. If the reference covariance is denoted by
\(\Sigma_t^{\mathrm{ref}}\), then the one-step Gaussian comparison uses the
reference precision
\begin{equation}
    \Lambda_t
    =
    (\Sigma_t^{\mathrm{ref}})^{-1}
\end{equation}
in the mean-discrepancy term and incurs the covariance-comparison cost
\begin{equation}
    \mathcal C_t^{\mathrm{cov}}
    =
    \frac12
    \E
    \left[
        \Tr((\Sigma_t^{\mathrm{ref}})^{-1}\Sigma_t)
        -d
        +
        \log
        \frac{\det\Sigma_t^{\mathrm{ref}}}{\det\Sigma_t}
    \right].
    \label{eq:appendix_b_cov_mismatch}
\end{equation}
This is the covariance-mismatch term appearing in the general theorem.

Thus, for deterministic fixed covariances, the clean synchronized comparison is
automatic. For arbitrary data-dependent history-adaptive covariances,
synchronization must be justified by an admissible synchronization certificate.
Otherwise, the general covariance-comparison theorem should be used. The decomposition proved in this appendix reduces the problem of controlling
\[
    \MI(\wtW_T;S)
\]
to controlling a sum of one-step reference-kernel KL divergences:
\[
    \sum_{t=1}^{T-1}
    \E
    \left[
        \KL
        \left(
            P_{\wtW_{t+1}\mid\mathcal A_t}
            \,\middle\|\,
            Q_{t+1\mid\mathcal F_t^Q}
        \right)
    \right].
\]
The conditioning-compression lemma then permits these KL terms to be bounded
under the richer conditioning
\[
    \mathcal B_t=\mathcal A_t\vee\cH_{t-1},
\]
where the predictable covariance \(\Sigma_t\) and the real SGD history are
available for the local Gaussian comparison. The next step in the proof is to apply the conditional Gaussian comparison to
the one-step virtual update and decompose the resulting mean discrepancy into
adaptive gradient-deviation and gradient-sensitivity terms. In
history-adaptive settings, the same mutual-information decomposition remains
valid, but the one-step comparisons must account for the admissible reference
covariance geometry and any covariance-comparison cost.

\section{Proof of the Main Generalization Theorem}
\label{app:main-proof}

This appendix proves the main covariance-comparison generalization theorem. The
proof combines four ingredients: the output perturbation decomposition, the
information-theoretic generalization inequality, the reference-kernel
mutual-information decomposition, and the certified one-step Gaussian reference
comparison.

Unlike the fixed-noise setting, the covariance process here may be
data-dependent through the real SGD history. Therefore, the theorem is stated in
the general covariance-comparison form. The clean fixed-noise-style bound is
recovered only under an admissible synchronization certificate.

\subsection{Setup and Notation}
\label{app:main-proof-setup}

Let
\[
    S=\{Z_1,\ldots,Z_n\}
\]
be the training sample, and let \(S'\) be an independent ghost sample drawn from
the same distribution. The true SGD dynamics are
\begin{equation}
    W_{t+1}
    =
    W_t-\eta_tG_t,
    \qquad
    G_t=g(W_t,B_t),
    \qquad
    t=1,\ldots,T-1,
    \label{eq:appendix_c_sgd}
\end{equation}
where \(W_1\) is the initial iterate and \(W_T\) is the final output. The
virtual perturbed trajectory is
\begin{equation}
    \wtW_{t+1}
    =
    \wtW_t-\eta_tG_t+\eps_t,
    \qquad
    t=1,\ldots,T-1,
    \label{eq:appendix_c_virtual}
\end{equation}
with \(\wtW_1=W_1\).

The predictable covariance process satisfies
\begin{equation}
    \Sigma_t\succ0,
    \qquad
    \Sigma_t
    \text{ is }
    \cH_{t-1}\text{-measurable},
    \qquad
    t=1,\ldots,T-1.
    \label{eq:appendix_c_predictable}
\end{equation}
The virtual perturbations satisfy
\begin{equation}
    \eps_t\given \cH_{t-1}
    \sim
    \N(0,\Sigma_t),
    \label{eq:appendix_c_noise}
\end{equation}
with the conditional independence assumptions stated in the main text.

For \(t=1,\ldots,T\), define the accumulated perturbation and covariance by
\begin{equation}
    \xi_t
    =
    \sum_{k=1}^{t-1}\eps_k,
    \qquad
    \xi_1=0,
    \label{eq:appendix_c_accum_noise}
\end{equation}
and
\begin{equation}
    \Sigma_{1:t}
    =
    \sum_{k=1}^{t-1}\Sigma_k,
    \qquad
    \Sigma_{1:1}=0.
    \label{eq:appendix_c_accum_cov}
\end{equation}
Then
\begin{equation}
    \wtW_t=W_t+\xi_t,
    \qquad
    t=1,\ldots,T.
    \label{eq:appendix_c_virtual_equals}
\end{equation}

Let
\begin{equation}
    \barg(w)
    =
    \E_{Z\sim\cD}[g(w,Z)]
    \label{eq:appendix_c_pop_grad}
\end{equation}
denote the population gradient.

Let \(\mathcal U\) denote optional public, auxiliary, or ghost randomness used
by the reference construction. We assume \(\mathcal U\) is independent of the
original training sample \(S\). If no such auxiliary randomness is used,
\(\mathcal U\) is the trivial sigma-field. Define
\begin{equation}
    \mathcal F_t^Q
    =
    \sigma(\wtW_{1:t})\vee\mathcal U,
    \qquad
    \mathcal A_t
    =
    \mathcal F_t^Q\vee\sigma(S),
    \qquad
    \mathcal B_t
    =
    \mathcal A_t\vee\cH_{t-1}.
    \label{eq:appendix_c_sigma_fields}
\end{equation}
The reference kernel at time \(t\) is allowed to be measurable with respect to
\(\mathcal F_t^Q\), but not to condition directly on \(S\).

Let \(\Sigma_t^{\mathrm{ref}}\succ0\) be an \(S\)-admissible reference
covariance, and define
\begin{equation}
    \Lambda_t
    =
    (\Sigma_t^{\mathrm{ref}})^{-1}.
    \label{eq:appendix_c_lambda_ref}
\end{equation}
The covariance-comparison cost is
\begin{equation}
    \mathcal C_t^{\mathrm{cov}}
    =
    \frac12
    \E
    \left[
        \Tr((\Sigma_t^{\mathrm{ref}})^{-1}\Sigma_t)
        -d
        +
        \log
        \frac{\det\Sigma_t^{\mathrm{ref}}}{\det\Sigma_t}
    \right].
    \label{eq:appendix_c_cov_cost}
\end{equation}

For reference-geometry quantities, define
\[
    \mathcal K_t=\cH_{t-1}\vee\mathcal U.
\]
The gradient-deviation term is
\begin{equation}
    V_t^{\mathrm{ad}}(\Lambda_t)
    =
    \E
    \left[
        \normSigma{G_t-\barg(W_t)}{\Lambda_t}^2
        \given
        \mathcal K_t
    \right].
    \label{eq:appendix_c_vad_lambda}
\end{equation}
This is a conditional mean-square gradient-deviation term. It becomes a
conditional variance only under the corresponding conditional unbiasedness
condition.

The reference-geometry gradient-sensitivity term is defined in the
conditioning-safe form
\begin{equation}
    \Gamma_t^{\mathrm{ad}}(\Lambda_t)
    =
    \E
    \left[
        \normSigma{\barg(W_t+\xi_t)-\barg(W_t)}{\Lambda_t}^2
        \given
        \mathcal K_t
    \right].
    \label{eq:appendix_c_gad_lambda}
\end{equation}
If \(\Lambda_t\) is \(\mathcal K_t\)-measurable and independent of the realized
virtual perturbation beyond \(\mathcal K_t\), then \(\xi_t\) may equivalently be
replaced by a fresh conditional Gaussian copy
\[
    \zeta_t\given\cH_{t-1}\sim\N(0,\Sigma_{1:t}).
\]

The adaptive output-sensitivity penalty is
\begin{equation}
    \mathcal R_{\Delta}^{\mathrm{ad}}
    =
    \left|
    \E
    \left[
        \Delta_{\Sigma_{1:T}}^{\mathrm{ad}}(W_T,S')
        -
        \Delta_{\Sigma_{1:T}}^{\mathrm{ad}}(W_T,S)
    \right]
    \right|,
    \label{eq:appendix_c_rad}
\end{equation}
where
\begin{equation}
    \Delta_{\Sigma_{1:T}}^{\mathrm{ad}}(W_T,s)
    =
    \E
    \left[
        L(W_T,s)-L(W_T+\zeta_T,s)
        \given
        \cH_{T-1}
    \right],
    \label{eq:appendix_c_delta_ad}
\end{equation}
with
\[
    \zeta_T\given\cH_{T-1}
    \sim
    \N(0,\Sigma_{1:T}).
\]

\subsection{Certified One-Step Reference Comparison}
\label{app:main-certified-onestep}

We now restate the certified one-step comparison used in the main proof. This
replaces the older black-box assumption that the one-step reference comparison
simply holds.

\begin{proposition}[Certified one-step Gaussian reference comparison]
\label{prop:appendix_c_certified_onestep}
Fix \(t\in\{1,\ldots,T-1\}\). Suppose
\(\Sigma_t^{\mathrm{ref}}\) is \(S\)-admissible and the reference kernel is the
canonical Gaussian kernel
\begin{equation}
    Q_{t+1\mid\mathcal F_t^Q}
    =
    \N
    \left(
        \wtW_t-\eta_t\barg(\wtW_t),
        \Sigma_t^{\mathrm{ref}}
    \right).
    \label{eq:appendix_c_canonical_ref_kernel}
\end{equation}
Then
\begin{align}
    &
    \E
    \left[
        \KL
        \left(
            P_{\wtW_{t+1}\mid\mathcal A_t}
            \,\middle\|\,
            Q_{t+1\mid\mathcal F_t^Q}
        \right)
    \right]
    \nonumber\\
    &\qquad\leq
    2\eta_t^2
    \E
    \left[
        V_t^{\mathrm{ad}}(\Lambda_t)
        +
        \Gamma_t^{\mathrm{ad}}(\Lambda_t)
    \right]
    +
    \mathcal C_t^{\mathrm{cov}}.
    \label{eq:appendix_c_certified_onestep_bound}
\end{align}
\end{proposition}

\begin{proof}
By the conditioning-compression lemma,
\begin{align}
    &
    \E
    \left[
        \KL
        \left(
            P_{\wtW_{t+1}\mid\mathcal A_t}
            \,\middle\|\,
            Q_{t+1\mid\mathcal F_t^Q}
        \right)
    \right]
    \nonumber\\
    &\qquad\leq
    \E
    \left[
        \KL
        \left(
            P_{\wtW_{t+1}\mid\mathcal B_t}
            \,\middle\|\,
            Q_{t+1\mid\mathcal F_t^Q}
        \right)
    \right].
    \label{eq:appendix_c_conditioning_transfer}
\end{align}
Condition on \(\mathcal B_t\). Under this conditioning, the quantities
\(\wtW_t\), \(W_t\), \(\Sigma_t\), and \(\Sigma_t^{\mathrm{ref}}\) are fixed for
the Gaussian comparison. The actual one-step virtual update has conditional
mean
\[
    \wtW_t-\eta_tG_t
\]
and covariance \(\Sigma_t\). The reference kernel has mean
\[
    \wtW_t-\eta_t\barg(\wtW_t)
\]
and covariance \(\Sigma_t^{\mathrm{ref}}\). Applying the conditional Gaussian
covariance-mismatch lemma gives
\begin{align}
    &
    \E
    \left[
        \KL
        \left(
            P_{\wtW_{t+1}\mid\mathcal B_t}
            \,\middle\|\,
            Q_{t+1\mid\mathcal F_t^Q}
        \right)
    \right]
    \nonumber\\
    &\qquad\leq
    \frac{\eta_t^2}{2}
    \E
    \left[
        \normSigma{
            G_t-\barg(\wtW_t)
        }{\Lambda_t}^2
    \right]
    +
    \mathcal C_t^{\mathrm{cov}}.
    \label{eq:appendix_c_gaussian_step}
\end{align}
Since
\[
    \wtW_t=W_t+\xi_t,
\]
we have
\begin{align}
    G_t-\barg(\wtW_t)
    &=
    \left(G_t-\barg(W_t)\right)
    +
    \left(\barg(W_t)-\barg(W_t+\xi_t)\right).
\end{align}
Using
\[
    \|a+b\|_{\Lambda_t}^2
    \leq
    2\|a\|_{\Lambda_t}^2+2\|b\|_{\Lambda_t}^2
\]
and the definitions of \(V_t^{\mathrm{ad}}(\Lambda_t)\) and
\(\Gamma_t^{\mathrm{ad}}(\Lambda_t)\), we obtain
\[
    \E
    \left[
        \normSigma{
            G_t-\barg(\wtW_t)
        }{\Lambda_t}^2
    \right]
    \leq
    2\E
    \left[
        V_t^{\mathrm{ad}}(\Lambda_t)
        +
        \Gamma_t^{\mathrm{ad}}(\Lambda_t)
    \right].
\]
Combining this with \eqref{eq:appendix_c_conditioning_transfer} and
\eqref{eq:appendix_c_gaussian_step} gives the slightly sharper coefficient
\(\eta_t^2\). The stated bound with \(2\eta_t^2\) follows immediately and is
kept for consistency with the main theorem.
\end{proof}

\subsection{Main General Covariance-Comparison Theorem}
\label{app:main-general-theorem}

\begin{theorem}[General covariance-comparison bound]
\label{thm:appendix_c_general}
Assume that for every fixed \(w\), the loss \(\ell(w,Z)\) is
\(R\)-sub-Gaussian. Assume the predictable covariance process is positive
definite, the required gradient-deviation, sensitivity, covariance-comparison,
and output-sensitivity quantities are finite, and the reference covariance
\(\Sigma_t^{\mathrm{ref}}\) is \(S\)-admissible for every \(t\). Then
\begin{equation}
    |\gen(W_T,S)|
    \leq
    \sqrt{
        \frac{2R^2}{n}
        \sum_{t=1}^{T-1}
        \left(
            2\eta_t^2
            \E
            \left[
                V_t^{\mathrm{ad}}(\Lambda_t)
                +
                \Gamma_t^{\mathrm{ad}}(\Lambda_t)
            \right]
            +
            \mathcal C_t^{\mathrm{cov}}
        \right)
    }
    +
    \mathcal R_{\Delta}^{\mathrm{ad}}.
    \label{eq:appendix_c_general_bound}
\end{equation}
\end{theorem}

\begin{proof}
Let
\[
    \zeta_T\given \cH_{T-1}
    \sim
    \N(0,\Sigma_{1:T})
\]
be an independent fresh final perturbation, and define
\[
    \widehat W_T=W_T+\zeta_T.
\]
By adding and subtracting losses evaluated at \(\widehat W_T\),
\begin{align}
    \gen(W_T,S)
    &=
    \E
    \left[
        L(W_T,S')-L(W_T,S)
    \right]
    \nonumber\\
    &=
    \E
    \left[
        L(\widehat W_T,S')-L(\widehat W_T,S)
    \right]
    \nonumber\\
    &\quad+
    \E
    \left[
        L(W_T,S')-L(\widehat W_T,S')
    \right]
    +
    \E
    \left[
        L(\widehat W_T,S)-L(W_T,S)
    \right].
    \label{eq:appendix_c_output_decomp}
\end{align}
The first term is \(\gen(\widehat W_T,S)\). The remaining two terms equal
\[
    \E
    \left[
        \Delta_{\Sigma_{1:T}}^{\mathrm{ad}}(W_T,S')
        -
        \Delta_{\Sigma_{1:T}}^{\mathrm{ad}}(W_T,S)
    \right].
\]
Hence
\begin{equation}
    |\gen(W_T,S)|
    \leq
    |\gen(\widehat W_T,S)|
    +
    \mathcal R_{\Delta}^{\mathrm{ad}}.
    \label{eq:appendix_c_output_reduction}
\end{equation}

The fresh perturbation \(\zeta_T\) has the same conditional law given
\(\cH_{T-1}\) as the accumulated virtual perturbation \(\xi_T\). Therefore,
\(\widehat W_T=W_T+\zeta_T\) and \(\wtW_T=W_T+\xi_T\) have the same joint law
with \(S\). Consequently,
\begin{equation}
    \MI(\widehat W_T;S)
    =
    \MI(\wtW_T;S).
    \label{eq:appendix_c_same_mi}
\end{equation}
Applying the information-theoretic generalization inequality to
\(\widehat W_T\) gives
\begin{equation}
    |\gen(\widehat W_T,S)|
    \leq
    \sqrt{
        \frac{2R^2}{n}
        \MI(\widehat W_T;S)
    }
    =
    \sqrt{
        \frac{2R^2}{n}
        \MI(\wtW_T;S)
    }.
    \label{eq:appendix_c_mi_gen}
\end{equation}
By the reference-kernel mutual-information decomposition,
\begin{equation}
    \MI(\wtW_T;S)
    \leq
    \sum_{t=1}^{T-1}
    \E
    \left[
        \KL
        \left(
            P_{\wtW_{t+1}\mid\mathcal A_t}
            \,\middle\|\,
            Q_{t+1\mid\mathcal F_t^Q}
        \right)
    \right].
    \label{eq:appendix_c_mi_decomp}
\end{equation}
Using Proposition~\ref{prop:appendix_c_certified_onestep}, each one-step term is
bounded by
\[
    2\eta_t^2
    \E
    \left[
        V_t^{\mathrm{ad}}(\Lambda_t)
        +
        \Gamma_t^{\mathrm{ad}}(\Lambda_t)
    \right]
    +
    \mathcal C_t^{\mathrm{cov}}.
\]
Therefore,
\begin{equation}
    \MI(\wtW_T;S)
    \leq
    \sum_{t=1}^{T-1}
    \left(
        2\eta_t^2
        \E
        \left[
            V_t^{\mathrm{ad}}(\Lambda_t)
            +
            \Gamma_t^{\mathrm{ad}}(\Lambda_t)
        \right]
        +
        \mathcal C_t^{\mathrm{cov}}
    \right).
    \label{eq:appendix_c_mi_final}
\end{equation}
Substituting \eqref{eq:appendix_c_mi_final} into
\eqref{eq:appendix_c_mi_gen}, and then using
\eqref{eq:appendix_c_output_reduction}, proves
\eqref{eq:appendix_c_general_bound}.
\end{proof}

\subsection{Comparable and Admissibly Synchronized Covariance Corollaries}
\label{app:main-corollaries}

The general theorem uses reference-geometry quantities. The following corollary
gives a more readable form when the reference precision is comparable to the
actual precision.

\begin{corollary}[Comparable covariance geometries]
\label{cor:appendix_c_comparable}
Assume the conditions of Theorem~\ref{thm:appendix_c_general}. Suppose there
exists \(\kappa\geq1\) such that, for all \(t=1,\ldots,T-1\),
\begin{equation}
    (\Sigma_t^{\mathrm{ref}})^{-1}
    \preceq
    \kappa\Sigma_t^{-1}
    \qquad
    \text{almost surely}.
    \label{eq:appendix_c_comparable_precision}
\end{equation}
Then
\begin{equation}
    |\gen(W_T,S)|
    \leq
    \sqrt{
        \frac{2R^2}{n}
        \left[
            2\kappa
            \sum_{t=1}^{T-1}
            \eta_t^2
            \E
            \left[
                V_t^{\mathrm{ad}}
                +
                \Gamma_t^{\mathrm{ad}}
            \right]
            +
            \sum_{t=1}^{T-1}
            \mathcal C_t^{\mathrm{cov}}
        \right]
    }
    +
    \mathcal R_{\Delta}^{\mathrm{ad}}.
    \label{eq:appendix_c_comparable_bound}
\end{equation}
\end{corollary}

\begin{proof}
The precision comparison \eqref{eq:appendix_c_comparable_precision} implies
\[
    \|x\|_{\Lambda_t}^2
    \leq
    \kappa\|x\|_{\Sigma_t^{-1}}^2.
\]
Therefore,
\[
    V_t^{\mathrm{ad}}(\Lambda_t)
    \leq
    \kappa V_t^{\mathrm{ad}},
    \qquad
    \Gamma_t^{\mathrm{ad}}(\Lambda_t)
    \leq
    \kappa\Gamma_t^{\mathrm{ad}}.
\]
Substituting these inequalities into Theorem~\ref{thm:appendix_c_general} gives
the claim.
\end{proof}

The clean fixed-noise-style result is obtained only when the synchronized
reference covariance is admissible.

\begin{corollary}[Admissibly synchronized covariance bound]
\label{cor:appendix_c_synchronized}
Assume the conditions of Theorem~\ref{thm:appendix_c_general}. Suppose that, for
every \(t=1,\ldots,T-1\), the covariance process admits an \(S\)-admissible
synchronized reference, so that the reference covariance can be chosen as
\[
    \Sigma_t^{\mathrm{ref}}=\Sigma_t
\]
through an admissible reference construction. Then
\begin{equation}
    |\gen(W_T,S)|
    \leq
    \sqrt{
        \frac{4R^2}{n}
        \sum_{t=1}^{T-1}
        \eta_t^2
        \E
        \left[
            V_t^{\mathrm{ad}}
            +
            \Gamma_t^{\mathrm{ad}}
        \right]
    }
    +
    \mathcal R_{\Delta}^{\mathrm{ad}}.
    \label{eq:appendix_c_sync_bound}
\end{equation}
\end{corollary}

\begin{proof}
Under admissible synchronization,
\[
    \Lambda_t
    =
    \Sigma_t^{-1}.
\]
Thus,
\[
    V_t^{\mathrm{ad}}(\Lambda_t)=V_t^{\mathrm{ad}},
    \qquad
    \Gamma_t^{\mathrm{ad}}(\Lambda_t)=\Gamma_t^{\mathrm{ad}}.
\]
Moreover,
\[
    \Tr(\Sigma_t^{-1}\Sigma_t)-d+
    \log\frac{\det\Sigma_t}{\det\Sigma_t}
    =
    d-d+0
    =
    0,
\]
so \(\mathcal C_t^{\mathrm{cov}}=0\). Substituting into
Theorem~\ref{thm:appendix_c_general} yields
\eqref{eq:appendix_c_sync_bound}.
\end{proof}

\begin{remark}
The synchronized covariance condition is automatic for deterministic fixed
covariances. For arbitrary history-adaptive covariances, it is not automatic
because \(\Sigma_t\) may depend on \(S\) through \(\cH_{t-1}\). In that case,
synchronization must be justified by an admissible synchronization certificate.
Otherwise, the general covariance-comparison theorem should be used.
\end{remark}




\section{Recovery of the Original Fixed-Noise Bounds}
\label{app:fixed-recovery}

This appendix shows how the predictable covariance framework recovers the usual
fixed-noise virtual perturbation bounds when the perturbation covariance
sequence is deterministic. The key point is that deterministic covariances are
automatically predictable and automatically admit admissibly synchronized
reference comparisons. Therefore, the covariance-comparison cost in the general
theorem vanishes. Throughout this appendix, we use the indexing convention of the main text:
\(W_1\) is the initial iterate, \(W_T\) is the final output, and there are
\(T-1\) updates.

\subsection{Deterministic Covariances Are Predictable}
\label{subsec:appendix_d_deterministic_predictable}

Suppose that
\begin{equation}
    \Sigma_1,\ldots,\Sigma_{T-1}
\end{equation}
are deterministic positive definite matrices. Then each \(\Sigma_t\) is
measurable with respect to every sigma-field, and in particular
\begin{equation}
    \Sigma_t
    \text{ is }
    \cH_{t-1}\text{-measurable},
    \qquad
    t=1,\ldots,T-1.
\end{equation}
Thus, the predictable covariance condition is automatically satisfied.

The accumulated covariance is also deterministic:
\begin{equation}
    \Sigma_{1:t}
    =
    \sum_{k=1}^{t-1}
    \Sigma_k,
    \qquad
    \Sigma_{1:1}=0.
    \label{eq:appendix_d_fixed_accum_cov}
\end{equation}
For the final output,
\begin{equation}
    \Sigma_{1:T}
    =
    \sum_{k=1}^{T-1}
    \Sigma_k.
\end{equation}

Because the covariance sequence is deterministic, the actual and reference
virtual kernels can use the same covariance without introducing any direct
dependence on the training sample \(S\). Thus, the synchronized reference
comparison is admissible:
\begin{equation}
    \Sigma_t^{\mathrm{ref}}
    =
    \Sigma_t,
    \qquad
    \Lambda_t
    =
    \Sigma_t^{-1},
    \qquad
    \mathcal C_t^{\mathrm{cov}}
    =
    0.
    \label{eq:appendix_d_synchronized_fixed}
\end{equation}
Therefore, the general covariance-comparison theorem reduces to the clean
admissibly synchronized bound
\begin{equation}
    |\gen(W_T,S)|
    \leq
    \sqrt{
        \frac{4R^2}{n}
        \sum_{t=1}^{T-1}
        \eta_t^2
        \E
        \left[
            V_t^{\mathrm{ad}}
            +
            \Gamma_t^{\mathrm{ad}}
        \right]
    }
    +
    \mathcal R_{\Delta}^{\mathrm{ad}}.
    \label{eq:appendix_d_clean_synchronized_bound}
\end{equation}

The remainder of this appendix identifies the terms in
\eqref{eq:appendix_d_clean_synchronized_bound} with their fixed-covariance
counterparts.

\subsection{Reduction of Adaptive Gradient-Deviation and Sensitivity Terms}
\label{subsec:appendix_d_deviation_sensitivity_reduction}

The adaptive gradient-deviation term is
\begin{equation}
    V_t^{\mathrm{ad}}
    =
    \E
    \left[
        \normSigma{
            G_t-\barg(W_t)
        }{\Sigma_t^{-1}}^2
        \given
        \cH_{t-1}
    \right],
    \qquad
    t=1,\ldots,T-1.
    \label{eq:appendix_d_adaptive_deviation}
\end{equation}
When \(\Sigma_t\) is deterministic, the inverse covariance geometry is fixed.
Thus \eqref{eq:appendix_d_adaptive_deviation} is the
history-conditioned fixed-geometry gradient-deviation term. For clarity, define
\begin{equation}
    V_{t,\Sigma_t}^{\mathrm{hist}}
    =
    \E
    \left[
        \normSigma{
            G_t-\barg(W_t)
        }{\Sigma_t^{-1}}^2
        \given
        \cH_{t-1}
    \right].
    \label{eq:appendix_d_fixed_history_deviation}
\end{equation}
Then
\begin{equation}
    V_t^{\mathrm{ad}}
    =
    V_{t,\Sigma_t}^{\mathrm{hist}}.
    \label{eq:appendix_d_deviation_identity}
\end{equation}

The notation \(V\) is retained because this term reduces to a variance term in
the conditionally unbiased case. In full generality, however,
\(V_t^{\mathrm{ad}}\) is a conditional mean-square deviation from the population
gradient. If
\begin{equation}
    \E[G_t\mid \cH_{t-1}]
    =
    \barg(W_t),
    \label{eq:appendix_d_conditional_unbiasedness}
\end{equation}
then \(V_t^{\mathrm{ad}}\) becomes the conditional gradient variance measured in
the geometry \(\Sigma_t^{-1}\). Without
\eqref{eq:appendix_d_conditional_unbiasedness}, it also includes the squared
conditional bias between the conditional mean gradient and the population
gradient.

In some fixed-noise presentations, the corresponding quantity is written as a
function of \(W_t\), for example
\begin{equation}
    V_{t,\Sigma_t}(W_t)
    =
    \E
    \left[
        \normSigma{
            G_t-\barg(W_t)
        }{\Sigma_t^{-1}}^2
        \given
        W_t
    \right].
    \label{eq:appendix_d_fixed_deviation_wt}
\end{equation}
The history-conditioned quantity
\(V_{t,\Sigma_t}^{\mathrm{hist}}\) is always the more precise object for the
present pathwise analysis. It reduces to the \(W_t\)-conditioned form after
averaging only under an additional conditional Markov-type sampling condition,
namely that the conditional law of the current gradient \(G_t\) given the
available past depends on the past only through \(W_t\). Under this extra
condition,
\begin{equation}
    \E[V_{t,\Sigma_t}^{\mathrm{hist}}]
    =
    \E[V_{t,\Sigma_t}(W_t)].
    \label{eq:appendix_d_deviation_reduction}
\end{equation}
For general finite-sample or history-dependent sampling schemes, one should
retain the history-conditioned expression.

Next consider the adaptive gradient sensitivity. Let
\begin{equation}
    \zeta_t\given\cH_{t-1}
    \sim
    \N(0,\Sigma_{1:t})
\end{equation}
be a fresh perturbation with the accumulated covariance. Since
\(\Sigma_t\) and \(\Sigma_{1:t}\) are deterministic in the fixed-noise setting,
this fresh-copy representation is unambiguous. The adaptive sensitivity is
\begin{equation}
    \Gamma_t^{\mathrm{ad}}
    =
    \E
    \left[
        \normSigma{
            \barg(W_t+\zeta_t)-\barg(W_t)
        }{\Sigma_t^{-1}}^2
        \given
        \cH_{t-1}
    \right].
    \label{eq:appendix_d_adaptive_sensitivity}
\end{equation}
When \(\Sigma_t\) and \(\Sigma_{1:t}\) are deterministic, this is the
history-conditioned fixed geometry-aware sensitivity. Define
\begin{equation}
    \Gamma_{\Sigma_t,\Sigma_{1:t}}^{\mathrm{hist}}
    =
    \E
    \left[
        \normSigma{
            \barg(W_t+\zeta_t)-\barg(W_t)
        }{\Sigma_t^{-1}}^2
        \given
        \cH_{t-1}
    \right],
    \qquad
    \zeta_t\sim\N(0,\Sigma_{1:t}).
    \label{eq:appendix_d_fixed_history_sensitivity}
\end{equation}
Then
\begin{equation}
    \Gamma_t^{\mathrm{ad}}
    =
    \Gamma_{\Sigma_t,\Sigma_{1:t}}^{\mathrm{hist}}.
    \label{eq:appendix_d_sensitivity_identity}
\end{equation}

Equivalently, since the covariance is deterministic, one may write the usual
fixed pathwise sensitivity as a function of \(W_t\):
\begin{equation}
    \Gamma_{\Sigma_t,\Sigma_{1:t}}(W_t)
    =
    \E_{\zeta\sim\N(0,\Sigma_{1:t})}
    \left[
        \normSigma{
            \barg(W_t+\zeta)-\barg(W_t)
        }{\Sigma_t^{-1}}^2
    \right].
    \label{eq:appendix_d_fixed_sensitivity_wt}
\end{equation}
In this case,
\begin{equation}
    \E[\Gamma_{\Sigma_t,\Sigma_{1:t}}^{\mathrm{hist}}]
    =
    \E[\Gamma_{\Sigma_t,\Sigma_{1:t}}(W_t)].
    \label{eq:appendix_d_sensitivity_reduction}
\end{equation}

Substituting \eqref{eq:appendix_d_deviation_identity} and
\eqref{eq:appendix_d_sensitivity_identity} into
\eqref{eq:appendix_d_clean_synchronized_bound} gives the fixed
geometry-aware history-conditioned form:
\begin{equation}
    |\gen(W_T,S)|
    \leq
    \sqrt{
        \frac{4R^2}{n}
        \sum_{t=1}^{T-1}
        \eta_t^2
        \E
        \left[
            V_{t,\Sigma_t}^{\mathrm{hist}}
            +
            \Gamma_{\Sigma_t,\Sigma_{1:t}}^{\mathrm{hist}}
        \right]
    }
    +
    \mathcal R_{\Delta}^{\mathrm{ad}}.
    \label{eq:appendix_d_fixed_geometry_partial}
\end{equation}
Under the additional Markov-type sampling condition discussed above, this can
also be expressed using \(V_{t,\Sigma_t}(W_t)\) and
\(\Gamma_{\Sigma_t,\Sigma_{1:t}}(W_t)\) after taking expectations.

\subsection{Reduction of the Output Sensitivity Penalty}
\label{subsec:appendix_d_output_sensitivity_reduction}

In the adaptive framework, the output sensitivity is
\begin{equation}
    \Delta_{\Sigma_{1:T}}^{\mathrm{ad}}(W_T,s)
    =
    \E
    \left[
        L(W_T,s)-L(W_T+\zeta_T,s)
        \given
        \cH_{T-1}
    \right],
    \label{eq:appendix_d_adaptive_delta}
\end{equation}
where
\begin{equation}
    \zeta_T\given \cH_{T-1}
    \sim
    \N(0,\Sigma_{1:T}).
\end{equation}
If \(\Sigma_{1:T}\) is deterministic, the conditional law of \(\zeta_T\) is a
fixed Gaussian law, and \eqref{eq:appendix_d_adaptive_delta} reduces to the
fixed covariance output sensitivity
\begin{equation}
    \Delta_{\Sigma_{1:T}}(W_T,s)
    =
    \E_{\zeta\sim\N(0,\Sigma_{1:T})}
    \left[
        L(W_T,s)-L(W_T+\zeta,s)
    \right].
    \label{eq:appendix_d_fixed_delta}
\end{equation}
Consequently,
\begin{equation}
    \mathcal R_{\Delta}^{\mathrm{ad}}
    =
    \mathcal R_{\Delta},
    \label{eq:appendix_d_penalty_reduction}
\end{equation}
where
\begin{equation}
    \mathcal R_{\Delta}
    =
    \left|
    \E
    \left[
        \Delta_{\Sigma_{1:T}}(W_T,S')
        -
        \Delta_{\Sigma_{1:T}}(W_T,S)
    \right]
    \right|.
    \label{eq:appendix_d_fixed_penalty}
\end{equation}

Combining \eqref{eq:appendix_d_fixed_geometry_partial} with
\eqref{eq:appendix_d_penalty_reduction}, we obtain
\begin{equation}
    |\gen(W_T,S)|
    \leq
    \sqrt{
        \frac{4R^2}{n}
        \sum_{t=1}^{T-1}
        \eta_t^2
        \E
        \left[
            V_{t,\Sigma_t}^{\mathrm{hist}}
            +
            \Gamma_{\Sigma_t,\Sigma_{1:t}}^{\mathrm{hist}}
        \right]
    }
    +
    \mathcal R_{\Delta}.
    \label{eq:appendix_d_fixed_geometry_bound}
\end{equation}
Under the additional Markov-type sampling convention, this can also be written
as
\begin{equation}
    |\gen(W_T,S)|
    \leq
    \sqrt{
        \frac{4R^2}{n}
        \sum_{t=1}^{T-1}
        \eta_t^2
        \E
        \left[
            V_{t,\Sigma_t}(W_t)
            +
            \Gamma_{\Sigma_t,\Sigma_{1:t}}(W_t)
        \right]
    }
    +
    \mathcal R_{\Delta}.
    \label{eq:appendix_d_fixed_geometry_bound_wt}
\end{equation}
This is the fixed geometry-aware virtual perturbation bound in the usual
\(W_t\)-conditioned notation.

\subsection{Fixed Isotropic Specialization}
\label{subsec:appendix_d_fixed_isotropic}

We now specialize further to deterministic isotropic perturbations:
\begin{equation}
    \Sigma_t
    =
    \sigma_t^2I,
    \qquad
    \sigma_t>0,
    \qquad
    t=1,\ldots,T-1.
    \label{eq:appendix_d_isotropic_cov}
\end{equation}
Then
\begin{equation}
    \Sigma_{1:t}
    =
    \sigma_{1:t}^2I,
    \qquad
    \sigma_{1:t}^2
    =
    \sum_{k=1}^{t-1}
    \sigma_k^2.
    \label{eq:appendix_d_isotropic_accum}
\end{equation}
The inverse covariance norm becomes
\begin{equation}
    \normSigma{x}{\Sigma_t^{-1}}^2
    =
    \frac{1}{\sigma_t^2}\norm{x}^2.
\end{equation}

Define the Euclidean history-conditioned gradient-deviation numerator
\begin{equation}
    V_t^{\mathrm{Euc}}
    =
    \E
    \left[
        \norm{
            G_t-\barg(W_t)
        }^2
        \given
        \cH_{t-1}
    \right].
    \label{eq:appendix_d_iso_deviation}
\end{equation}
This is a variance numerator only under the conditional unbiasedness condition
\[
    \E[G_t\mid\cH_{t-1}]
    =
    \barg(W_t).
\]
Define also the Euclidean sensitivity numerator
\begin{equation}
    \Gamma_{\sigma_{1:t}}^{\mathrm{hist}}
    =
    \E
    \left[
        \norm{
            \barg(W_t+\zeta_t)-\barg(W_t)
        }^2
        \given
        \cH_{t-1}
    \right],
    \qquad
    \zeta_t\sim\N(0,\sigma_{1:t}^2I).
    \label{eq:appendix_d_iso_sensitivity}
\end{equation}
Then
\begin{equation}
    V_t^{\mathrm{ad}}
    =
    \frac{1}{\sigma_t^2}
    V_t^{\mathrm{Euc}},
    \qquad
    \Gamma_t^{\mathrm{ad}}
    =
    \frac{1}{\sigma_t^2}
    \Gamma_{\sigma_{1:t}}^{\mathrm{hist}}.
\end{equation}

Substituting these expressions into
\eqref{eq:appendix_d_fixed_geometry_bound}, we obtain
\begin{equation}
    |\gen(W_T,S)|
    \leq
    \sqrt{
        \frac{4R^2}{n}
        \sum_{t=1}^{T-1}
        \frac{\eta_t^2}{\sigma_t^2}
        \E
        \left[
            V_t^{\mathrm{Euc}}
            +
            \Gamma_{\sigma_{1:t}}^{\mathrm{hist}}
        \right]
    }
    +
    \mathcal R_{\Delta,\sigma}.
    \label{eq:appendix_d_fixed_isotropic_bound}
\end{equation}
Here
\begin{equation}
    \mathcal R_{\Delta,\sigma}
    =
    \left|
    \E
    \left[
        \Delta_{\sigma_{1:T}}(W_T,S')
        -
        \Delta_{\sigma_{1:T}}(W_T,S)
    \right]
    \right|,
\end{equation}
with
\begin{equation}
    \Delta_{\sigma_{1:T}}(W_T,s)
    =
    \E_{\zeta\sim\N(0,\sigma_{1:T}^2I)}
    \left[
        L(W_T,s)-L(W_T+\zeta,s)
    \right],
\end{equation}
and
\begin{equation}
    \sigma_{1:T}^2
    =
    \sum_{k=1}^{T-1}
    \sigma_k^2.
\end{equation}

Thus, the deterministic isotropic fixed-noise theorem is recovered as a special
case of the predictable covariance framework.

\begin{remark}
The recovery above should be interpreted at the level of the expected
quantities appearing in the generalization bound. Our adaptive notation
conditions on the full real SGD history \(\cH_{t-1}\), whereas some fixed-noise
presentations write gradient-deviation, variance, and sensitivity terms as
functions of \(W_t\) alone. Under an additional condition ensuring that the
conditional law of the current stochastic gradient depends on the past only
through \(W_t\), these formulations agree after taking expectations. In more
general sampling settings, the history-conditioned form is the more precise
analogue.
\end{remark}

\begin{remark}
Deterministic fixed covariances remove both sources of additional complication
in the history-adaptive theory: they are automatically predictable, and the
actual and reference covariances can be synchronized without depending on
\(S\). For genuinely history-adaptive covariances, the general
covariance-comparison theorem should be used unless an admissible synchronized
reference comparison is available.
\end{remark}




\section{Smoothness-Based Sensitivity Control}
\label{app:smoothness}

This appendix proves that the adaptive output perturbation-sensitivity term can
be controlled under smoothness and local curvature assumptions. The key
observation is that, for a centered Gaussian perturbation, the first-order term
in a Taylor expansion vanishes in expectation. A global smoothness assumption
then controls the sensitivity by the trace of the perturbation covariance,
whereas a local Hessian expansion gives a sharper curvature-weighted
interpretation. This appendix controls only the adaptive output-sensitivity penalty
\[
    \mathcal R_{\Delta}^{\mathrm{ad}}.
\]
It is independent of the covariance-comparison term
\[
    \mathcal C_t^{\mathrm{cov}},
\]
which appears in the information part of the main theorem when the actual and
reference virtual perturbation covariances differ.

\subsection{Fixed-Covariance Output Sensitivity}
\label{subsec:appendix_smoothness_fixed_output}

For a deterministic sample \(s=\{z_1,\ldots,z_n\}\), define the fixed-covariance
output sensitivity by
\begin{equation}
    \Delta_{\Sigma}(W,s)
    =
    \E_{\zeta\sim\N(0,\Sigma)}
    \left[
        L(W,s)-L(W+\zeta,s)
    \right],
    \label{eq:appendix_e_fixed_delta}
\end{equation}
where \(\zeta\) is a fresh perturbation independent of all other randomness. If
\(W\) is random, the expectation in \eqref{eq:appendix_e_fixed_delta} is
understood conditionally on \(W\), or equivalently over an independent draw of
\(\zeta\) given \(W\).

Because of the sign convention
\[
    \Delta_{\Sigma}(W,s)
    =
    \E_{\zeta}
    \left[
        L(W,s)-L(W+\zeta,s)
    \right],
\]
the sensitivity may be negative near a local minimum, since perturbing \(W\)
often increases the loss. Therefore, the most useful control is an
absolute-value bound.

We use the following smoothness assumption.

\begin{assumption}[Two-sided smoothness remainder]
\label{ass:appendix_e_smoothness}
For every deterministic sample \(s\), the empirical risk \(L(\cdot,s)\) is
differentiable and satisfies, for all \(w,u\in\R^d\),
\begin{equation}
    \left|
        L(w+u,s)-L(w,s)-\inner{\nabla L(w,s)}{u}
    \right|
    \leq
    \frac{\mu}{2}\norm{u}^2.
    \label{eq:appendix_e_two_sided_smoothness}
\end{equation}
\end{assumption}

Assumption~\ref{ass:appendix_e_smoothness} follows, for example, when
\(L(\cdot,s)\) has a \(\mu\)-Lipschitz gradient. Indeed, \(\mu\)-smoothness gives
both upper and lower quadratic Taylor remainder bounds:
\[
    -\frac{\mu}{2}\norm{u}^2
    \leq
    L(w+u,s)-L(w,s)-\inner{\nabla L(w,s)}{u}
    \leq
    \frac{\mu}{2}\norm{u}^2.
\]

\begin{proposition}[Fixed-covariance sensitivity control]
\label{prop:appendix_e_fixed_control}
Suppose Assumption~\ref{ass:appendix_e_smoothness} holds. Let
\(\zeta\sim\N(0,\Sigma)\), where \(\Sigma\succeq0\), and assume \(\zeta\) is
independent of \(W\) when \(W\) is random. Then, for every deterministic sample
\(s\),
\begin{equation}
    \left|
    \Delta_{\Sigma}(W,s)
    \right|
    \leq
    \frac{\mu}{2}\Tr(\Sigma).
    \label{eq:appendix_e_fixed_control}
\end{equation}
\end{proposition}

\begin{proof}
Condition on \(W\). After conditioning, \(W\) is fixed and
\(\zeta\sim\N(0,\Sigma)\) is centered. Applying
Assumption~\ref{ass:appendix_e_smoothness} with \(w=W\) and \(u=\zeta\) gives
\[
    \left|
        L(W+\zeta,s)-L(W,s)-\inner{\nabla L(W,s)}{\zeta}
    \right|
    \leq
    \frac{\mu}{2}\norm{\zeta}^2.
\]
Taking expectation with respect to \(\zeta\), the linear term vanishes because
\[
    \E[\zeta]=0.
\]
Therefore,
\[
    \left|
    \E
    \left[
        L(W+\zeta,s)-L(W,s)
    \right]
    \right|
    \leq
    \frac{\mu}{2}
    \E[\norm{\zeta}^2].
\]
Since
\[
    \E[\norm{\zeta}^2]=\Tr(\Sigma),
\]
and
\[
    \Delta_{\Sigma}(W,s)
    =
    -
    \E
    \left[
        L(W+\zeta,s)-L(W,s)
    \right],
\]
the claim follows.
\end{proof}

\begin{remark}[Sign convention]
Near a local minimum with positive semidefinite curvature, perturbing the
parameter often increases the loss on average. Thus one typically expects
\[
    \Delta_{\Sigma}(W,s)\leq0
\]
near such points. The absolute-value bound in
Proposition~\ref{prop:appendix_e_fixed_control} is therefore the more
informative stability statement.
\end{remark}

\subsection{Adaptive Conditional Output Sensitivity}
\label{subsec:appendix_smoothness_adaptive_output}

We now extend the previous argument to the adaptive setting. Recall that the
final accumulated adaptive covariance is
\begin{equation}
    \Sigma_{1:T}
    =
    \sum_{k=1}^{T-1}\Sigma_k.
    \label{eq:appendix_e_final_accum_cov}
\end{equation}
Since \(W_T\) is the final output after \(T-1\) updates, the final real SGD
history is \(\cH_{T-1}\). The adaptive output sensitivity is
\begin{equation}
    \Delta_{\Sigma_{1:T}}^{\mathrm{ad}}(W_T,s)
    =
    \E
    \left[
        L(W_T,s)-L(W_T+\zeta_T,s)
        \given
        \cH_{T-1}
    \right],
    \label{eq:appendix_e_adaptive_delta}
\end{equation}
where
\begin{equation}
    \zeta_T\given\cH_{T-1}
    \sim
    \N(0,\Sigma_{1:T})
    \label{eq:appendix_e_final_conditional_noise}
\end{equation}
is a fresh final perturbation.

\begin{proposition}[Adaptive conditional sensitivity control]
\label{prop:appendix_e_adaptive_control}
Suppose Assumption~\ref{ass:appendix_e_smoothness} holds. Then, for every
deterministic sample \(s\),
\begin{equation}
    \left|
    \Delta_{\Sigma_{1:T}}^{\mathrm{ad}}(W_T,s)
    \right|
    \leq
    \frac{\mu}{2}\Tr(\Sigma_{1:T})
    \qquad
    \text{almost surely}.
    \label{eq:appendix_e_adaptive_control}
\end{equation}
\end{proposition}

\begin{proof}
Condition on \(\cH_{T-1}\). After conditioning, both \(W_T\) and
\(\Sigma_{1:T}\) are fixed, and
\[
    \zeta_T\given \cH_{T-1}
    \sim
    \N(0,\Sigma_{1:T}).
\]
Moreover,
\[
    \E[\zeta_T\given\cH_{T-1}]=0,
    \qquad
    \E[\norm{\zeta_T}^2\given\cH_{T-1}]
    =
    \Tr(\Sigma_{1:T}).
\]
Applying the same argument as in
Proposition~\ref{prop:appendix_e_fixed_control}, conditionally on
\(\cH_{T-1}\), gives
\[
    \left|
    \E
    \left[
        L(W_T,s)-L(W_T+\zeta_T,s)
        \given
        \cH_{T-1}
    \right]
    \right|
    \leq
    \frac{\mu}{2}\Tr(\Sigma_{1:T}).
\]
This proves \eqref{eq:appendix_e_adaptive_control}.
\end{proof}

\subsection{Final Penalty Bound}
\label{subsec:appendix_smoothness_final_penalty}

The adaptive final sensitivity penalty is
\begin{equation}
    \mathcal R_{\Delta}^{\mathrm{ad}}
    =
    \left|
    \E
    \left[
        \Delta_{\Sigma_{1:T}}^{\mathrm{ad}}(W_T,S')
        -
        \Delta_{\Sigma_{1:T}}^{\mathrm{ad}}(W_T,S)
    \right]
    \right|.
    \label{eq:appendix_e_rad_def}
\end{equation}

\begin{proposition}[Smoothness control of the final sensitivity penalty]
\label{prop:appendix_e_final_penalty}
Suppose Assumption~\ref{ass:appendix_e_smoothness} holds and
\[
    \E[\Tr(\Sigma_{1:T})]<\infty.
\]
Then
\begin{equation}
    \mathcal R_{\Delta}^{\mathrm{ad}}
    \leq
    \mu
    \E[\Tr(\Sigma_{1:T})].
    \label{eq:appendix_e_final_penalty_bound}
\end{equation}
\end{proposition}

\begin{proof}
Using the triangle inequality,
\begin{align}
    \mathcal R_{\Delta}^{\mathrm{ad}}
    &\leq
    \E
    \left[
        \left|
        \Delta_{\Sigma_{1:T}}^{\mathrm{ad}}(W_T,S')
        \right|
    \right]
    +
    \E
    \left[
        \left|
        \Delta_{\Sigma_{1:T}}^{\mathrm{ad}}(W_T,S)
        \right|
    \right].
\end{align}
Applying Proposition~\ref{prop:appendix_e_adaptive_control} to both terms gives
\[
    \mathcal R_{\Delta}^{\mathrm{ad}}
    \leq
    \frac{\mu}{2}\E[\Tr(\Sigma_{1:T})]
    +
    \frac{\mu}{2}\E[\Tr(\Sigma_{1:T})]
    =
    \mu\E[\Tr(\Sigma_{1:T})].
\]
\end{proof}

The bound above is robust and easy to use, but it can be pessimistic in
high-dimensional nonconvex settings because it uses a global smoothness
constant. We next record a local curvature refinement.

\subsection{Local Curvature Refinement}
\label{subsec:appendix_smoothness_local_curvature}

For a deterministic sample \(s\), define the local Hessian
\begin{equation}
    H_s(W)
    =
    \nabla^2 L(W,s).
    \label{eq:appendix_e_local_hessian}
\end{equation}

\begin{assumption}[Local Hessian-Lipschitz condition]
\label{ass:appendix_e_local_hessian_lipschitz}
For a deterministic sample \(s\), assume that \(L(\cdot,s)\) is twice
differentiable in a neighborhood of \(W\), and that there exists
\(\rho_s\geq0\) such that
\begin{equation}
    \left\|
        \nabla^2L(W+u,s)-\nabla^2L(W,s)
    \right\|_{\mathrm{op}}
    \leq
    \rho_s\norm{u}
    \label{eq:appendix_e_hessian_lipschitz}
\end{equation}
for all perturbations \(u\) in the region of interest. If the condition holds
globally, the region of interest is all of \(\R^d\).
\end{assumption}

\begin{proposition}[Second-order expansion of fixed output sensitivity]
\label{prop:appendix_e_fixed_local_curvature}
Suppose Assumption~\ref{ass:appendix_e_local_hessian_lipschitz} holds for
sample \(s\). Let \(\zeta\sim\N(0,\Sigma)\). Then
\begin{equation}
    \Delta_{\Sigma}(W,s)
    =
    -
    \frac12
    \Tr
    \left(
        H_s(W)\Sigma
    \right)
    +
    R_s(W,\Sigma),
    \label{eq:appendix_e_fixed_local_expansion}
\end{equation}
where the remainder satisfies
\begin{equation}
    |R_s(W,\Sigma)|
    \leq
    \frac{\rho_s}{6}
    \E_{\zeta\sim\N(0,\Sigma)}
    \left[
        \norm{\zeta}^3
    \right].
    \label{eq:appendix_e_fixed_remainder}
\end{equation}
\end{proposition}

\begin{proof}
Taylor's theorem with third-order remainder controlled by the Hessian-Lipschitz
condition gives
\begin{align}
    L(W+\zeta,s)
    &=
    L(W,s)
    +
    \inner{\nabla L(W,s)}{\zeta}
    +
    \frac12
    \zeta^\top H_s(W)\zeta
    +
    r_s(\zeta),
\end{align}
where
\[
    |r_s(\zeta)|
    \leq
    \frac{\rho_s}{6}\norm{\zeta}^3.
\]
Taking expectation over \(\zeta\), the linear term vanishes and
\[
    \E[\zeta^\top H_s(W)\zeta]
    =
    \Tr(H_s(W)\Sigma).
\]
Using
\[
    \Delta_{\Sigma}(W,s)
    =
    \E[L(W,s)-L(W+\zeta,s)]
\]
gives the expansion with
\[
    R_s(W,\Sigma)
    =
    -\E[r_s(\zeta)].
\]
The remainder bound follows immediately.
\end{proof}

\begin{proposition}[Second-order expansion of adaptive output sensitivity]
\label{prop:appendix_e_adaptive_local_curvature}
Suppose Assumption~\ref{ass:appendix_e_local_hessian_lipschitz} holds for
sample \(s\) at \(W=W_T\), conditionally on \(\cH_{T-1}\). Then
\begin{equation}
    \Delta_{\Sigma_{1:T}}^{\mathrm{ad}}(W_T,s)
    =
    -
    \frac12
    \Tr
    \left(
        H_s(W_T)\Sigma_{1:T}
    \right)
    +
    R_s(W_T,\Sigma_{1:T}),
    \label{eq:appendix_e_adaptive_local_expansion}
\end{equation}
where
\begin{equation}
    \left|
        R_s(W_T,\Sigma_{1:T})
    \right|
    \leq
    \frac{\rho_s}{6}
    \E
    \left[
        \norm{\zeta_T}^3
        \given
        \cH_{T-1}
    \right].
    \label{eq:appendix_e_adaptive_remainder}
\end{equation}
\end{proposition}

\begin{proof}
Condition on \(\cH_{T-1}\). Then \(W_T\) and \(\Sigma_{1:T}\) are fixed, and
\[
    \zeta_T\given\cH_{T-1}
    \sim
    \N(0,\Sigma_{1:T}).
\]
Apply Proposition~\ref{prop:appendix_e_fixed_local_curvature} conditionally
with \(W=W_T\) and \(\Sigma=\Sigma_{1:T}\).
\end{proof}

\begin{lemma}[Gaussian third-moment control]
\label{lem:appendix_e_gaussian_third_moment}
If \(\zeta\sim\N(0,\Sigma)\), then
\begin{equation}
    \E\norm{\zeta}^3
    \leq
    2^{3/2}
    \left[
        \Tr(\Sigma)
    \right]^{3/2}.
    \label{eq:appendix_e_gaussian_third_moment}
\end{equation}
\end{lemma}

\begin{proof}
The bound is a standard Gaussian moment estimate. Since
\[
    \E\norm{\zeta}^2=\Tr(\Sigma),
\]
one may use a Gaussian norm moment comparison to obtain
\[
    \E\norm{\zeta}^3
    \leq
    2^{3/2}
    \left[
        \Tr(\Sigma)
    \right]^{3/2}.
\]
\end{proof}

\subsection{Curvature-Mismatch Control of the Adaptive Penalty}
\label{subsec:appendix_smoothness_curvature_mismatch}

The local expansion shows that the output-sensitivity penalty is controlled by
the difference between the local curvature of the ghost loss and the local
curvature of the training loss, weighted by the accumulated adaptive covariance.

\begin{proposition}[Local curvature control of the adaptive output penalty]
\label{prop:appendix_e_curvature_mismatch_penalty}
Suppose the adaptive local curvature expansion
\eqref{eq:appendix_e_adaptive_local_expansion} holds for both \(S\) and \(S'\).
Then
\begin{align}
    \mathcal R_{\Delta}^{\mathrm{ad}}
    &\leq
    \frac12
    \left|
    \E
    \left[
        \Tr
        \left(
            \left[
                H_{S'}(W_T)-H_S(W_T)
            \right]
            \Sigma_{1:T}
        \right)
    \right]
    \right|
    \nonumber\\
    &\quad+
    \E
    \left[
        \left|
            R_{S'}(W_T,\Sigma_{1:T})
        \right|
        +
        \left|
            R_S(W_T,\Sigma_{1:T})
        \right|
    \right].
    \label{eq:appendix_e_curvature_mismatch_penalty}
\end{align}
If the local Hessian-Lipschitz constants are \(\rho_{S'}\) and \(\rho_S\), then
\begin{align}
    \mathcal R_{\Delta}^{\mathrm{ad}}
    &\leq
    \frac12
    \left|
    \E
    \left[
        \Tr
        \left(
            \left[
                H_{S'}(W_T)-H_S(W_T)
            \right]
            \Sigma_{1:T}
        \right)
    \right]
    \right]
    \nonumber\\
    &\quad+
    \frac{2^{3/2}}{6}
    \E
    \left[
        (\rho_{S'}+\rho_S)
        \left[
            \Tr(\Sigma_{1:T})
        \right]^{3/2}
    \right].
    \label{eq:appendix_e_curvature_mismatch_trace}
\end{align}
\end{proposition}

\begin{proof}
By Proposition~\ref{prop:appendix_e_adaptive_local_curvature},
\[
    \Delta_{\Sigma_{1:T}}^{\mathrm{ad}}(W_T,S')
    =
    -
    \frac12
    \Tr(H_{S'}(W_T)\Sigma_{1:T})
    +
    R_{S'}(W_T,\Sigma_{1:T}),
\]
and
\[
    \Delta_{\Sigma_{1:T}}^{\mathrm{ad}}(W_T,S)
    =
    -
    \frac12
    \Tr(H_S(W_T)\Sigma_{1:T})
    +
    R_S(W_T,\Sigma_{1:T}).
\]
Subtracting the two expressions, taking expectation and absolute values, and
applying the triangle inequality gives
\eqref{eq:appendix_e_curvature_mismatch_penalty}. The trace-form remainder
bound follows from Proposition~\ref{prop:appendix_e_adaptive_local_curvature}
and Lemma~\ref{lem:appendix_e_gaussian_third_moment}.
\end{proof}

\begin{remark}[Local flatness interpretation]
If the local Hessian \(H_s(W_T)\) is positive semidefinite and the Taylor
remainder is negligible, then
\[
    \left|
    \Delta_{\Sigma_{1:T}}^{\mathrm{ad}}(W_T,s)
    \right|
    \approx
    \frac12
    \Tr(H_s(W_T)\Sigma_{1:T}).
\]
Thus, the output sensitivity is governed by covariance-weighted local
curvature. Perturbation mass in low-curvature directions contributes little to
the penalty, while perturbation mass aligned with high-curvature directions can
increase it.
\end{remark}

\begin{remark}[Relation to the main theorem]
This appendix controls only the output-sensitivity term
\(\mathcal R_{\Delta}^{\mathrm{ad}}\). The covariance-comparison term
\(\mathcal C_t^{\mathrm{cov}}\), which appears in the information part of the
general theorem when actual and reference perturbation covariances differ, is
controlled separately through Gaussian KL comparison.
\end{remark}

\end{document}